\documentclass[11pt]{article}
\usepackage[margin=1in]{geometry}
\usepackage{times}
\usepackage[round]{natbib}

\usepackage{amsmath,amsfonts,bm}

\def\eqref#1{equation~\ref{#1}}

\def\1{\bm{1}}

\DeclareMathAlphabet{\mathsfit}{\encodingdefault}{\sfdefault}{m}{sl}
\SetMathAlphabet{\mathsfit}{bold}{\encodingdefault}{\sfdefault}{bx}{n}

\newcommand{\E}{\mathbb{E}}

\newcommand{\R}{\mathbb{R}}

\DeclareMathOperator*{\argmax}{arg\,max}
\DeclareMathOperator*{\argmin}{arg\,min}

\usepackage{hyperref}
\hypersetup{colorlinks=true,linkcolor=blue,citecolor=blue,urlcolor=blue}
\usepackage{url}
\usepackage{algorithm}
\usepackage{algorithmic}

\usepackage{mystyle_1}
\title{Uncertainty-Aware Selection of Online Algorithms with Simulator Ensembles}

\author{Yongyi Guo\textsuperscript{1} \quad Zifan Xu\textsuperscript{2} \quad Ziping Xu\textsuperscript{3} \quad Kelly W. Zhang\textsuperscript{4}\\[1.5ex]
\normalsize
\textsuperscript{1}Department of Statistics, University of Wisconsin--Madison\\
\textsuperscript{2}Department of Computer Science, The University of Texas at Austin\\
\textsuperscript{3}School of Data and Information Sciences, University of North Carolina at Chapel Hill\\
\textsuperscript{4}Department of Mathematics, Imperial College London\\
\textsuperscript{ } \\
\textsuperscript{*}Authors are listed in alphabetical order.}
\date{}

\usepackage{graphicx}

\newif\ifshowcomments
\showcommentsfalse

\usepackage{amsthm}
\usepackage{booktabs}
\usepackage{enumitem}
\usepackage{mathtools}
\newtheorem{theorem}{Theorem}
\newtheorem{lemma}[theorem]{Lemma}
\newtheorem{proposition}[theorem]{Proposition}

\newcommand{\eqrefff}[1]{(\ref{#1})}
\newtheorem{innercustomgeneric}{\customgenericname}
\providecommand{\customgenericname}{}
\newcommand{\newcustomtheorem}[2]{%
  \newenvironment{#1}[1]
  {%
   \renewcommand\customgenericname{#2}%
   \renewcommand\theinnercustomgeneric{##1}%
   \innercustomgeneric
  }
  {\endinnercustomgeneric}
}
\newcustomtheorem{customthm}{Theorem}
\newcustomtheorem{customlemma}{Lemma}
\newcustomtheorem{customprop}{Proposition}
\newcustomtheorem{customcond}{Condition}
\newcustomtheorem{customcor}{Corollary}

\begin{document}

\maketitle

\begin{abstract}
The performance of online reinforcement learning depends critically on design choices, especially those that affect exploration. These choices are often selected by fitting a simulator to offline data, evaluating candidate algorithms in that simulator, and deploying the best-performing one. The simplest \emph{Plug-In selection rule} simply selects the best performing algorithm on the fitted simulator, making evaluations unreliable when the offline data used to fit the simulator are limited.
We investigate \emph{Uncertainty-Aware selection}, which forms an ensemble of simulators---for example, obtained by bootstrap resampling---and selects the online algorithm with the best average performance across the ensemble. While ensemble-based approaches have been used to mitigate distribution shift and facilitate sim-to-real transfer, we formally show that this approach can mitigate the effects of limited data when fitting the simulator and theoretically has significant regret gains compared to Plug-In selection in multi-armed bandits. We also empirically investigate the Uncertainty-Aware selection approach in deep RL experiments on robotic control tasks that involve selecting reward-shaping hyperparameters, and show that it leads to more reliable selection and improved online performance.
\end{abstract}

\section{Introduction}
Selecting an effective online reinforcement learning (RL) algorithm to deploy in practice is challenging because the deployment performance of online RL algorithms can vary substantially with design decisions (e.g., architecture, hyperparameters, and reward scaling or shaping) \citep{henderson2018deep,wang2022no,fan2025fragility}. Moreover, deployment of online RL in important domains, including robotics \citep{kober2013reinforcement,dulac2019challenges,ibarz2021how}, healthcare \citep{trella2022designing,figueroa2021adaptive}, power grids \citep{zhang2018review,chen2022reinforcement}, is often very expensive and thus, it can be infeasible to test out many candidate online RL algorithms in a real deployment.

In many
domains, however, one may have access to a (potentially limited) offline dataset before deployment to warm start or inform online RL algorithm design decisions. A common approach is to use the offline dataset to form an environment simulator that can be used as a surrogate to assess candidate
algorithms \citep{mandel2016offline}.
In robotics, system identification
approaches use a limited
offline data to calibrate a physical simulator for policy training
prior to deployment \citep{ljung1999system,peng2018sim,muratore2022robot}. In digital health interventions, digital twins built using offline data are used to simulate patients for
assessing candidate online intervention algorithms before deployment on a study population \citep{bjornsson2020digital,trella2022designing,xu2026diffusion}.

The appeal of the simulator approach is its generality: a simulator can run any candidate algorithm end to end no matter how complex the online RL algorithm design grows. Simulators can be constructed nonparametrically, by replaying logged transitions through queues \citep{mandel2016offline}, nearest-neighbor retrieval \citep{wang2022no}, rejection sampling \citep{tang2022towards} or parametrically, by fitting a model of the environment to the data \citep{trella2022designing,ghosh2024rebandit,m2023model,yang2023foundation}. The evaluated algorithm can be adaptive, history-dependent, or otherwise non-trivial, whereas offline policy evaluation \citep{precup2000eligibility, jiang2016doubly, uehara2022review} scores a fixed policy and does not directly extend to online RL algorithms.

Many existing approaches to fitting a simulator from offline data involve fitting a single model of the environment and selecting the candidate algorithm that performs the best on that fitted simulator. We call this selection approach the \textit{Plug-In} selection rule.
This approach has been used in practice in a variety of settings, including the Oralytics mobile health trial \citep{trella2022designing,trella2024oralytics} and the educational game study \citep{mandel2016offline}.
In this paper, we identify a failure mode of the Plug-In simulator: it carries no information about how uncertain
its estimate is, and we show that this omission can be costly when the offline data are scarce and the online horizon is long.

In this work, we advocate for an Uncertainty-Aware selection approach that constructs an \emph{ensemble} of simulators (Figure~\ref{fig:selection}). The ensemble consists of copies of the fitted simulator whose parameters are perturbed according to their estimation uncertainty, for example by refitting on bootstrap resamples of the offline data. Each candidate algorithm is scored by its average return across the simulators in the ensemble.
Our contributions are as follows.
\begin{enumerate}[leftmargin=*]
\item\textbf{Provable benefits of Uncertainty-Aware selection in bandits.} For UCB and Thompson Sampling in two-armed Gaussian bandits, we establish a provable and empirically validated mechanism by which accounting for simulator uncertainty improves algorithm selection. Specifically, Theorem~\ref{thm::UCB-regret-main} characterizes the asymmetric costs of exploration: underexploration incurs polynomial regret, whereas overexploration preserves logarithmic regret. Building on this result, we show that Uncertainty-Aware selection asymptotically chooses at least as much exploration as Plug-In (Proposition~\ref{prop:ucb-exploration}), resulting in no larger asymptotic regret growth and, with nonvanishing probability, a polynomial-to-logarithmic improvement (Theorem~\ref{thm::UCB-selection-main-paper}). Simulations support this mechanism, showing reductions in both the mean and variability of regret.
    \item \textbf{Uncertainty-Aware selection as a curvature regularized version of Plug-In selection.} In Proposition \ref{prop:regularized1}, we formally show that Uncertainty-Aware selection criterion can be interpreted as a regularized version of the Plug-In selection criterion that depends on the curvature of the objective function with respect to the simulator parameters. This result holds for general decision-making settings (including bandits and Markov decision processes (MDPs)).
    \item \textbf{Empirical benefit for specifying reward-shaping for deep RL.} We demonstrate the benefit of Uncertainty-Aware selection empirically in a deep RL reward-shaping task in which physics parameters of a MuJoCo simulator are identified from offline trajectories.
\end{enumerate}
Taken together, our results show that ensembles of simulators are not only a tool for training robust policies: they provably improve the selection of online learning algorithms from limited offline data.

\begin{figure}[t]
  \centering
  \vspace{-3mm}
  \includegraphics[width=0.9\linewidth]{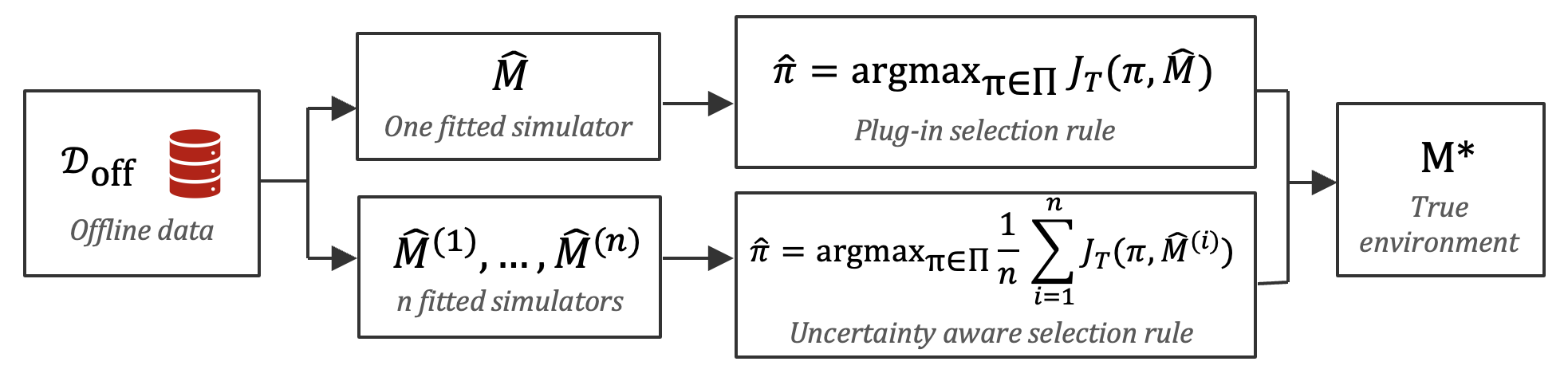}
  \vspace{-5mm}
  \caption{\textbf{Selecting an online RL algorithm using limited offline data.} Both rules fit simulators to $\mathcal{D}_{\textnormal{off}}$ and then deploy the selected algorithm in the true environment $M^*$. The Plug-In rule commits to a single fitted simulator $\widehat{M}$; the Uncertainty-Aware rule averages the objective over $\widehat{M}^{(1)}, \dots, \widehat{M}^{(n)}$, favoring algorithms that perform consistently across plausible realizations.}
  \label{fig:selection}
  \vspace{-5mm}
\end{figure}

\section{Problem setup}
\label{sec:setup}
\textbf{Decision-making environment.}
States $S_{t,h}$ take values in state space $\mathcal{S}$ and actions $A_{t,h}$
in action space $\mathcal{A}$.
$M^*$ denotes the general decision-making environment, which specifies the distribution of the reward and next state given the most recent action, state, and history, i.e., $\mathbb{P}( R_{t,h}, S_{t,h+1} \mid A_{t,h}, S_{t,h}, \mathcal{H}_{t,h-1})$, where the history up to episode $t$ and the within-episode timestep $h$:
\begin{align*}
    \mathcal{H}_{t,h-1} := \big\{ (S_{t,h'}, A_{t,h'}, R_{t,h'}) \big\}_{h'=1}^{h-1} \cup \big\{ (S_{t',h}, A_{t',h}, R_{t',h}) \big\}_{t'=1,h=1}^{t'=t-1,h=H}.
\end{align*}
In our theory and experiments $M^*$ is either a finite-horizon MDP \citep{puterman2014markov} or a multi-armed bandit (MAB) environment \citep{lattimore2020}. In the MAB setting, $H=1$ and $|\mathcal{S}| = 1$ so we use $\mathcal{H}_t = \{ (A_{t^\prime}, R_{t^\prime}) \}_{t^\prime=1}^{t}$ to denote the history.

\textbf{Offline dataset.}
We assume we have access to an offline dataset consisting of $T_{\textnormal{off}}$ episodes from environment $M^*$ each of horizon $H$ collected by a behavior policy $\pi_b$:
\begin{align*}
    \mathcal{D}_{\textnormal{off}} := \big\{ (S_{t,1}, A_{t,1}, R_{t,1}), (S_{t,2}, A_{t,2}, R_{t,2}), \dots (S_{t,H}, A_{t,H}, R_{t,H}) \big\}_{t=1}^{T_{\textnormal{off}}}.
\end{align*}
The policy $\pi_b$ takes $(\mathcal{H}_{t,h-1}, S_{t,h})$ as input and outputs a distribution over actions $\mathcal{A}$. Note, in the MAB special case where $H=1$ and $|\mathcal{S}| = 1$, for simplicity, we just write $\mathcal{D}_{\textnormal{off}} := \{ (A_{t}, R_{t}) \}_{t=1}^{T_{\textnormal{off}}}$.

\textbf{Algorithm selection objective.}
Our goal is to use the offline dataset $\mathcal{D}_{\textnormal{off}}$ to inform the selection of an online RL algorithm $\pi \in \Pi$, which specifies action selection probabilities $\pi(A_{t,h} \mid S_{t,h}, \mathcal{H}_{t,h-1}) := \mathbb{P}(A_{t,h} \mid S_{t,h}, \mathcal{H}_{t,h-1})$. The algorithm $\pi$ will be deployed for $T$ episodes in the same environment $M^*$ where $\mathcal{D}_{\textnormal{off}}$ was collected. Specifically, we want to select algorithm $\pi$ to maximize an objective function $J_T(\pi, M^*)$. Unless stated otherwise, throughout this work, the objective could be the expected cumulative reward
\begin{align}
J_T(\pi, M^*) = \E_{\pi, M^*} \bigg[ \sum_{t=1}^T \sum_{h=1}^H R_{t,h} \bigg],
\end{align}
where the expectation is taken over the random trajectories induced by the online algorithm $\pi$. In other settings like robotic learning, one might also be interested in the expected reward on the final episode $J_T(\pi, M^*) = \E_{\pi, M^*} \big[ \sum_{h=1}^H R_{T,h} \big]$.

\textbf{Candidate algorithm class.}
The class of candidate history-dependent policies $\Pi$ could include exploration strategies of algorithm (e.g., $\epsilon$-greedy \citep{wunder2010classes}, Thompson Sampling \citep{daniel2018tutorial}, bonus-based exploration \citep{bellemare2016unifying}),
model architectures, state specification, reward specification/reward shaping, and algorithm hyperparameters.
We index candidate online algorithms by a tuning parameter $\theta \in \Theta$, so $\Pi = \{\pi_\theta : \theta \in \Theta\}$. For example, in UCB, $\theta$ controls the width of the exploration bonus; in $\epsilon$-greedy Q-learning, it may control the exploration rate. In general, $\theta$ may index other design choices, such as model architecture, state representation, or reward shaping. We use $\theta_T^* = \mathrm{argmax}_{\theta \in \Theta} J_T(\pi_\theta,M^*)$ to denote the optimal $\theta$ for a given $T$.

\textbf{Plug-In simulation selection rule.}
We use the offline dataset $\mathcal{D}_{\textnormal{off}}$ to fit an approximate environment model $\widehat M$ of $M^*$, which may be an explicit dynamics model or a learned generative world model. We then use $\widehat M$ as a proxy environment in which to evaluate candidate decision procedures.

A simple and commonly used simulator-based selection rule \citep{m2023model,wang2022no,yang2023foundation} is the Plug-In selection rule which forms an estimated simulator $\widehat{M}$
\begin{align}
    \label{eq::plugin-selection}
    \widehat\pi_{\textnormal{plug}}
    \in \argmax_{\pi\in\Pi} J_T(\pi,\widehat M).
\end{align}
 The above solution can be solved empirically by repeatedly rolling out the algorithm $\pi$ in environment $\widehat{M}$ for different choices of algorithm $\pi \in \Pi$.
\footnote{Picking the best RL algorithm using replay-style simulators \citep{mandel2016offline,tang2022towards,wang2022no} is also Plug-In selection where the simulator is an empirical distribution fit non-parametrically.}

\textbf{Uncertainty-Aware simulation selection rule.}
In this work we advocate for an Uncertainty-Aware (UA) selection approach. This approach involves using $\mathcal{D}_{\textnormal{off}}$ to fit an ensemble of simulators and choosing the algorithm candidate that does well on average across an ensemble of fitted simulators $\widehat M^{(1)}, \dots \widehat M^{(n)}$. Then the Uncertainty-Aware selection rule selects the algorithm as follows:
\begin{align}
    \widehat\pi_{\UA} = \argmax_{\pi \in \Pi} \, \frac{1}{n} \sum_{i=1}^n J_T \big(\pi, \widehat M^{(i)} \big),
    \label{eqn:ensembleObjective}
\end{align}
Rather than
maximizing performance on the single fitted simulator $\widehat M$, \eqrefff{eqn:ensembleObjective} instead evaluates candidates on average across ensembles $\widehat M^{(n)}$, which favors algorithms with consistently strong performance across plausible realizations of the simulator $\widehat{M}$. There are a variety of ways to construct the ensemble of simulators---including bootstrap and adding noise to the fitted simulator $\widehat{M}$.

Throughout, we construct the ensemble of simulators $\widehat M^{(1)}, \dots \widehat M^{(n)}$ using bootstrap-based approaches. In particular, we use both empirical and parametric bootstrapping. Empirical bootstrap resamples the episodes in $\mathcal{D}_{\textnormal{off}} = \big\{ (S_{t,1}, A_{t,1}, R_{t,1}), (S_{t,2}, A_{t,2}, R_{t,2}), \dots (S_{t,H}, A_{t,H}, R_{t,H}) \big\}_{t=1}^{T_{\textnormal{off}}}$ with replacement. In simpler models, e.g., Gaussian bandits, we use parametric bootstrap methods to construct the ensemble \citep[Chapter 23.1]{van2000asymptotic}.

\subsection{Related work}

Our main finding is a more robust and plausible algorithm selection through ensembled simulators that are uncertainty-aware.

This approach is similar to domain randomization (DR) and its data-calibrated variants \citep{tobin2017domain,peng2018sim,ramos2019bayessim,chebotar2019closing,muratore2022neural,tiboni2023dropo} that also train against a distribution of simulators.
Existing literature on DR primarily concerns the mismatch between the simulator and the deployment environment, and its purpose is to make one fixed policy robust to that mismatch.
In our setting there is no mismatch, since the offline data and the deployment come from the same environment; the spread of the ensemble reflects only the sampling variability of a finite dataset and vanishes as the data grow, and the gain arises because the candidates are online learners whose regret responds asymmetrically to the estimation error.

Ensembles are also standard inside model-based RL, where bootstrapped or probabilistic ensembles quantify epistemic uncertainty to drive exploration \citep{osband2016deep,chua2018deep} or to penalize policies that visit states where the members disagree \citep{yu2020mopo,kidambi2020morel}.
In those methods the ensemble is part of the algorithm being run and changes how it acts; in ours the ensemble sits outside the candidates, which are evaluated as black boxes, and it changes only which candidate algorithm is chosen.

A similar distinction separates our rule from robust and distributionally robust RL \citep{iyengar2005robust,nilim2005robust,eysenbach2021maximum}, which optimize a worst case over an uncertainty set by solving an optimization problem: we show that a plain average over the ensemble is already a curvature regularization of the Plug-In criterion (Proposition~\ref{prop:regularized1}), and that this average suffices for the regret gains of Section~\ref{sec:mab}.

Finally, the object being selected differs from that of offline policy selection \citep{konyushova2021active,yang2022offline,zitovsky2023revisiting},
which ranks fixed policies through fitted value functions; we rank learners that keep learning after deployment, so what the simulator must get right is their ordering by online regret.
Among the works that select online algorithms \citep{wang2022no,tang2022towards}, no prior work asks how the estimation error of a data-fitted simulator propagates into the choice of an online learning algorithm, or shows that an ensemble corrects it.
See Appendix~\ref{app:related} for a more complete review of the related literature.

\section{Plug-In vs. UA selection in multi-armed bandits}\label{sec:mab}
In this section, we characterize the cost of poor hyperparameter selection and compare the Plug-In and Uncertainty-Aware selection rules in a simple \textbf{Gaussian multi-armed bandit environment} $M^*$, in which $H=1$, $|\mathcal{S}| = 1$, $\mathcal{A} = \{ 1, 2 \}$, and rewards $R_t|A_t\!=\!a \sim \mathcal{N}(\mu_a, \sigma^2)$. More specifically, we study two common MAB algorithms with a hyperparameter $\theta$ that impacts the amount of exploration. Our results center around two key findings, which we support empirically and theoretically:
\begin{enumerate}[label=(\alph*), leftmargin=*]
    \item Underexploration can be substantially more costly than overexploration: an exploration parameter $\theta$ that is too small can lead to much larger regret than  one that is too large.
    \item Uncertainty-Aware selection tends to favor greater exploration than Plug-In selection, thereby reducing the risk of costly underexploration and leading to improved regret.
\end{enumerate}

\textbf{Bandit algorithms.}
We focus on two MAB algorithms. For both algorithms, $\theta$ controls the amount of exploration: larger values of $\theta$ induce more exploration, while smaller values induce less exploration. The first, UCB \citep{auer2002,lattimore2020,abbasi2011improved} at round $t$ pulls the arm
\begin{align}
    A_t\in\argmax_{a\in\mathcal A}\big\{\widehat\mu_a(t-1) + \sqrt{2\theta \log t / N_a(t-1)}\big\},
    \label{eqn:ucb1}
\end{align}
where $N_a(t)\!= \!\sum_{\tau=1}^{t} \ind(A_{\tau} \!=\!a)$ and $\widehat\mu_a(t)\!=\!\sum_{\tau=1}^{t} \ind(A_{\tau}\!=\!a) R_{\tau} / N_a(t)$  are the running pull count and mean reward of arm $a$, respectively. Here the hyperparameter $\theta>0$ is the \emph{exploration radius}.

The second algorithm, Thompson Sampling (TS) \citep{daniel2018tutorial} forms a Bayesian model of the reward means where we have a flat prior over arm means $\mu_a$, i.e., $P(\mu_a) \propto 1$, and rewards are conditionally Gaussian with variance $\theta$, i.e., $R_t |A_t\!=\!a, \mu_a \sim \cN(\mu_a, \theta)$. At each time $t$, the posterior distribution is then $\mu_a |\mathcal{H}_{t-1} \sim \cN( \widehat\mu_a(t-1), \, \theta / N_a(t-1) )$. TS selects actions via probability matching in that actions are selected according to the posterior probability that they are optimal:
\begin{align}
    \PP( A_t = a \mid \mathcal{H}_{t-1}) = \PP( a = \textnormal{argmax}_{a' \in \mathcal{A}} \, \mu_{a'} \mid \mathcal{H}_{t-1}).
    \label{eqn:TS}
\end{align}
We refer to the hyperparameter $\theta>0$ as the
\emph{working reward variance}.

\begin{figure}[t]
    \centering
    \vspace{-5mm}
    \includegraphics[width=\linewidth]{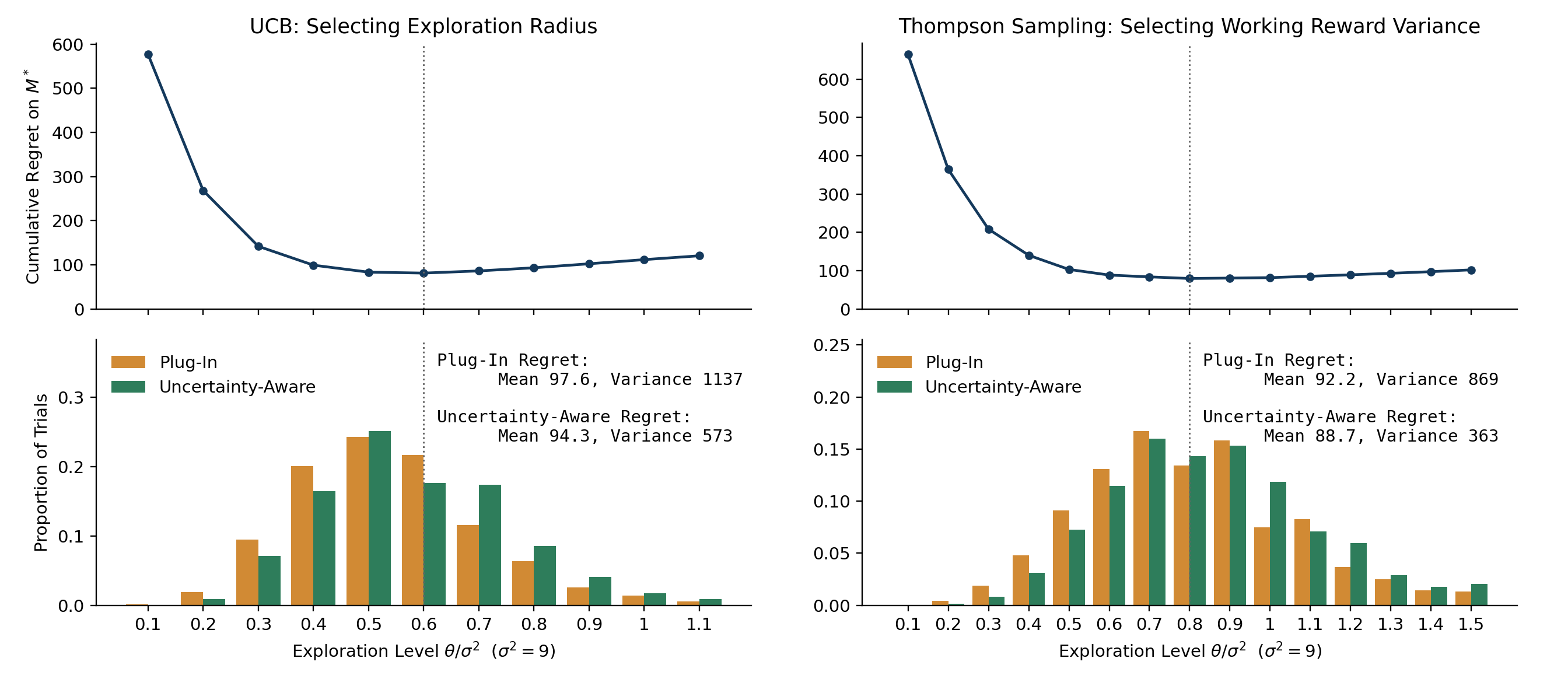}
    \vspace{-5mm}
    \caption{\textbf{Asymmetric costs of exploration and the benefit of Uncertainty-Aware selection.} Simulations are in a two-armed Gaussian bandit with $T=5000$, $T_{\textnormal{off}}=25$, margin $\mu_1 - \mu_2 = 1$, and noise standard deviation $\sigma=3$.
    \emph{Top:} Regret of UCB (left) and Thompson Sampling (right) as a function of the exploration level $\theta$ (see \eqrefff{eqn:ucb1} and \eqrefff{eqn:TS}). Regret rises far more steeply below the optimal value $\theta_T^*$ (dotted vertical line) than above it, i.e., underexploration is more costly.
    \emph{Bottom:} Using multiple samples of $\mathcal{D}_{\textnormal{off}}$,  Uncertainty-Aware selection (ensemble models constructed via bootstrap) more often selects higher values of $\theta$ compared to Plug-In, which leads to lower regret on average. Additionally, Uncertainty-Aware selection has lower variance in regret across trials.}
    \label{fig:selection-regret}
    \vspace{-5mm}
\end{figure}

\subsection{Asymmetric costs of under- and overexploration}
\label{sec:UCB-asymmetry}
For both UCB and TS, given an exploration parameter $\theta$, we measure the online performance of $\pi_\theta$ by its expected cumulative regret:
\begin{align}
    \label{eq:bandit-regret}
    \textnormal{Reg}_T(\pi_\theta, M^*) = T\mu^*(M^*)-J_T(\pi_\theta,M^*),
\end{align}
where $\mu^*(M^*) = \max_{a \in \{1,2\}} \mu_a$, and $\pi_\theta$ denotes either UCB or TS with hyperparameter $\theta$. Figure~\ref{fig:selection-regret} (top) illustrates a pronounced asymmetry: regret increases much more sharply when $\theta$ falls below the best-performing choice than when it increases above it. We observe the same pattern across a range of environment configurations in Appendix~\ref{app:plane-robust}.

The following result provides a theoretical characterization of this asymmetry for both UCB and TS.
\begin{theorem}[Asymmetric regret of UCB and TS]
\label{thm::UCB-regret-main}
Let $M^*$ be a two-armed Gaussian bandit with $\mu_1>\mu_2$ and $\sigma^2>0$. Let $\pi_\theta$ denote either the UCB algorithm (\ref{eqn:ucb1}) or the TS algorithm (\ref{eqn:TS}) with hyperparameter $\theta>0$. Then:
\begin{itemize}
\vspace{-3mm}
\item[(i)] For any $\theta\in(0, \sigma^2)$, $
\lim_{T\to\infty}
\frac{\log \textnormal{Reg}_T(\pi_\theta,M^*) }{\log T}=
1-\frac{\theta}{\sigma^2}.
$
\item[(ii)] For any $\theta\geq \sigma^2$,
$
\lim_{T\to\infty}
\frac{\textnormal{Reg}_T(\pi_\theta,M^*)}{\log T}=
\frac{2\theta}{\mu_1-\mu_2}.
$
\end{itemize}
\end{theorem}

Theorem~\ref{thm::UCB-regret-main} identifies $\sigma^2$ as a critical exploration level for both algorithms. Below this threshold, regret grows polynomially as $T^{1-\theta/\sigma^2+o(1)}$; at or above the threshold, it grows only logarithmically, with leading constant $2\theta/(\mu_1\!-\!\mu_2)$. Thus, insufficient exploration can change the order of regret from logarithmic to polynomial, whereas additional exploration above the threshold affects only the leading constant of the logarithmic regret. This asymmetry is also reflected in the finite-horizon results in Figure~\ref{fig:selection-regret} (top).

Theorem~\ref{thm::UCB-regret-main} is closely related to the results in \citet{fan2025fragility} and \citet{fan2022typical}. In Appendices~\ref{app:ucb} and~\ref{app:ts}, we establish stronger results for both algorithms: the regret exponent converges \emph{uniformly} in $\theta$ over compact subsets of $(0,\infty)$, while the leading-order logarithmic asymptotics hold uniformly over compact subsets of $(\sigma^2,\infty)$. We also provide associated finite-horizon regret bounds that are uniform in both $\theta$ and the suboptimality gap (Lemma \ref{lem:ucb-bounds} and \ref{lem:fixed-env-TS}). These uniformity results are important for analyzing data- and horizon-dependent hyperparameter selection.
Our uniform polynomial lower bounds for both algorithms exploit low initial rewards that delay further sampling of the optimal arm, following the intuition of \citet{fan2025fragility}. For the upper bounds, we adapt a pull-count decomposition argument to obtain the required uniform control; the TS analysis further builds on the posterior-odds comparison of \citet{agrawal2013further}, for which we derive uniform bounds on truncated Gaussian odds to handle underexploration.

\subsection{Uncertainty-Aware selection favors exploration and improves regret}

\textbf{Uncertainty-Aware selection lowers both the mean and the variance of regret (Figure~\ref{fig:selection-regret}).}
We empirically compare the Plug-In and Uncertainty-Aware selection rules for UCB and Thompson Sampling. We use a Gaussian reward two-armed bandit simulation environment $M^* = \left( \cN(\mu_1,\sigma_1^2), \mathcal N(\mu_2,\sigma_2^2) \right)$.
The Plug-In selection rule fits $\widehat M = \big( \cN(\widehat\mu_1,\widehat\sigma^2), \cN(\widehat\mu_2,\widehat\sigma^2) \big)$, where $\widehat\mu_a$ is the sample mean of arm $a\in\{1,2\}$ and $\widehat\sigma^2$ is the reward variance pooled across the two arms, both estimated using $\mathcal{D}_{\textnormal{off}}$. The Uncertainty-Aware selection rule forms an ensemble of simulators $\widehat M^{(1)}, \dots, \widehat M^{(n)}$ by a parametric bootstrap: member $i$ is $\widehat M^{(i)} = \big( \cN\big( \widehat\mu_1^{(i)}, (\widehat\sigma^{(i)})^2 \big), \cN\big( \widehat\mu_2^{(i)}, (\widehat\sigma^{(i)})^2 \big) \big)$, where $\widehat\mu_a^{(i)}$ and $(\widehat\sigma^{(i)})^2$ are the same estimates computed on a fresh dataset of $T_{\textnormal{off}}$ pairs drawn from $\widehat M$ under the behavior policy (Appendix~\ref{app:plane-rules}).

Figure~\ref{fig:selection-regret} (bottom) shows that Uncertainty-Aware selection has lower regret compared to Plug-In. The lower regret under UA selection is because it selects higher values of exploration compared to Plug-In. Specifically, for Plug-In selection $61.0\%$ of the regret is due to selecting a $\theta$ below $\theta_T^*$, while for Uncertainty-Aware selection only $50.6\%$ of regret is due to selecting a $\theta$ below $\theta_T^*$. The Plug-In rule therefore underperforms not because its estimate of $\theta$ is noisy, but because the errors that fall on the side of too little exploration are the expensive ones, and their cost compounds with the horizon. Furthermore, the variance of the cumulative regret across trials is dramatically lower under Uncertainty-Aware selection compared to Plug-In (573 vs 1137 for UCB; 363 vs 869 for TS). We discuss how Uncertainty-Aware selection can be interpreted as a curvature regularized version of the Plug-In objective in Section \ref{sec:regularization}.

\textbf{Uncertainty-Aware selection leads to greater exploration (Figure~\ref{fig:UCB-exploration}).}
Figure~\ref{fig:UCB-exploration} illustrates the mechanism behind improvement in regret under UA selection. Across much of the region of $(\widehat\Delta,\widehat\sigma)$ values supported by the offline data, the Uncertainty-Aware rule selects a normalized exploration parameter $\theta/\sigma^2$ that is at least as large as that selected by the Plug-In rule. Thus, accounting for uncertainty in the fitted simulator systematically guards against selections that are overly exploitative under a favorable but inaccurate point estimate. This effect is not simply a uniform increase in exploration: it is concentrated in regions where uncertainty about the reward gap and noise level can change the relative performance of candidate settings. By avoiding these costly underexploratory selections while making similar choices when the fitted model is sufficiently informative, the Uncertainty-Aware rule reduces both the mean and variability of regret observed above.

\begin{figure}[t]
    \centering
    \vspace{-3mm}
    \includegraphics[width=\linewidth]{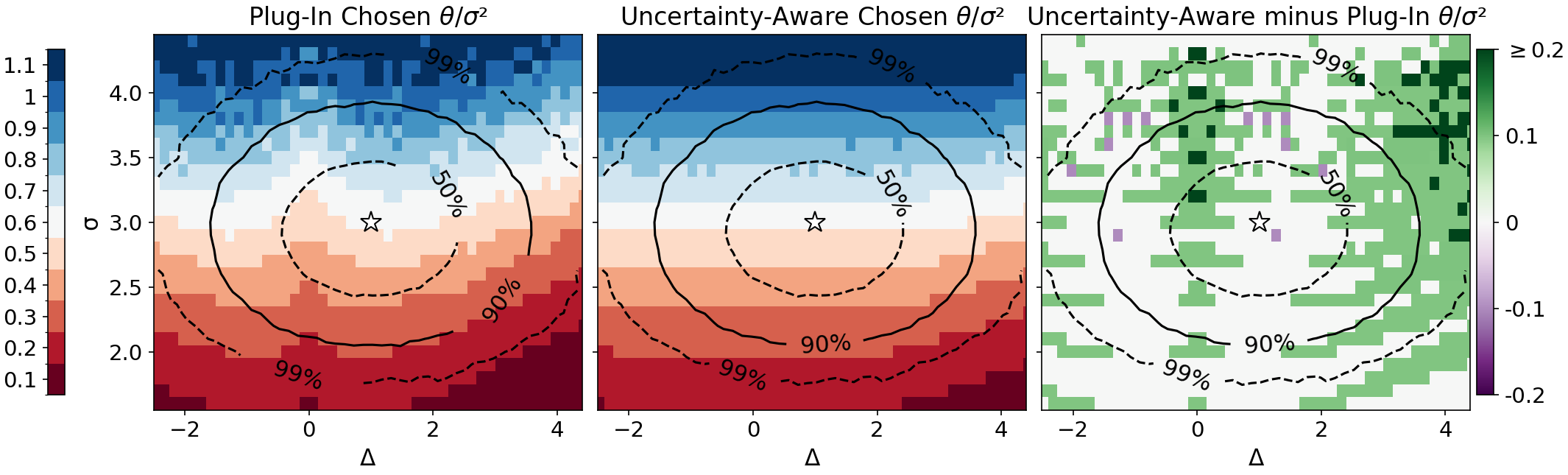}
    \vspace{-7mm}
    \caption{\textbf{Uncertainty-Aware (UA) selection leads to greater exploration than Plug-In under UCB. } We independently sample $10^6$ $\mathcal{D}_{\textnormal{off}}$ datasets and compute an estimate of the margin $\Delta = \mu_1 - \mu_2$ and the reward standard deviation $\sigma$. The star depicts the true value of $\Delta = 1$ and $\sigma = 3$ in the environment. The contour lines at 50\% depicts where 50\% of the $10^6$ estimated $\widehat{\Delta}, \widehat{\sigma}$ are; similarly for the 90\%, and 99\% contour lines.
    We depict the selected value of $\theta/\sigma^2$ for each combination of $(\Delta, \sigma)$ under the Plug-In selection rule (left) and that for the UA selection rule (center). In the right most plot, we depict the difference in chosen $\theta/\sigma^2$ for the Plug-In and UA rules.
    The plot shows that the UA selection rule leads to similar or greater exploration than Plug-In for most values of $(\Delta, \sigma)$.
    Combined with the asymmetric cost of underexploration (Section~\ref{sec:UCB-asymmetry}), this helps explain how UA decreases regret compared to Plug-In.
    }
    \label{fig:UCB-exploration}
    \vspace{-3mm}
\end{figure}

We next provide theoretical support for these empirical findings. For simplicity, we assume the offline dataset $\mathcal{D}_{\textnormal{off}}$ was collected by a behavior policy that samples each arm $T_{1,\textnormal{off}} := \lfloor T_{\textnormal{off}}/2\rfloor$ times, and let $ \widehat M=\big(\cN(\widehat\mu_1,\widehat\sigma_1^2), \cN(\widehat\mu_2,\widehat\sigma_2^2)\big)$ denote the fitted environment. For either UCB or TS, let $\pi_\theta$ denote the algorithm with hyperparameter $\theta$. For the selection analysis, we restrict the candidate parameters to a compact interval $\Theta=[\underline\theta,\overline\theta]$ satisfying $0<\underline\theta<\sigma^2<\overline\theta<\infty$.

The Plug-In rule selects $\widehat\theta_{\plug,T} \in \argmax_{\theta\in\Theta} J_T(\pi_\theta,\widehat M)$. For the theoretical analysis of the Uncertainty-Aware rule, we consider a simplified parametric-bootstrap ensemble \citep{van2000asymptotic} that perturbs the fitted arm means according to their estimated sampling uncertainty. Conditional on
$\mathcal D_{\mathrm{off}}$, each ensemble member is
\begin{align}
    \widehat M^{(i)}
    = \left( \mathcal N\big(\widehat\mu_1+ \widehat\sigma_1/\sqrt{T_{1, \off}}\cdot Z_1^{(i)}, \, \widehat\sigma_1^2 \big), ~
    \mathcal N\big(
    \widehat\mu_2 + \widehat\sigma_2 / \sqrt{T_{1, \off}}\cdot Z_2^{(i)}, \, \widehat\sigma_2^2 \big)
    \right),
    \label{eq:perturbed-bandit}
\end{align}
where $Z_1^{(i)}, Z_2^{(i)} \stackrel{\textnormal{iid}}{\sim}\mathcal N(0,1)$. We analyze the population counterpart of the finite-ensemble UA objective
in (\ref{eqn:ensembleObjective}), obtained in the large-ensemble limit:
\begin{equation}\label{eq:theta-UA-ucb}
    \widehat\theta_{\UA,T}
    := \argmax_{\theta\in\Theta} \E \Big[
    J_T(\pi_\theta,\widehat M^{(i)}) \, \Big| \, \mathcal D_{\textnormal{off}} \Big].
\end{equation}
The maximizers $\widehat\theta_{\plug,T}$ and $\widehat\theta_{\UA,T}$ exist almost surely over $\mathcal D_{\textnormal{off}}$ (see Appendices \ref{app:ucb}, \ref{app:ts}).

Our first result is Proposition \ref{prop:ucb-exploration} which shows that, by accounting for uncertainty in the fitted environment, the UA selection rule asymptotically never selects less exploration than the Plug-In rule.

\begin{proposition}[Uncertainty-Aware selection leads to higher exploration]
Consider a two-armed Gaussian bandit $M^*$ with $\mu_1>\mu_2$ and $\sigma^2>0$. For either UCB or TS, almost surely over $\mathcal D_{\textnormal{off}}$, $\liminf_{T\to\infty}\widehat\theta_{\UA,T}\geq\lim_{T\to\infty}\widehat\theta_{\plug,T}$.
    \label{prop:ucb-exploration}
\end{proposition}

The greater exploration induced by UA selection translates into a more favorable asymptotic regret growth rate. To state this comparison, define the asymptotic log ratio between the cumulative regret growth rates of Plug-In and UA selection:
\begin{equation}
    \label{eq:G}
    G(\mathcal D_{\textnormal{off}})
    := \liminf_{T\to\infty}
    \frac{ \log \textnormal{Reg}_T
    (\pi_{\widehat\theta_{\plug,T}},M^*)
    }{ \log \textnormal{Reg}_T
    (\pi_{\widehat\theta_{\UA,T}},M^*) }.
\end{equation}

\begin{theorem}\label{thm::UCB-selection-main-paper}
In a two-armed Gaussian bandit $M^*$ with $\mu_1>\mu_2$ and $\sigma^2>0$, for either UCB or TS:
\begin{itemize}
    \item[(i)] For any $T_{\textnormal{off}}\geq 4$,   $G(\mathcal D_{\textnormal{off}})\geq 1$ almost surely.

    \item[(ii)] $\liminf_{T_{\textnormal{off}}\to\infty}\mathbb P\!\left(G(\mathcal D_{\textnormal{off}})=\infty\right)\geq \frac14$.
\end{itemize}
\end{theorem}

Proposition~\ref{prop:ucb-exploration} provides the intuition for Theorem~\ref{thm::UCB-selection-main-paper}(i): UA asymptotically selects at least as much exploration as Plug-In, while the regret exponent decreases with the exploration parameter. A formal argument requires uniform regret bounds for the data-dependent sequences $\widehat\theta_{\plug,T}$ and $\widehat\theta_{\UA,T}$. Part~(ii) further shows that UA can strictly improve the regret growth rate. The proof identifies an event whose probability converges to $1/4$ as $T_{\off}\to\infty$, on which Plug-In selects an exploration parameter below the critical level $\sigma^2$, whereas UA selects above it. By the asymmetric regret analysis in Section~\ref{sec:UCB-asymmetry}, the resulting deployment regrets are polynomial and logarithmic in $T$, respectively, and hence $G(\mathcal D_{\off})=\infty$. Proofs for UCB and TS are given in Appendices~\ref{app:ucb} and~\ref{app:ts}, respectively.

\section{Bandits and beyond: Uncertainty-Aware simulation selection as a curvature regularization}
\label{sec:regularization}

Motivated by the bandit results, we next give a local characterization of how UA selection differs from Plug-In selection. This characterization does not rely on the bandit setting: it applies whenever an offline dataset is used to fit a simulator and candidate decision procedures are evaluated within that simulator. The simulator may be an explicit dynamics-and-reward model or a learned latent world model, and the evaluation criterion $J_T$ may be application-specific, e.g., could incorporate risk-sensitive utility, safety, or another deployment criterion.

Plug-In selection ranks candidates using a single fitted simulator, whereas UA selection ranks them by average performance across plausible perturbations of that simulator. We show that, locally, this averaging induces a curvature-dependent adjustment to the Plug-In objective. This provides one interpretation of UA selection as regularization at the selection stage: it can downweight candidates whose favorable Plug-In evaluation is sensitive to uncertainty in the fitted simulator, without modifying the candidate algorithm
itself.

\textbf{Curvature correction for simulator uncertainty.}
To make this interpretation precise, consider a parametric simulator family $\{M_\lambda:\lambda\in\mathbb{R}^d\}$ and suppose that the fitted simulator is $\widehat M=M_{\widehat\lambda}$. In particular, $\lambda$ may parameterize transition dynamics, rewards, or components of a latent world model.

Let an ensemble simulator take the form $M^{(i)} = M_{\widehat\lambda+\epsilon^{(i)}}$ where, conditional on $\mathcal D_{\mathrm{off}}$, $\E[\epsilon^{(i)}\mid\mathcal D_{\mathrm{off}}]=0$, $\E[\epsilon^{(i)}(\epsilon^{(i)})^\top \mid\mathcal D_{\mathrm{off}}]=\Sigma$, and $\E[\|\epsilon^{(i)}\|_2^3 \mid \mathcal{D}_{\mathrm{off}}] < \infty$.
Here, the perturbations $\epsilon^{(i)}$ may arise from an empirical bootstrap, a parametric bootstrap, posterior samples, or another ensemble construction intended to represent uncertainty in the fitted simulator.

The Plug-In selection rule selects $\widehat\theta_{\plug,T} \in \argmax_{\theta\in\Theta} J_T(\pi_\theta,M_{\widehat\lambda})$. The UA objective here maximizes the population counterpart of the finite-ensemble UA objective from \eqrefff{eqn:ensembleObjective}:
\begin{align}
    \label{eqn:expectedEnsemble}
    \widehat\theta_{\UA,T} \in \argmax_{\theta\in\Theta} \E \left[
        J_T \big( \pi_\theta,
        M_{\widehat\lambda + \epsilon^{(i)}} \big) \,\middle|\, \mathcal D_{\mathrm{off}}
    \right].
\end{align}

\begin{proposition}[Ensembles induce a curvature correction]
    \label{prop:regularized1}
    Suppose $\lambda \mapsto J_T(\pi_\theta, M_\lambda)$ is twice continuously differentiable
    and its Hessian is $L$-Lipschitz, i.e., $\| \nabla_\lambda^2 J_T(\pi_\theta, M_{\lambda}) - \nabla_\lambda^2 J_T(\pi_\theta, M_{\lambda'}) \|_{\mathrm{op}} \le L \| \lambda - \lambda' \|_2$ for all $\lambda, \lambda'$, where $\| \cdot \|_{\mathrm{op}}$ is the operator norm. Then
    \begin{align}
        \Big| \E \big[ J_T(\pi_\theta, M_{\widehat \lambda+\epsilon^{(i)}}) \mid \mathcal{D}_{\mathrm{off}} \big]
        - J_T(\pi_\theta, M_{\widehat \lambda}) - \tfrac{1}{2} \operatorname{Tr} \big( \nabla_\lambda^2 J_T(\pi_\theta, M_{\widehat \lambda}) \Sigma \big) \Big|
        \le \frac{L}{6} \, \E \big[ \| \epsilon^{(i)} \|_2^3 \mid \mathcal{D}_{\mathrm{off}} \big].
        \label{eqn:curvature-correction}
    \end{align}
    \vspace{-5mm}
\end{proposition}
Proposition~\ref{prop:regularized1} shows that UA selection is locally a curvature-adjusted version of Plug-In selection. In directions of simulator uncertainty along which a candidate's value is locally concave, the correction in \eqrefff{eqn:curvature-correction} is negative: performance lost under unfavorable perturbations exceeds performance gained under equally sized favorable perturbations. UA selection therefore lowers the ranking of candidates whose high Plug-In score is brittle to uncertainty in the fitted simulator.
UA selection favors candidates according to their value curvature in the uncertain directions of the simulator. This is why we view UA selection as a form of regularization rather than as a rule that uniformly penalizes exploration, optimism, or any particular hyperparameter value. The proof is in Appendix \ref{app:regularization}.

\section{Deep RL experiments}
\label{sec:experiments}

We apply UA selection to continuous control in MuJoCo \citep{todorov2012mujoco}, tuning the reward-shaping weights of PPO \citep{schulman2017proximal, raffin2021sb3} in the setup of \citet{morse2025}, on simulators whose physics parameters are identified from offline data.

\textbf{Tasks and online learners.}
The tasks are \texttt{Hopper} and \texttt{Walker2d} from Gymnasium \citep{towers2024gymnasium}, in which a planar robot with three (\texttt{Hopper}) or six (\texttt{Walker2d}) actuated joints must move forward without falling, with horizon $H = 250$.
Reward shaping trains the learner on a hand-designed reward with several terms, and the weights on those terms strongly affect the trained policy \citep{morse2025}.
Here the tuning parameter $\theta = (\theta_{\textnormal{surv}}, \theta_{\textnormal{fwd}}, \theta_{\textnormal{ctrl}})$ collects the weights of three handcrafted heuristic reward terms, a survival bonus, a forward-progress term and a control-cost penalty, with $\theta_{\textnormal{surv}} = 1$ fixed (Appendix~\ref{app:deeprl-reward}).
A larger ratio $\theta_{\textnormal{fwd}} / \theta_{\textnormal{ctrl}}$ pays more for moving forward and charges less for actuation, so the trained policy runs harder; we call such a $\theta$ more aggressive.
The candidate algorithm $\pi_\theta$ is PPO trained on the reward with weights $\theta$, with every other setting held fixed (Appendix~\ref{app:deeprl-learner}), and the candidate set $\Theta$ is a grid of $25$ weight vectors on \texttt{Hopper} and $30$ on \texttt{Walker2d} (Appendix~\ref{app:deeprl-grid}).

Following \citet{morse2025},
the performance is measured by a success score that is zero if the average squared actuation over the episode exceeds a bound $c_{\max}$, and otherwise rises linearly from $0$ to $1$ as the distance traveled increases from a threshold $x_{\textnormal{lo}}$ to a threshold $x_{\textnormal{hi}}$ (Appendix~\ref{app:deeprl-reward}). We evaluate the performance of the trained policy at $T = 3 \times 10^5$ environment steps.
Every reported value is a Monte Carlo estimate over $50$ evaluation episodes, which we call the success score.

\textbf{Simulator and offline data.}
The simulator family is $\{M_\lambda\}$, where $M_\lambda$ is the MuJoCo model of the robot with $\lambda = (\lambda_{\textnormal{mass}}, \lambda_{\textnormal{fric}}, \lambda_{\textnormal{damp}}) \in \mathbb{R}^3$ multiplying the body masses, the friction coefficients and the joint damping of a baseline model.
Estimating $\lambda$ from logged data is standard practice in robotics, known as system identification \citep{ljung1999system, tiboni2023dropo}.
The true environment is $M^* = M_{\lambda^*}$ with $\lambda^* = (1.5, 0.7, 1.4)$, a robot that is heavier, slipperier and more damped than the baseline model; the practitioner typically knows the model but not $\lambda^*$. We add Gaussian noise to the observations, in the offline data, in every simulator and on the true environment, so that identifying $\lambda$ from $\mathcal{D}_{\textnormal{off}}$ is a nontrivial task (Appendix~\ref{app:deeprl-plant}).

\begin{table*}[t]
\vspace{-5mm}
\centering
\caption{Success score of the weight vector selected by the Plug-In and UA rules on the true environment, mean $\pm$ standard error and $10$th percentile over $300$ paired replications; higher is better.
The better of the two rules is in bold, and $^*$ marks a significant difference.
\\}
\label{tab:deeprl-simex}
\small
\resizebox{\textwidth}{!}{%
\begin{tabular}{@{}lcccc@{\hspace{1.5em}}cccc@{}}
\toprule
& \multicolumn{4}{c}{\texttt{Hopper}, $J_T(\pi_{\theta_T^*}, M^*) = 0.991$}
& \multicolumn{4}{c}{\texttt{Walker2d}, $J_T(\pi_{\theta_T^*}, M^*) = 0.689$} \\
\cmidrule(lr){2-5}\cmidrule(l){6-9}
$T_{\textnormal{off}}$
& Plug-In & UA & $q_{10}$ Plug-In & $q_{10}$ UA
& Plug-In & UA & $q_{10}$ Plug-In & $q_{10}$ UA \\
\midrule
$20$
& $0.933 \pm 0.007$ & $\mathbf{0.956 \pm 0.005}^{*}$ & $0.764$ & $\mathbf{0.930}$
& $0.585 \pm 0.006$ & $\mathbf{0.594 \pm 0.006}~~$ & $0.456$ & $0.456$ \\
$80$
& $0.965 \pm 0.003$ & $\mathbf{0.971 \pm 0.003}~~$ & $0.930$ & $0.930$
& $0.594 \pm 0.005$ & $\mathbf{0.616 \pm 0.005}^{*}$ & $0.456$ & $\mathbf{0.517}$ \\
$320$
& $0.976 \pm 0.002$ & $\mathbf{0.981 \pm 0.002}~~$ & $0.930$ & $\mathbf{0.936}$
& $0.590 \pm 0.005$ & $\mathbf{0.610 \pm 0.005}^{*}$ & $0.456$ & $\mathbf{0.527}$ \\
\bottomrule
\end{tabular}%
}
\vspace{-10pt}
\end{table*}

The offline dataset $\mathcal{D}_{\textnormal{off}}$ consists of $T_{\textnormal{off}} \in \{20, 80, 320\}$ episodes collected on $M^*$ by a behavior policy $\pi_b$ that draws actions uniformly at random. For each offline episode $t = 1, \dots, T_{\textnormal{off}}$ we fit $\widehat \lambda_t$ by nonlinear least squares over its transitions \citep{tiboni2023dropo} combined with a bias-correction method called SIMEX \citep{cook1994simex, carroll2006measurement}.
The Plug-In simulator is $\widehat M = M_{\widehat\lambda}$ with $\widehat\lambda = T_{\textnormal{off}}^{-1} \sum_{t=1}^{T_{\textnormal{off}}} \widehat\lambda_t$, the average of the per-episode estimates.
The UA ensemble is the empirical bootstrap over episodes.
Member $i = 1, \dots, n = 12$ is $\widehat M^{(i)} = M_{\widehat\lambda^{(i)}}$ with $\widehat\lambda^{(i)} = T_{\textnormal{off}}^{-1} \sum_{t \in \mathcal{I}^{(i)}} \widehat\lambda_t$, where $\mathcal{I}^{(i)}$ is a multiset of $T_{\textnormal{off}}$ episodes.
We give both rules the same budget of $K = 12$ training runs per $\theta$.
We report the success score of the selected $\widehat\theta$ on the true environment, averaged over $300$ replications, each of which draws a fresh $\mathcal{D}_{\textnormal{off}}$ that both rules then see (Appendix~\ref{app:deeprl-protocol}).

\textbf{Results.}
Table~\ref{tab:deeprl-simex} reports the six cells, two tasks by three offline dataset sizes.
UA selection has the higher mean success score in every cell, by $0.005$ to $0.023$, significantly in three.
Its $10$th percentile (last two columns) is higher in four cells and equal in the other two.
Figure~\ref{fig:cond-picks-deeprl} in Appendix~\ref{app:deeprl-moves} validates the regularization effect of UA (Section~\ref{sec:regularization}).

\section{Discussion}

Our setting isolates simulator-estimation uncertainty by considering offline data and online deployment from the same environment. This is deliberately complementary to work on distribution shift, where techniques such as domain randomization can improve robustness by accounting for variation between training and deployment conditions. Our results show that simulator ensembles can improve algorithm selection even in the absence of such shift: averaging across plausible simulators mitigates uncertainty arising solely from limited offline data. An important direction for future work is to study how Uncertainty-Aware selection and distribution-shift methods can be combined when the deployment environment may differ systematically from the environment that generated the offline data. Other directions include extending the theoretical analysis beyond bandits to MDPs and a broader range of online algorithms, and studying richer algorithm-design choices such as architectures, state representations, and reward specifications. Finally, practical progress will require ensemble-construction methods that quantify uncertainty reliably in complex simulators.

\bibliographystyle{plainnat}
\bibliography{main}

@inproceedings{agrawal2013further,
  title={Further optimal regret bounds for {Thompson} sampling},
  author={Agrawal, Shipra and Goyal, Navin},
  booktitle={Proceedings of the Sixteenth International Conference on Artificial Intelligence and Statistics (AISTATS)},
  series={Proceedings of Machine Learning Research},
  volume={31},
  pages={99--107},
  year={2013}
}

@article{fan2022typical,
  title={The typical behavior of bandit algorithms},
  author={Fan, Lin and Glynn, Peter W.},
  journal={arXiv preprint arXiv:2210.05660},
  year={2022}
}

@article{osband2016deep,
  title={Deep exploration via bootstrapped DQN},
  author={Osband, Ian and Blundell, Charles and Pritzel, Alexander and Van Roy, Benjamin},
  journal={Advances in neural information processing systems},
  volume={29},
  year={2016}
}

@book{efron2000introduction,
  title={An Introduction to the Bootstrap},
  author={Efron, Bradley and Tibshirani, Robert J.},
  series={Monographs on Statistics and Applied Probability},
  volume={57},
  publisher={Chapman and Hall/CRC},
  address={Boca Raton, FL},
  year={1994}
}

@inproceedings{eysenbach2021maximum,
  title={Maximum entropy {RL} (provably) solves some robust {RL} problems},
  author={Eysenbach, Benjamin and Levine, Sergey},
  booktitle={International Conference on Learning Representations (ICLR)},
  year={2022}
}

@article{nilim2005robust,
  title={Robust control of Markov decision processes with uncertain transition matrices},
  author={Nilim, Arnab and El Ghaoui, Laurent},
  journal={Operations Research},
  volume={53},
  number={5},
  pages={780--798},
  year={2005},
  publisher={INFORMS}
}

@article{iyengar2005robust,
  title={Robust dynamic programming},
  author={Iyengar, Garud N},
  journal={Mathematics of Operations Research},
  volume={30},
  number={2},
  pages={257--280},
  year={2005},
  publisher={INFORMS}
}

@article{daniel2018tutorial,
  title={A tutorial on {Thompson} sampling},
  author={Russo, Daniel J. and Van Roy, Benjamin and Kazerouni, Abbas and Osband, Ian and Wen, Zheng},
  journal={Foundations and Trends in Machine Learning},
  volume={11},
  number={1},
  pages={1--96},
  year={2018},
  publisher={Now Publishers}
}

@article{gordon1941values,
  title={Values of Mills' ratio of area to bounding ordinate and of the normal probability integral for large values of the argument},
  author={Gordon, Robert D.},
  journal={The Annals of Mathematical Statistics},
  volume={12},
  number={3},
  pages={364--366},
  year={1941}
}

@article{fan2025fragility,
  title={The fragility of optimized bandit algorithms},
  author={Fan, Lin and Glynn, Peter W},
  journal={Operations Research},
  volume={73},
  number={6},
  pages={3173--3198},
  year={2025},
  publisher={INFORMS}
}

@article{muratore2022robot,
  title={Robot learning from randomized simulations: A review},
  author={Muratore, Fabio and Ramos, Fabio and Turk, Greg and Yu, Wenhao and Gienger, Michael and Peters, Jan},
  journal={Frontiers in Robotics and AI},
  volume={9},
  pages={799893},
  year={2022},
  publisher={Frontiers Media SA}
}

@article{aastrom1971system,
  title={System identification—a survey},
  author={{\AA}str{\"o}m, Karl Johan and Eykhoff, Peter},
  journal={Automatica},
  volume={7},
  number={2},
  pages={123--162},
  year={1971},
  publisher={Elsevier}
}

@article{li2025uncertainty,
  title={Uncertainty-aware robotic world model makes offline model-based reinforcement learning work on real robots},
  author={Li, Chenhao and Krause, Andreas and Hutter, Marco},
  journal={arXiv preprint arXiv:2504.16680},
  year={2025}
}

@article{chua2018deep,
  title={Deep reinforcement learning in a handful of trials using probabilistic dynamics models},
  author={Chua, Kurtland and Calandra, Roberto and McAllister, Rowan and Levine, Sergey},
  journal={Advances in neural information processing systems},
  volume={31},
  year={2018}
}

@article{tang2022towards,
  title={Towards Data-Driven Offline Simulations for Online Reinforcement Learning},
  author={Tang, Shengpu and Frujeri, Felipe Vieira and Misra, Dipendra and Lamb, Alex and Langford, John and Mineiro, Paul and Kochman, Sebastian},
  journal={arXiv preprint arXiv:2211.07614},
  year={2022}
}

@inproceedings{mandel2016offline,
  title={Offline evaluation of online reinforcement learning algorithms},
  author={Mandel, Travis and Liu, Yun-En and Brunskill, Emma and Popovi{\'c}, Zoran},
  booktitle={Proceedings of the AAAI Conference on Artificial Intelligence},
  volume={30},
  year={2016}
}

@article{dulac2019challenges,
  title={Challenges of real-world reinforcement learning},
  author={Dulac-Arnold, Gabriel and Mankowitz, Daniel and Hester, Todd},
  journal={arXiv preprint arXiv:1904.12901},
  year={2019}
}

@inproceedings{tobin2017domain,
  title={Domain randomization for transferring deep neural networks from simulation to the real world},
  author={Tobin, Josh and Fong, Rachel and Ray, Alex and Schneider, Jonas and Zaremba, Wojciech and Abbeel, Pieter},
  booktitle={IEEE/RSJ International Conference on Intelligent Robots and Systems (IROS)},
  pages={23--30},
  year={2017},
  organization={IEEE}
}

@article{tiboni2023dropo,
  title={{DROPO}: Sim-to-real transfer with offline domain randomization},
  author={Tiboni, Gabriele and Arndt, Karol and Kyrki, Ville},
  journal={Robotics and Autonomous Systems},
  volume={166},
  pages={104432},
  year={2023},
  publisher={Elsevier}
}

@article{auer2002,
  title={Finite-time analysis of the multiarmed bandit problem},
  author={Auer, Peter and Cesa-Bianchi, Nicol{\`o} and Fischer, Paul},
  journal={Machine Learning},
  volume={47},
  number={2--3},
  pages={235--256},
  year={2002},
  publisher={Springer}
}

@inproceedings{wunder2010classes,
  title={Classes of multiagent {Q}-learning dynamics with $\epsilon$-greedy exploration},
  author={Wunder, Michael and Littman, Michael L. and Babes, Monica},
  booktitle={Proceedings of the 27th International Conference on Machine Learning (ICML)},
  pages={1167--1174},
  year={2010}
}

@book{puterman2014markov,
  title={Markov decision processes: discrete stochastic dynamic programming},
  author={Puterman, Martin L},
  year={2014},
  publisher={John Wiley \& Sons}
}

@book{lattimore2020,
  title={Bandit Algorithms},
  author={Lattimore, Tor and Szepesv{\'a}ri, Csaba},
  publisher={Cambridge University Press},
  year={2020}
}

@article{morse2025,
  title={Automatic Reward Shaping from Multi-Objective Human Heuristics},
  author={Xie, Yuqing and Chen, Jiayu and Tang, Wenhao and Zhang, Ya and Yu, Chao and Wang, Yu},
  journal={arXiv preprint arXiv:2512.15120},
  year={2025}
}

@article{nesterov2006cubic,
  title={Cubic regularization of Newton method and its global performance},
  author={Nesterov, Yurii and Polyak, Boris T},
  journal={Mathematical programming},
  volume={108},
  number={1},
  pages={177--205},
  year={2006},
  publisher={Springer}
}

@book{ljung1999system,
  title={System Identification: Theory for the User},
  author={Ljung, Lennart},
  edition={2nd},
  publisher={Prentice Hall},
  year={1999}
}

@article{abbasi2011improved,
  title={Improved algorithms for linear stochastic bandits},
  author={Abbasi-Yadkori, Yasin and P{\'a}l, D{\'a}vid and Szepesv{\'a}ri, Csaba},
  journal={Advances in neural information processing systems},
  volume={24},
  year={2011}
}

@article{chen2022reinforcement,
  title={Reinforcement learning for selective key applications in power systems: Recent advances and future challenges},
  author={Chen, Xin and Qu, Guannan and Tang, Yujie and Low, Steven and Li, Na},
  journal={IEEE Transactions on Smart Grid},
  volume={13},
  number={4},
  pages={2935--2958},
  year={2022},
  publisher={IEEE}
}

@article{zhang2018review,
  title={Review on the research and practice of deep learning and reinforcement learning in smart grids},
  author={Zhang, Dongxia and Han, Xiaoqing and Deng, Chunyu},
  journal={CSEE Journal of Power and Energy Systems},
  volume={4},
  number={3},
  pages={362--370},
  year={2018},
  publisher={CSEE}
}

@article{figueroa2021adaptive,
  title={Adaptive learning algorithms to optimize mobile applications for behavioral health: guidelines for design decisions},
  author={Figueroa, Caroline A. and Aguilera, Adrian and Chakraborty, Bibhas and Modiri, Arghavan and Aggarwal, Jai and Deliu, Nina and Sarkar, Urmimala and Williams, Joseph Jay and Lyles, Courtney R.},
  journal={Journal of the American Medical Informatics Association},
  volume={28},
  number={6},
  pages={1225--1234},
  year={2021}
}

@inproceedings{henderson2018deep,
  title={Deep reinforcement learning that matters},
  author={Henderson, Peter and Islam, Riashat and Bachman, Philip and Pineau, Joelle and Precup, Doina and Meger, David},
  booktitle={Proceedings of the AAAI conference on artificial intelligence},
  volume={32},
  year={2018}
}

@inproceedings{peng2018sim,
  title={Sim-to-real transfer of robotic control with dynamics randomization},
  author={Peng, Xue Bin and Andrychowicz, Marcin and Zaremba, Wojciech and Abbeel, Pieter},
  booktitle={IEEE International Conference on Robotics and Automation (ICRA)},
  pages={3803--3810},
  year={2018},
  organization={IEEE}
}

@inproceedings{ramos2019bayessim,
  title={{BayesSim}: Adaptive domain randomization via probabilistic inference for robotics simulators},
  author={Ramos, Fabio and Possas, Rafael Carvalhaes and Fox, Dieter},
  booktitle={Robotics: Science and Systems (RSS)},
  year={2019}
}

@inproceedings{chebotar2019closing,
  title={Closing the sim-to-real loop: Adapting simulation randomization with real world experience},
  author={Chebotar, Yevgen and Handa, Ankur and Makoviychuk, Viktor and Macklin, Miles and Issac, Jan and Ratliff, Nathan and Fox, Dieter},
  booktitle={IEEE International Conference on Robotics and Automation (ICRA)},
  pages={8973--8979},
  year={2019},
  organization={IEEE}
}

@inproceedings{todorov2012mujoco,
  title={{MuJoCo}: A physics engine for model-based control},
  author={Todorov, Emanuel and Erez, Tom and Tassa, Yuval},
  booktitle={IEEE/RSJ International Conference on Intelligent Robots and Systems (IROS)},
  pages={5026--5033},
  year={2012},
  organization={IEEE}
}

@inproceedings{towers2024gymnasium,
  title={Gymnasium: A standard interface for reinforcement learning environments},
  author={Towers, Mark and Kwiatkowski, Ariel and Terry, Jordan and Balis, John U. and De Cola, Gianluca and Deleu, Tristan and Goul{\~a}o, Manuel and Kallinteris, Andreas and Krimmel, Markus and KG, Arjun and Perez-Vicente, Rodrigo and Pierr{\'e}, Andrea and Schulhoff, Sander and Tai, Jun Jet and Tan, Hannah and Younis, Omar G.},
  booktitle={Advances in Neural Information Processing Systems 38 (NeurIPS 2025), Datasets and Benchmarks Track},
  pages={163114--163129},
  year={2025}
}

@article{raffin2021sb3,
  title={Stable-Baselines3: Reliable reinforcement learning implementations},
  author={Raffin, Antonin and Hill, Ashley and Gleave, Adam and Kanervisto, Anssi and Ernestus, Maximilian and Dormann, Noah},
  journal={Journal of Machine Learning Research},
  volume={22},
  number={268},
  pages={1--8},
  year={2021}
}

@article{cook1994simex,
  title={Simulation-extrapolation estimation in parametric measurement error models},
  author={Cook, John R. and Stefanski, Leonard A.},
  journal={Journal of the American Statistical Association},
  volume={89},
  number={428},
  pages={1314--1328},
  year={1994}
}

@book{carroll2006measurement,
  title={Measurement Error in Nonlinear Models: A Modern Perspective},
  author={Carroll, Raymond J. and Ruppert, David and Stefanski, Leonard A. and Crainiceanu, Ciprian M.},
  edition={2nd},
  publisher={Chapman and Hall/CRC},
  year={2006}
}

@article{schulman2017proximal,
  title={Proximal policy optimization algorithms},
  author={Schulman, John and Wolski, Filip and Dhariwal, Prafulla and Radford, Alec and Klimov, Oleg},
  journal={arXiv preprint arXiv:1707.06347},
  year={2017}
}

@article{m2023model,
  title={Model-based reinforcement learning: A survey},
  author={Moerland, Thomas M. and Broekens, Joost and Plaat, Aske and Jonker, Catholijn M.},
  journal={Foundations and Trends in Machine Learning},
  volume={16},
  number={1},
  pages={1--118},
  year={2023},
  publisher={Now Publishers}
}

@article{kidambi2020morel,
  title={Morel: Model-based offline reinforcement learning},
  author={Kidambi, Rahul and Rajeswaran, Aravind and Netrapalli, Praneeth and Joachims, Thorsten},
  journal={Advances in neural information processing systems},
  volume={33},
  pages={21810--21823},
  year={2020}
}

@article{wang2022no,
  title={No more pesky hyperparameters: Offline hyperparameter tuning for {RL}},
  author={Wang, Han and Sakhadeo, Archit and White, Adam and Bell, James and Liu, Vincent and Zhao, Xutong and Liu, Puer and Kozuno, Tadashi and Fyshe, Alona and White, Martha},
  journal={Transactions on Machine Learning Research},
  year={2022}
}

@article{yang2023foundation,
  title={Foundation models for decision making: Problems, methods, and opportunities},
  author={Yang, Sherry and Nachum, Ofir and Du, Yilun and Wei, Jason and Abbeel, Pieter and Schuurmans, Dale},
  journal={arXiv preprint arXiv:2303.04129},
  year={2023}
}

@inproceedings{janner2019trust,
  title={When to trust your model: Model-based policy optimization},
  author={Janner, Michael and Fu, Justin and Zhang, Marvin and Levine, Sergey},
  booktitle={Advances in Neural Information Processing Systems},
  volume={32},
  year={2019}
}

@inproceedings{yu2020mopo,
  title={{MOPO}: Model-based offline policy optimization},
  author={Yu, Tianhe and Thomas, Garrett and Yu, Lantao and Ermon, Stefano and Zou, James Y. and Levine, Sergey and Finn, Chelsea and Ma, Tengyu},
  booktitle={Advances in Neural Information Processing Systems},
  volume={33},
  pages={14129--14142},
  year={2020}
}

@book{van2000asymptotic,
  title={Asymptotic Statistics},
  author={van der Vaart, Aad W.},
  series={Cambridge Series in Statistical and Probabilistic Mathematics},
  volume={3},
  year={1998},
  publisher={Cambridge University Press}
}

@article{konyushova2021active,
  title={Active offline policy selection},
  author={Konyushova, Ksenia and Chen, Yutian and Paine, Thomas and Gulcehre, Caglar and Paduraru, Cosmin and Mankowitz, Daniel J and Denil, Misha and de Freitas, Nando},
  journal={Advances in Neural Information Processing Systems},
  volume={34},
  pages={24631--24644},
  year={2021}
}

@inproceedings{yang2022offline,
  title={Offline policy selection under uncertainty},
  author={Yang, Mengjiao and Dai, Bo and Nachum, Ofir and Tucker, George and Schuurmans, Dale},
  booktitle={International Conference on Artificial Intelligence and Statistics},
  pages={4376--4396},
  year={2022},
  organization={PMLR}
}

@inproceedings{zitovsky2023revisiting,
  title={Revisiting {B}ellman errors for offline model selection},
  author={Zitovsky, Joshua P and De Marchi, Daniel and Agarwal, Rishabh and Kosorok, Michael Rene},
  booktitle={International conference on machine learning},
  pages={43369--43406},
  year={2023},
  organization={PMLR}
}

@article{trella2022designing,
  title={Designing reinforcement learning algorithms for digital interventions: pre-implementation guidelines},
  author={Trella, Anna L and Zhang, Kelly W and Nahum-Shani, Inbal and Shetty, Vivek and Doshi-Velez, Finale and Murphy, Susan A},
  journal={Algorithms},
  volume={15},
  number={8},
  pages={255},
  year={2022},
  publisher={MDPI}
}

@inproceedings{ghosh2024rebandit,
  title={{reBandit}: Random effects based online {RL} algorithm for reducing cannabis use},
  author={Ghosh, Susobhan and Guo, Yongyi and Hung, Pei-Yao and Coughlin, Lara and Bonar, Erin and Nahum-Shani, Inbal and Walton, Maureen and Murphy, Susan},
  booktitle={Proceedings of the Thirty-Third International Joint Conference on Artificial Intelligence (IJCAI-24)},
  pages={7278--7286},
  year={2024},
  doi={10.24963/ijcai.2024/805}
}

@article{bjornsson2020digital,
  title={Digital twins to personalize medicine},
  author={Bj{\"o}rnsson, Bergthor and Borrebaeck, Carl and Elander, Nils and Gasslander, Thomas and Gawel, Danuta R. and Gustafsson, Mika and J{\"o}rnsten, Rebecka and Lee, Eun Jung and Li, Xinxiu and Lilja, Sandra and Mart{\'\i}nez-Enguita, David and Matussek, Andreas and Sandstr{\"o}m, Per and Sch{\"a}fer, Samuel and Stenmarker, Margaretha and Sun, Xiao-Feng and Sysoev, Oleg and Zhang, Huan and Benson, Mikael},
  journal={Genome Medicine},
  volume={12},
  number={1},
  pages={4},
  year={2020},
  doi={10.1186/s13073-019-0701-3}
}

@article{gazi2025digital,
  author = {Gazi, Asim H. and Gao, Daiqi and Ghosh, Susobhan and Xu, Ziping and Trella, Anna L. and Klasnja, Predrag and Murphy, Susan A.},
  title = {Digital Twins for Just-in-Time Adaptive Interventions ({JITAIs}): Framework for Optimizing and Continually Improving {JITAIs}},
  journal = {Journal of Medical Internet Research},
  volume = {28},
  pages = {e72830},
  year = {2026},
  doi = {10.2196/72830}
}

@article{xu2026diffusion,
  title={A Diffusion-Model Subpopulation Digital Twin for Mobile Health Deployment: A Case Study on the {HeartSteps} Intervention},
  author={Xu, Ziping and Chang, Yuyi and Ni, Chenshun and Sugavanam, Nithin and Gazi, Asim H. and Klasnja, Pedja and Ertin, Emre and Murphy, Susan A.},
  journal={arXiv preprint arXiv:2607.21403},
  year={2026},
  eprint={2607.21403},
  archivePrefix={arXiv},
  url={https://arxiv.org/abs/2607.21403}
}

@inproceedings{rajeswaran2017epopt,
  title={{EPOpt}: Learning Robust Neural Network Policies Using Model Ensembles},
  author={Rajeswaran, Aravind and Ghotra, Sarvjeet and Ravindran, Balaraman and Levine, Sergey},
  booktitle={International Conference on Learning Representations (ICLR)},
  year={2017}
}

@inproceedings{muratore2022neural,
  title={Neural Posterior Domain Randomization},
  author={Muratore, Fabio and Gruner, Theo and Wiese, Florian and Belousov, Boris and Gienger, Michael and Peters, Jan},
  booktitle={Conference on Robot Learning (CoRL)},
  series={Proceedings of Machine Learning Research},
  volume={164},
  pages={1532--1542},
  year={2022}
}

@inproceedings{precup2000eligibility,
  title={Eligibility traces for off-policy policy evaluation},
  author={Precup, Doina and Sutton, Richard S. and Singh, Satinder},
  booktitle={International Conference on Machine Learning (ICML)},
  pages={759--766},
  year={2000}
}

@inproceedings{jiang2016doubly,
  title={Doubly robust off-policy value evaluation for reinforcement learning},
  author={Jiang, Nan and Li, Lihong},
  booktitle={International Conference on Machine Learning},
  pages={652--661},
  year={2016},
  organization={PMLR}
}

@article{uehara2022review,
  title={A review of off-policy evaluation in reinforcement learning},
  author={Uehara, Masatoshi and Shi, Chengchun and Kallus, Nathan},
  journal={Statistical Science},
  volume={41},
  number={3},
  pages={561--581},
  year={2026},
  doi={10.1214/25-STS985}
}

@article{trella2024oralytics,
  title={Oralytics Reinforcement Learning Algorithm},
  author={Trella, Anna L. and Zhang, Kelly W. and Carpenter, Stephanie M. and Elashoff, David and Greer, Zara M. and Nahum-Shani, Inbal and Ruenger, Dennis and Shetty, Vivek and Murphy, Susan A.},
  journal={arXiv preprint arXiv:2406.13127},
  year={2024}
}

@article{muratore2021bayrn,
  title={Data-efficient Domain Randomization with {B}ayesian Optimization},
  author={Muratore, Fabio and Eilers, Christian and Gienger, Michael and Peters, Jan},
  journal={IEEE Robotics and Automation Letters},
  volume={6},
  number={2},
  pages={911--918},
  year={2021}
}

@inproceedings{chen2022understanding,
  title={Understanding Domain Randomization for Sim-to-real Transfer},
  author={Chen, Xiaoyu and Hu, Jiachen and Jin, Chi and Li, Lihong and Wang, Liwei},
  booktitle={International Conference on Learning Representations},
  year={2022}
}

@inproceedings{as2025spidr,
  title={{SPiDR}: A Simple Approach for Zero-Shot Safety in Sim-to-Real Transfer},
  author={As, Yarden and Qu, Chengrui and Unger, Benjamin and Kang, Dongho and van der Hart, Max and Shi, Laixi and Coros, Stelian and Wierman, Adam and Krause, Andreas},
  booktitle={Advances in Neural Information Processing Systems 38 (NeurIPS 2025)},
  pages={100933--100975},
  year={2025}
}

@inproceedings{mehta2020active,
  title={Active Domain Randomization},
  author={Mehta, Bhairav and Diaz, Manfred and Golemo, Florian and Pal, Christopher J. and Paull, Liam},
  booktitle={Conference on Robot Learning},
  pages={1162--1176},
  year={2020}
}

@inproceedings{tiboni2024doraemon,
  title={Domain Randomization via Entropy Maximization},
  author={Tiboni, Gabriele and Klink, Pascal and Peters, Jan and Tommasi, Tatiana and D'Eramo, Carlo and Chalvatzaki, Georgia},
  booktitle={International Conference on Learning Representations},
  year={2024}
}

@article{smirnova2019distributionally,
  title={Distributionally Robust Reinforcement Learning},
  author={Smirnova, Elena and Dohmatob, Elvis and Mary, J{\'e}r{\'e}mie},
  journal={arXiv preprint arXiv:1902.08708},
  year={2019}
}

@article{duchi2021learning,
  title={Learning Models with Uniform Performance via Distributionally Robust Optimization},
  author={Duchi, John C. and Namkoong, Hongseok},
  journal={The Annals of Statistics},
  volume={49},
  number={3},
  pages={1378--1406},
  year={2021}
}

@inproceedings{fujinami2025domain,
  title={Domain Randomization is Sample Efficient for Linear Quadratic Control},
  author={Fujinami, Tesshu and Lee, Bruce D. and Matni, Nikolai and Pappas, George J.},
  booktitle={Proceedings of the 7th Annual Learning for Dynamics and Control Conference (L4DC)},
  series={Proceedings of Machine Learning Research},
  volume={283},
  pages={907--919},
  year={2025}
}

@article{queeney2024optimal,
  title={Optimal Transport Perturbations for Safe Reinforcement Learning with Robustness Guarantees},
  author={Queeney, James and Ozcan, Erhan Can and Paschalidis, Ioannis Ch. and Cassandras, Christos G.},
  journal={Transactions on Machine Learning Research},
  year={2024}
}

@inproceedings{barsce2019hierarchical,
  title={A Hierarchical Two-tier Approach to Hyper-parameter Optimization in Reinforcement Learning},
  author={Barsce, Juan Cruz and Palombarini, Jorge A. and Mart{\'\i}nez, Ernesto},
  booktitle={Anales del Simposio Argentino de Inteligencia Artificial (ASAI), 48 Jornadas Argentinas de Inform{\'a}tica (JAIIO)},
  pages={32--38},
  year={2019}
}

@inproceedings{bellemare2016unifying,
  title={Unifying Count-Based Exploration and Intrinsic Motivation},
  author={Bellemare, Marc G. and Srinivasan, Sriram and Ostrovski, Georg and Schaul, Tom and Saxton, David and Munos, R{\'e}mi},
  booktitle={Advances in Neural Information Processing Systems},
  volume={29},
  year={2016}
}

@article{kober2013reinforcement,
  title={Reinforcement learning in robotics: A survey},
  author={Kober, Jens and Bagnell, J. Andrew and Peters, Jan},
  journal={The International Journal of Robotics Research},
  volume={32},
  number={11},
  pages={1238--1274},
  year={2013}
}

@article{ibarz2021how,
  title={How to train your robot with deep reinforcement learning: lessons we have learned},
  author={Ibarz, Julian and Tan, Jie and Finn, Chelsea and Kalakrishnan, Mrinal and Pastor, Peter and Levine, Sergey},
  journal={The International Journal of Robotics Research},
  volume={40},
  number={4--5},
  pages={698--721},
  year={2021}
}

\newpage

\appendix
\section{Related work}
\label{app:related}

This appendix expands the related-work paragraph of Section~\ref{sec:setup}.
We organize the literature by the four bodies of work a reader is most likely to connect with our method: the selection of online algorithms from offline data, the construction of simulators and their calibration through domain randomization, ensembles and the bootstrap as tools for uncertainty in RL, and robust or distributionally robust RL.
For each, we state what is shared with our setting and where our question differs.

\paragraph{Selection of online algorithms from offline data.}
The performance of online RL algorithms is known to be sensitive to hyperparameters and other design choices \citep{henderson2018deep,wang2022no,fan2025fragility}, which makes the choice of what to deploy a first-order concern.
\citet{mandel2016offline} formalize the offline algorithm selection problem: choose the best online algorithm from a fixed set using only a logged dataset.
The problem is related to offline policy evaluation \citep{precup2000eligibility,jiang2016doubly,uehara2022review} and offline policy selection \citep{konyushova2021active,yang2022offline}, also called offline model selection \citep{zitovsky2023revisiting}, which rank a fixed set of policies using an offline dataset.
The two problems coincide when the policy being selected is allowed to be non-Markovian, since an online algorithm is a policy over interaction histories.
Existing offline policy selection methods, however, assume Markovian candidates and rely on fitted value functions or Bellman errors, which do not extend to candidates whose actions depend on the whole history through adaptive exploration.
Hyperparameter optimization methods such as Bayesian optimization over the hyperparameter space \citep{barsce2019hierarchical} tune against returns observed online and therefore require repeated deployments, which is what the offline setting rules out.
Existing work on offline algorithm selection has predominantly taken the simulator approach: build a simulator from the offline data and evaluate each candidate algorithm in it.
Examples include queue-based replay of logged rewards \citep{mandel2016offline}, $k$-nearest-neighbor simulators that resample observed transitions close to the queried state--action pair \citep{wang2022no}, rejection-sampling simulators that correct for the mismatch between the logging policy and the candidate algorithm \citep{tang2022towards}, learned or foundation world models used as the evaluation environment \citep{m2023model,yang2023foundation}, and parametric models that encode domain knowledge in digital health applications \citep{trella2022designing,trella2024oralytics,ghosh2024rebandit}.
In each of these works the candidates are scored in a single simulator fitted to the data, which is the Plug-In rule we study.
The digital health works are a partial exception: following the guidelines of \citet{trella2022designing}, candidates are additionally evaluated across a small set of hand-designed environment variants, for example with different treatment-effect sizes, as a sensitivity analysis.
Those variants are chosen by the designer rather than drawn from the estimation uncertainty of the fitted model, and none of these works analyzes how the estimation error of the simulator affects which algorithm is selected.

\paragraph{Simulators from offline data and domain randomization.}
Using offline data to construct a simulator for decision-making is a common thread connecting system identification in control and robotics, which calibrates physical parameters from recorded trajectories \citep{aastrom1971system,ljung1999system}; digital twins in medicine, which model patient dynamics to evaluate treatment strategies \citep{bjornsson2020digital,gazi2025digital,xu2026diffusion}; and learned world models, which predict environment transitions to support planning and policy learning \citep{li2025uncertainty}.
In robotics, the mismatch between such a simulator and the physical system is addressed by domain randomization, which trains a policy on a distribution over simulator parameters so that it transfers to the real system \citep{tobin2017domain,peng2018sim}; see \citet{muratore2022robot} for a review.
The randomization distribution can be fixed by hand or learned: Bayesian system identification places a posterior over physical parameters and randomizes over it \citep{ramos2019bayessim}, and other work adapts the distribution using real rollouts \citep{chebotar2019closing}, Bayesian optimization on real-world returns \citep{muratore2021bayrn}, a neural posterior over simulator parameters \citep{muratore2022neural}, a likelihood fitted to offline trajectories \citep{tiboni2023dropo}, or objectives that seek informative or maximally wide parameter ranges \citep{mehta2020active,tiboni2024doraemon}.
Sampling a distribution of models is also used directly to obtain robustness: EPOpt trains on an ensemble of simulated models and optimizes a conditional value-at-risk of the return across them \citep{rajeswaran2017epopt}, and SPiDR incorporates the uncertainty about the sim-to-real gap into safety constraints \citep{as2025spidr}.
On the theoretical side, \citet{chen2022understanding} bound the sim-to-real gap of a policy trained under domain randomization and show that memory-dependent policies are needed to close it, and \citet{fujinami2025domain} show that for linear quadratic control, a controller that averages its objective over a posterior on the system parameters is as sample efficient as certainty equivalence, that is, as the Plug-In controller.
Two features separate our setting from this literature.
First, domain randomization models a mismatch between the simulator and the deployment environment, whereas in our setting there is no mismatch: offline data and deployment come from the same environment, and the spread of the ensemble reflects only the sampling variability of a finite dataset, which vanishes as the data grow.
Second, the object being made robust is a fixed policy, whereas we select among online learners that continue to learn after deployment, so what matters is their ranking by online regret rather than the return of any one policy.
The result of \citet{fujinami2025domain} is the closest in spirit, since it compares averaging over parameter uncertainty with certainty equivalence, but it concerns a fixed controller and asks about sample efficiency rather than about the selection among learners with asymmetric regret.

\paragraph{Ensembles and the bootstrap in RL.}
The bootstrap is a classical tool for approximating sampling uncertainty by resampling the observed data \citep{efron2000introduction}, and the parametric bootstrap we use for bandits is standard \citep[Chapter 23]{van2000asymptotic}.
In RL, bootstrapped and probabilistic ensembles quantify epistemic uncertainty to drive exploration \citep{osband2016deep} and to represent dynamics uncertainty in model-based planning and policy optimization \citep{chua2018deep,janner2019trust}.
In offline model-based RL, the disagreement across ensemble members is turned into a pessimistic penalty that keeps the learned policy away from states where the model is uncertain \citep{yu2020mopo,kidambi2020morel}, and recent work propagates epistemic uncertainty through multi-step predictions of autoregressive world models so that policies trained entirely offline can be deployed on physical robots \citep{li2025uncertainty}.
In all of these methods the ensemble is part of the algorithm being run and changes how that algorithm acts, either by adding an exploration bonus or by penalizing uncertain regions.
In our setting the ensemble sits outside the candidate algorithms, which are evaluated as black boxes, and it changes only which candidate is chosen; the candidates themselves are not modified.

\paragraph{Robust and distributionally robust RL.}
Robust MDPs account for model uncertainty by optimizing the worst-case value over an uncertainty set of transition kernels \citep{iyengar2005robust,nilim2005robust}.
Distributionally robust optimization extends this to uncertainty sets around an empirical distribution, in supervised learning \citep{duchi2021learning} and in RL, where it yields risk-averse exploration with lower-bound guarantees on state values \citep{smirnova2019distributionally} or safe policies with robustness guarantees under optimal-transport perturbations of the dynamics \citep{queeney2024optimal}.
Regularized RL objectives have also been shown to be equivalent to robustness against certain classes of environment perturbations \citep{eysenbach2021maximum}.
These methods share our motivation of accounting for model uncertainty, but they do so by changing the objective to a worst case over an uncertainty set, and again the object being optimized is a fixed policy.
We keep the objective as a plain average over the ensemble and show that this average already acts as a curvature regularization of the Plug-In criterion (Proposition~\ref{prop:regularized1}) and suffices for the regret gains of Section~\ref{sec:mab}; no pessimism is required.

\paragraph{Summary.}
The works discussed above primarily take a practical perspective, showing how modeling simulator uncertainty can improve policy robustness and the reliability of model-based learning.
In contrast, this paper approaches the problem from a theoretical perspective, demonstrating how a Plug-In simulator that ignores estimation uncertainty can lead to underexploration.
This failure is particularly consequential in online algorithm selection, where exploration determines how effectively the selected learner acquires experience during deployment.
We characterize the resulting regret and show how accounting for simulator uncertainty through ensemble-based selection mitigates this failure.

\section{Simulation details for the two-armed bandit figures}
\label{app:plane}

This appendix states the protocol behind Figure~\ref{fig:selection-regret}, Figure~\ref{fig:UCB-exploration} and its Thompson Sampling counterpart, Figure~\ref{fig:TS-exploration}.
The three figures share one measured objective surface and one pair of selection rules, given in Sections~\ref{app:plane-env}--\ref{app:plane-rules}; Sections~\ref{app:plane-fig2} and \ref{app:plane-fig3} say what each figure draws from them.

\subsection{Environments and candidate sets}
\label{app:plane-env}

For a margin $\Delta \in \mathbb R$ and a noise level $\sigma > 0$, let $M_{\Delta, \sigma}$ denote the two-armed Gaussian bandit with $\mu_1 = \Delta$, $\mu_2 = 0$ and $R_t \mid A_t = a \sim \cN(\mu_a, \sigma^2)$.
The true environment is $M^* = M_{1, 3}$, and the online horizon is $T = 5000$.

Adding a constant to both arm means leaves the action sequence of \eqrefff{eqn:ucb1} and of \eqrefff{eqn:TS} unchanged and shifts $J_T$ by an amount free of $\theta$, and relabeling the arms replaces $\Delta$ by $|\Delta|$ and again shifts $J_T$ by an amount free of $\theta$.
A two-armed Gaussian bandit is therefore summarized by $(\Delta, \sigma)$ for the purpose of ranking candidates.

The candidate sets are the discrete grids
\begin{align}
\Theta_{\textnormal{UCB}} = \{0.1,\, 0.2,\, \dots,\, 1.1\} \cdot \sigma^2,
\qquad
\Theta_{\textnormal{TS}} = \{0.1,\, 0.2,\, \dots,\, 1.5\} \cdot \sigma^2 ,
\label{eq:plane-grid}
\end{align}
where $\sigma = 3$ is the noise level of $M^*$; the grid for \eqrefff{eqn:TS} runs further because its objective is flatter above $\theta_T^*$.
Every axis and legend reports the normalized exploration level $\theta / \sigma^2$.

\subsection{The objective surface}
\label{app:plane-surface}

Write $\Theta$ for the candidate set of the algorithm at hand.
The objective $J_T(\pi_\theta, M_{\Delta, \sigma})$ is measured by simulation for every $\theta \in \Theta$ and every
\begin{align}
(\Delta, \sigma) \in \mathcal G = \{0.075,\, 0.225,\, \dots,\, 5.925\} \times \{0.9,\, 1.0,\, \dots,\, 5.7\},
\label{eq:plane-grid-env}
\end{align}
which is $40 \times 49$ pairs whose $\Delta$ values sit half a step off the origin, so that no cell is centered on $\Delta = 0$.
The cell of $\mathcal G$ nearest $M^*$ is $M_{0.975,\, 3}$.

The runs use common random numbers.
Fix a seed and draw
\begin{align}
\psi_{t,r,a}, \ \zeta_{t,r,a} \stackrel{\textnormal{iid}}{\sim} \cN(0,1),
\qquad t \le T, \quad r \le 2000, \quad a \in \mathcal A .
\label{eq:plane-crn}
\end{align}
In replication $r$ of the run at $(\theta, \Delta, \sigma)$, pulling arm $a$ at round $t$ returns $\mu_a + \sigma \psi_{t,r,a}$, and the posterior draw of \eqrefff{eqn:TS} at that round is $\widehat\mu_a(t-1) + \sqrt{\theta / N_a(t-1)}\, \zeta_{t,r,a}$.
The single draw \eqrefff{eq:plane-crn} serves every $(\theta, \Delta, \sigma)$ of the run.
With $N_{2,r}(T)$ the pull count of the suboptimal arm in replication $r$, the estimate is
\begin{align}
\widehat J_T(\pi_\theta, M_{\Delta, \sigma}) = T \Delta - \frac{\Delta}{2000} \sum_{r=1}^{2000} N_{2,r}(T),
\label{eq:plane-Jhat}
\end{align}
which subtracts the cumulative pseudo-regret from the return of the optimal arm.
Sixteen runs with distinct seeds are pooled, giving $32{,}000$ replications per $(\theta, \Delta, \sigma)$.
At the cell nearest $M^*$ the pooled standard error of \eqrefff{eq:plane-Jhat} is $0.94$ at the best candidate for \eqrefff{eqn:ucb1} and $1.12$ for \eqrefff{eqn:TS}, rising to $8.3$ and $8.9$ at the smallest candidate.

\subsection{Offline data, the Plug-In simulator and the ensemble}
\label{app:plane-rules}

The behavior policy draws each $A_t$ uniformly on $\mathcal A$, so $\mathcal D_{\textnormal{off}}$ holds $T_{\textnormal{off}} = 25$ pairs and $N_1(T_{\textnormal{off}})$ is binomial.
Let $\textnormal{SS}_a$ be the within-arm sum of squares of arm $a$ in $\mathcal D_{\textnormal{off}}$.
The Plug-In simulator is $\widehat M = M_{\widehat\Delta, \widehat\sigma}$ with
\begin{align}
\widehat\Delta = \widehat\mu_1(T_{\textnormal{off}}) - \widehat\mu_2(T_{\textnormal{off}}),
\qquad
\widehat\sigma^2 = \frac{\textnormal{SS}_1 + \textnormal{SS}_2}{T_{\textnormal{off}} - 2},
\label{eq:plane-fit}
\end{align}
one reward variance pooled across the two arms.

Ensemble member $i$ is fitted by \eqrefff{eq:plane-fit} to a fresh dataset of $T_{\textnormal{off}}$ pairs drawn from $\widehat M$ under the same behavior policy, giving $\widehat M^{(i)} = M_{\widehat\Delta^{(i)}, \widehat\sigma^{(i)}}$, a parametric bootstrap of \eqrefff{eq:plane-fit}.
The two selections are
\begin{align}
\widehat\theta_{\plug,T} = \argmax_{\theta \in \Theta} J_T(\pi_\theta, \widehat M),
\qquad
\widehat\theta_{\UA,T} = \argmax_{\theta \in \Theta} \frac{1}{n} \sum_{i=1}^{n} J_T(\pi_\theta, \widehat M^{(i)}),
\qquad n = 8000 .
\label{eq:plane-rules}
\end{align}
Both objectives in \eqrefff{eq:plane-rules} are read off \eqrefff{eq:plane-Jhat} by bilinear interpolation over \eqrefff{eq:plane-grid-env}, the argument first clipped to $\mathcal G$; no fresh rollout is simulated at selection time.

Selections are tabulated on a common grid of fits: the window and bin size are
\begin{align}
\widehat\Delta \in [-2.5,\, 4.5], \qquad \widehat\sigma \in [1.5,\, 4.5], \qquad \textnormal{bin } 0.15 \times 0.10 ,
\label{eq:plane-bins}
\end{align}
and a dataset is assigned the selection computed at the center of the bin its fit \eqrefff{eq:plane-fit} falls in.

\subsection{Figure~\ref{fig:selection-regret}}
\label{app:plane-fig2}

Both rows are evaluated at the cell of $\mathcal G$ nearest $M^*$, with one column per algorithm.

The top row plots the cumulative pseudo-regret $T \Delta - \widehat J_T(\pi_\theta, M_{\Delta, \sigma})$ of \eqrefff{eq:plane-Jhat} against $\theta / \sigma^2$ over the candidate set \eqrefff{eq:plane-grid}, and the dotted line marks its minimizer $\theta_T^*$.
Its minimum is the pseudo-regret of the best candidate, so the curve does not reach zero.

The bottom row plots the distribution of $\widehat\theta_{\plug,T}$ and of $\widehat\theta_{\UA,T}$ over independent draws of $\mathcal D_{\textnormal{off}}$ from $M^*$.
The annotation gives the mean and the variance, over those draws, of the cumulative pseudo-regret of the selected candidate, measured on the scale of the row above.
The two rules are compared on one set of offline datasets, each assigned its selection through \eqrefff{eq:plane-bins}.

\subsection{Asymmetry across environment configurations}
\label{app:plane-robust}

The top row of Figure~\ref{fig:selection-regret} is one environment.
This section repeats that measurement over a grid of environments.

Each environment is measured on its own noise-relative candidate grid
\begin{align}
\theta = c\,\sigma^2, \qquad c \in \{0.05,\, 0.10,\, \dots,\, 2.00\},
\label{eq:robust-grid}
\end{align}
so that the best candidate of every environment lies inside the grid and the horizontal axis $\theta/\sigma^2$ means the same thing in every panel.
The environments are the sixteen pairs
\begin{align}
(\Delta, \sigma) \in \{0.25,\, 0.5,\, 1,\, 2\} \times \{1,\, 2,\, 3,\, 5\} .
\label{eq:robust-envs}
\end{align}
Each $(\theta, \Delta, \sigma)$ is measured at $T = 5000$ from $40{,}000$ replications sharing the common random numbers of \eqrefff{eq:plane-crn}, by a run separate from the surface of Section~\ref{app:plane-surface}.
At $(\Delta, \sigma) = (1, 3)$ the two agree: the cumulative pseudo-regret at $c = 0.1$, $0.3$ and $0.6$ is $586$, $143$ and $79$ here against $577$, $142$ and $81$ in Figure~\ref{fig:selection-regret}.

Write $c^\star$ for the maximizer of $\widehat J_T(\pi_{c \sigma^2}, M_{\Delta, \sigma})$ over \eqrefff{eq:robust-grid}.
Figures~\ref{fig:asym-ucb} and \ref{fig:asym-ts} plot the excess regret of a candidate over the best one, as a percentage of the regret of the best one,
\begin{align}
e(c) = 100 \cdot \frac{\widehat J_T(\pi_{c^\star \sigma^2}, M_{\Delta, \sigma}) - \widehat J_T(\pi_{c \sigma^2}, M_{\Delta, \sigma})}{T\Delta - \widehat J_T(\pi_{c^\star \sigma^2}, M_{\Delta, \sigma})} .
\label{eq:robust-excess}
\end{align}
The normalization in \eqrefff{eq:robust-excess} makes the panels comparable: the regret at $c^\star$ ranges from $0.4$ to $1100$ over \eqrefff{eq:robust-envs}, while $e$ is a ratio.

The asymmetry holds in every configuration.
Taking $e$ at half and at twice the best exploration level, by linear interpolation in $c$, the ratio $e(c^\star/2) / e(2c^\star)$ exceeds one in all sixteen configurations for both algorithms: for \eqrefff{eqn:ucb1} it has median $1.7$ and range $1.1$ to $3.8$, and for \eqrefff{eqn:TS} median $1.9$ and range $1.4$ to $4.7$.
Halving the exploration level therefore costs between $1.1$ and $4.7$ times what doubling it costs, across margins spanning a factor of eight and noise levels spanning a factor of five.

\begin{figure}[t]
    \centering
    \includegraphics[width=\linewidth]{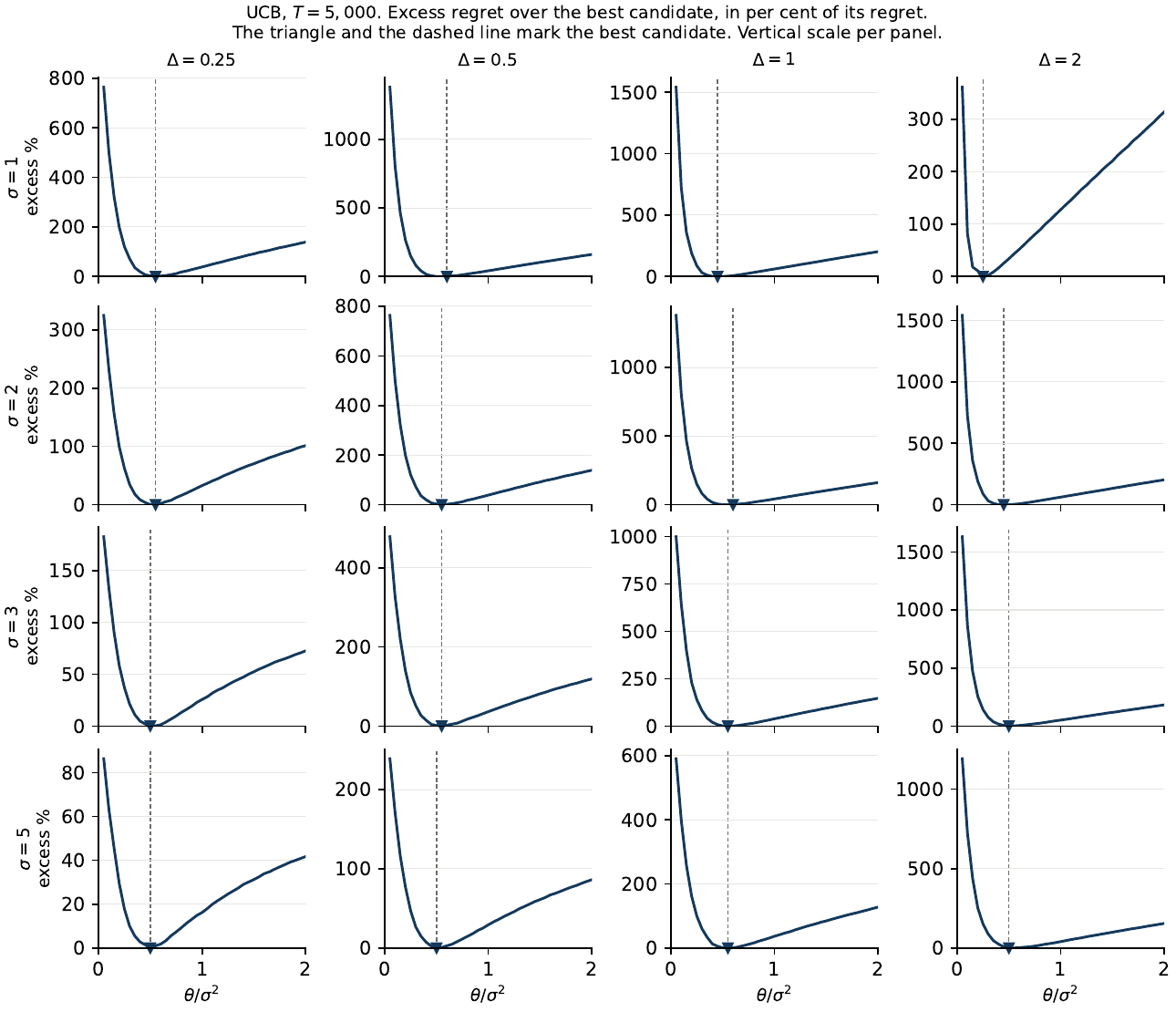}
    \caption{\textbf{The cost of misspecified exploration is asymmetric in every environment tested, under UCB.}
    Excess regret \eqrefff{eq:robust-excess} of UCB$(\theta)$ against $\theta/\sigma^2$, for the environments \eqrefff{eq:robust-envs}, at $T = 5000$.
    The triangle and the dashed line mark the best candidate.
    The vertical scale is per panel, since the regret at the best candidate varies by three orders of magnitude across these environments.}
    \label{fig:asym-ucb}
\end{figure}

\begin{figure}[t]
    \centering
    \includegraphics[width=\linewidth]{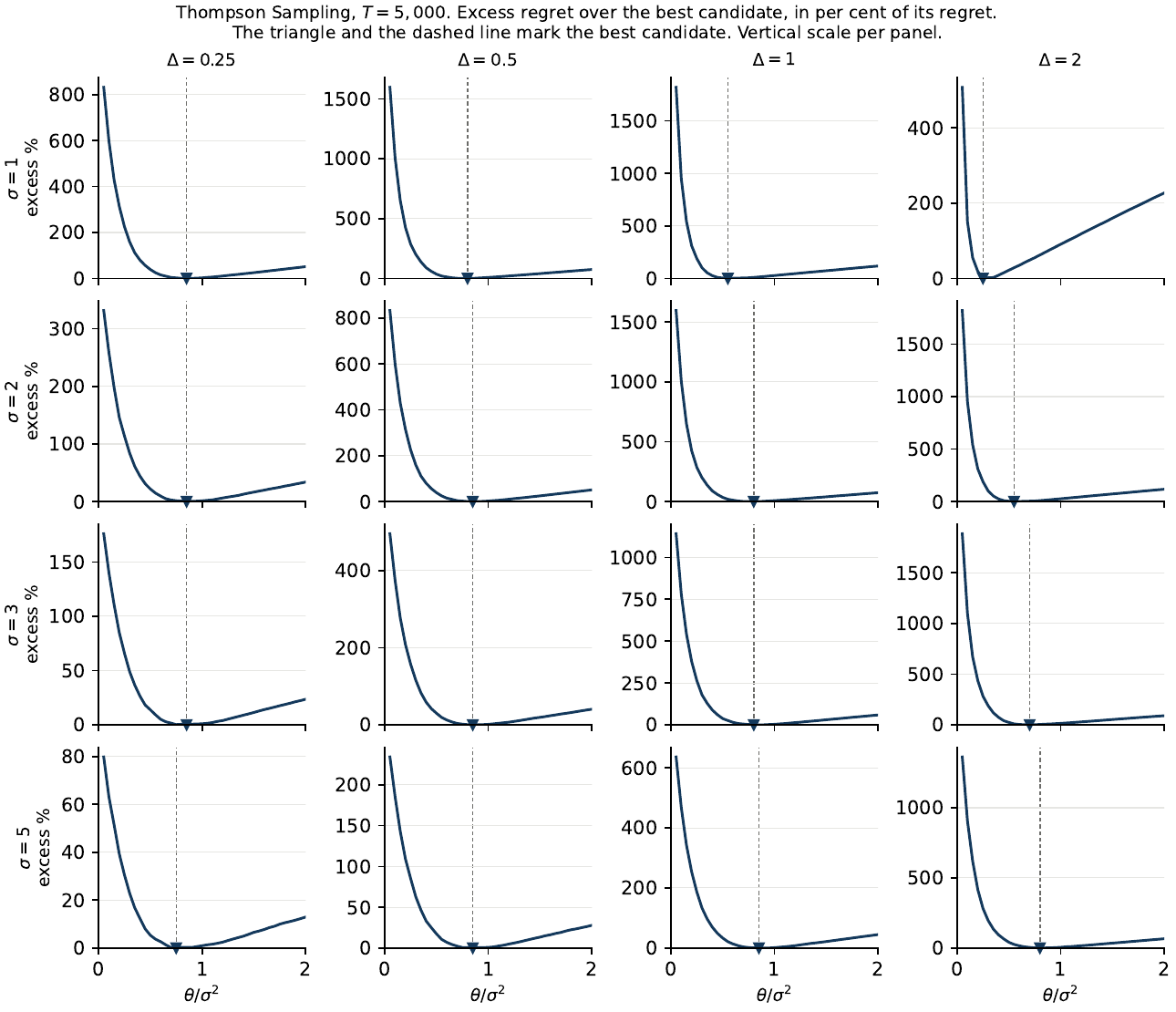}
    \caption{\textbf{The cost of misspecified exploration is asymmetric in every environment tested, under Thompson Sampling.}
    The counterpart of Figure~\ref{fig:asym-ucb} for the algorithm of \eqrefff{eqn:TS}.}
    \label{fig:asym-ts}
\end{figure}

\subsection{Figures~\ref{fig:UCB-exploration} and \ref{fig:TS-exploration}}
\label{app:plane-fig3}

The contours are level sets of the law of $(\widehat\Delta, \widehat\sigma)$ under \eqrefff{eq:plane-fit} applied to the same $10^6$ datasets, at the levels enclosing $50\%$, $90\%$ and $99\%$ of its mass.
The left and center panels report $\widehat\theta_{\plug,T} / \sigma^2$ and $\widehat\theta_{\UA,T} / \sigma^2$ over the bins \eqrefff{eq:plane-bins}, and the right panel reports their difference.
Where an ensemble member falls outside $\mathcal G$ its argument is clipped to the boundary; this affects more than $5\%$ of the members only in the upper right corner of the window, which carries almost none of the mass of the contours.

\subsection{Budgets}
\label{app:plane-budget}

Each cell of \eqrefff{eq:plane-Jhat} averages $32{,}000$ runs of $T$ rounds, which is $6.9 \times 10^8$ runs over the surface for \eqrefff{eqn:ucb1} and $9.4 \times 10^8$ for \eqrefff{eqn:TS}.
The selections use $10^6$ draws of $\mathcal D_{\textnormal{off}}$ and $n$ ensemble members per bin, each scored by interpolation of \eqrefff{eq:plane-Jhat} rather than by fresh simulation.
The bins at the $50\%$, $90\%$ and $99\%$ contours hold about $2200$, $470$ and $69$ of the $10^6$ datasets.

\begin{figure}[t]
    \centering
    \includegraphics[width=\linewidth]{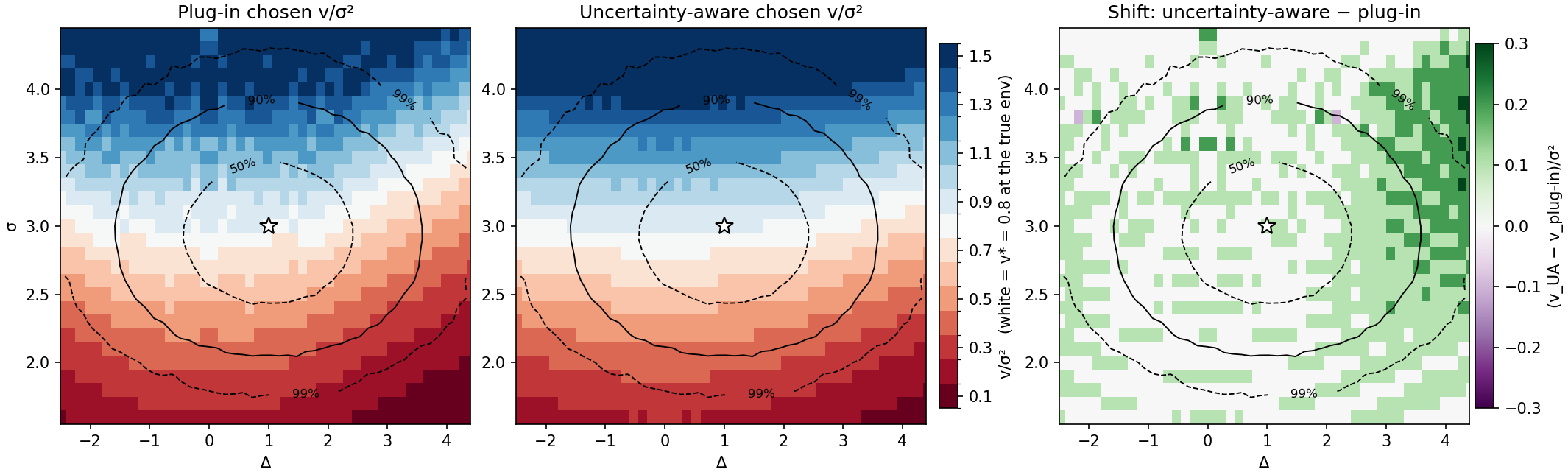}
    \caption{\textbf{Uncertainty-Aware (UA) selection leads to greater exploration than Plug-In under Thompson Sampling.}
    The counterpart of Figure~\ref{fig:UCB-exploration} for the algorithm of \eqrefff{eqn:TS}, on the candidate grid $\Theta_{\textnormal{TS}}$ of \eqrefff{eq:plane-grid}.
    Star, contours and panels are as in Figure~\ref{fig:UCB-exploration}.}
    \label{fig:TS-exploration}
\end{figure}

\section{Deep reinforcement learning experiments}
\label{app:deeprl}

This appendix records the definitions and settings behind Section~\ref{sec:experiments} and the results not shown there.
Throughout, $M_\lambda$, $\lambda^*$, $\widehat\lambda$, $\widehat\lambda^{(i)}$, $\theta$ and $J_T$ are as defined in Section~\ref{sec:experiments}.

\subsection{Tasks, reward and success criterion}
\label{app:deeprl-reward}

The tasks are \texttt{Hopper-v5} (three actuators) and \texttt{Walker2d-v5} (six actuators) from Gymnasium, with action space $\mathcal{A} = [-1, 1]^{d_{\mathcal{A}}}$, where $d_{\mathcal{A}}$ is the number of actuators.
An episode ends after $H = 250$ steps or when the robot leaves Gymnasium's healthy region.
The feature vector is
\begin{align}
\varphi(S, A) = \big(1,\ v_x,\ -\lVert A \rVert_2^2\big),
\end{align}
where the first entry is a survival bonus paid on every step, $v_x$ is the forward velocity of the torso reported by the environment, and the third entry is the squared actuation.
The learner trains on
\begin{align}
R_\theta(S, A) = \theta^\top \varphi(S, A) = \theta_{\textnormal{surv}} + \theta_{\textnormal{fwd}}\, v_x - \theta_{\textnormal{ctrl}}\, \lVert A \rVert_2^2,
\qquad \theta_{\textnormal{surv}} = 1.
\end{align}

The score of an episode is computed from the true state.
Let $H' \le H$ be the number of steps the episode lasted, $x_{H'}$ the horizontal distance traveled by its end, and $\bar c = \frac{1}{H'} \sum_{h=1}^{H'} \lVert A_h \rVert_2^2$ the average squared actuation over those steps.
Then
\begin{align}
\textnormal{score} = \mathbf{1}(\bar c \le c_{\max}) \cdot \min\!\Big\{1,\ \max\!\Big\{0,\ \frac{x_{H'} - x_{\textnormal{lo}}}{x_{\textnormal{hi}} - x_{\textnormal{lo}}}\Big\}\Big\}.
\label{eq:app-deeprl-score}
\end{align}
The objective $J_T(\pi_\theta, M)$ of Section~\ref{sec:experiments} is the expected score of the policy that $\pi_\theta$ has learned in $M$; its estimate, the average score over $50$ evaluation episodes, is the success score.
The thresholds (Table~\ref{tab:deeprl-app-thresholds}) follow the calibration protocol of \citet{morse2025}: a policy is trained at their default weights $\theta = (1, 1, 0.002)$ on the true environment and rolled out for $6$ seeds and $100$ episodes each; $x_{\textnormal{hi}}$ is the median of its $x_{H'}$, $x_{\textnormal{lo}} = x_{\textnormal{hi}} / 2$, and $c_{\max}$ is $0.9$ times its mean squared actuation.
Two details differ from \citet{morse2025}.
The thresholds are re-derived at our training budget, because a policy trained for $3 \times 10^5$ steps travels past their published distance thresholds at every weight vector in the grid.
The effort constraint is applied to the episode average of $\lVert A \rVert_2^2$ rather than to its value at every step, because at our budget the per-step maximum equals the actuator limit for every trained policy.

\begin{table}[htbp]
\centering
\caption{Thresholds of the success criterion \eqrefff{eq:app-deeprl-score}, in meters and in units of $\lVert A \rVert_2^2$.}
\label{tab:deeprl-app-thresholds}
\small
\begin{tabular}{@{}lrrr@{}}
\toprule
task & $x_{\textnormal{lo}}$ & $x_{\textnormal{hi}}$ & $c_{\max}$ \\
\midrule
\texttt{Hopper} & $1.523$ & $3.047$ & $1.548$ \\
\texttt{Walker2d} & $1.669$ & $3.337$ & $3.346$ \\
\bottomrule
\end{tabular}
\end{table}

\subsection{Candidate set and true success scores}
\label{app:deeprl-grid}

The candidate set $\Theta$ is the product of $\theta_{\textnormal{fwd}} \in \{0.3, 1, 3, 10, 30\}$ and $\theta_{\textnormal{ctrl}} \in \{0.002, 0.01, 0.05, 0.25, 1\}$ on \texttt{Hopper} ($25$ weight vectors), and of $\theta_{\textnormal{fwd}} \in \{0.03, 0.1, 0.3, 1, 3, 10\}$ and the same $\theta_{\textnormal{ctrl}}$ on \texttt{Walker2d} ($30$ weight vectors); the \texttt{Walker2d} range was shifted downward so that its optimum is interior.
The value of a weight vector is $J_T(\pi_\theta, M^*)$, the success score of a policy trained on the true environment under it.
Table~\ref{tab:deeprl-app-truth} gives it for every $\theta \in \Theta$; the success score of a selected $\widehat\theta$ is read from it, and the score of $\theta_T^*$ is its ceiling.
Both tasks have $\theta_T^*$ at $(\theta_{\textnormal{fwd}}, \theta_{\textnormal{ctrl}}) = (1, 0.25)$, with $J_T(\pi_{\theta_T^*}, M^*) = 0.991$ on \texttt{Hopper} and $0.689$ on \texttt{Walker2d}.

\begin{table}[htbp]
\centering
\caption{Success score $J_T(\pi_\theta, M^*)$ of every weight vector $\theta \in \Theta$ on the true environment, $\theta_{\textnormal{surv}} = 1$, under the sensor at $\sigma_{\mathrm{rel}} = 0.05$; $30$ PPO seeds per cell on \texttt{Hopper} and $32$ on \texttt{Walker2d}, $50$ evaluation episodes per seed. The optimum $\theta_T^*$ of each task is in bold; the success score of a selected $\theta$ is read from it.}
\label{tab:deeprl-app-truth}
\small
\begin{tabular}{@{}lrrrrr@{}}
\toprule
\texttt{Hopper} & \multicolumn{5}{c}{$\theta_{\textnormal{ctrl}}$} \\
$\theta_{\textnormal{fwd}}$ & $0.002$ & $0.01$ & $0.05$ & $0.25$ & $1$ \\
\midrule
$0.3$ & $0.190$ & $0.067$ & $0.062$ & $0.218$ & $0.000$ \\
$1$ & $0.381$ & $0.555$ & $0.872$ & $\mathbf{0.991}$ & $0.553$ \\
$3$ & $0.365$ & $0.456$ & $0.783$ & $0.930$ & $0.936$ \\
$10$ & $0.378$ & $0.360$ & $0.552$ & $0.764$ & $0.764$ \\
$30$ & $0.255$ & $0.368$ & $0.438$ & $0.614$ & $0.637$ \\
\bottomrule
\end{tabular}
\vspace{6pt}

\begin{tabular}{@{}lrrrrr@{}}
\toprule
\texttt{Walker2d} & \multicolumn{5}{c}{$\theta_{\textnormal{ctrl}}$} \\
$\theta_{\textnormal{fwd}}$ & $0.002$ & $0.01$ & $0.05$ & $0.25$ & $1$ \\
\midrule
$0.03$ & $0.001$ & $0.002$ & $0.000$ & $0.000$ & $0.000$ \\
$0.1$ & $0.001$ & $0.058$ & $0.062$ & $0.009$ & $0.000$ \\
$0.3$ & $0.048$ & $0.140$ & $0.527$ & $0.517$ & $0.000$ \\
$1$ & $0.013$ & $0.134$ & $0.456$ & $\mathbf{0.689}$ & $0.000$ \\
$3$ & $0.013$ & $0.020$ & $0.056$ & $0.536$ & $0.093$ \\
$10$ & $0.033$ & $0.040$ & $0.026$ & $0.107$ & $0.143$ \\
\bottomrule
\end{tabular}
\vspace{6pt}

\end{table}

\subsection{Learner}
\label{app:deeprl-learner}

The learner is PPO from Stable-Baselines3 with a two-hidden-layer MLP policy of width $64$ and tanh activations, $1024$ steps per update, minibatches of $64$, $10$ epochs per update, discount $0.99$, GAE parameter $0.95$, clipping range $0.2$, no entropy bonus, learning rate $3 \times 10^{-4}$, and $3 \times 10^5$ environment steps in a single environment.
Observations and rewards are normalized by running estimates of their mean and variance, with normalized observations clipped to $[-10, 10]$; the observation statistics are frozen for evaluation.
Reward normalization makes training invariant to the scale of $\theta$, which is why $\theta_{\textnormal{surv}}$ is fixed at $1$.
Evaluation runs the deterministic policy for $50$ episodes.
Every training run in the experiment uses these settings; only the reward weights $\theta$, the physics parameter $\lambda$ and the seed vary.

\subsection{True environment, offline data and sensor}
\label{app:deeprl-plant}

The physics parameter $\lambda = (\lambda_{\textnormal{mass}}, \lambda_{\textnormal{fric}}, \lambda_{\textnormal{damp}})$ multiplies the MuJoCo model's body masses, the sliding coefficient of every geom's friction, and the damping of every degree of freedom, relative to the factory model at $\lambda = (1, 1, 1)$.
The true environment is $M^* = M_{\lambda^*}$ with $\lambda^* = (1.5, 0.7, 1.4)$; the simulator $M_\lambda$ is the same model with the factors set to $\lambda$.
The offline dataset $\mathcal{D}_{\textnormal{off}}$ is $T_{\textnormal{off}}$ episodes on $M^*$ under a behavior policy $\pi_b$ that draws each action uniformly from $\mathcal{A}$, with the termination rules above.
The log records the simulator state $(q_{\textnormal{pos}}, q_{\textnormal{vel}})$ without the horizontal position, not the Gymnasium observation, which clips velocities to $\pm 10$; in this appendix $S_{t,h}$ denotes that logged state at step $h$ of offline episode $t$.

The sensor adds independent Gaussian noise to every coordinate of the logged state and of the observation the policy receives.
Its standard deviation on coordinate $j$ is $\sigma_{\mathrm{rel}}\, \sigma_j$ with $\sigma_{\mathrm{rel}} = 0.05$, where $\sigma_j$ is the standard deviation of that coordinate over $2000$ steps of the random policy on the factory model, floored at $10^{-3}$.
The same $\sigma_j$ and the same $\sigma_{\mathrm{rel}}$ apply to the offline log, to training in every simulator, and to training and evaluation on the true environment.
The reward features $\varphi$ and the score are computed from the true state.
The noise is what makes identification a statistical problem: MuJoCo is deterministic, so with an exact log a handful of transitions would pin down $\lambda^*$ to machine precision, the fitted simulator would equal $M^*$, and the Plug-In and UA rules would select the same $\theta$.
Because the same noise is applied everywhere, the only difference between $M^*$, $\widehat M$ and the ensemble members is the value of $\lambda$.

\subsection{Identification}
\label{app:deeprl-sysid}

\paragraph{Per-episode least squares.}
For one offline episode $t$ with logged transitions $(S_{t,h}, A_{t,h}, S_{t,h+1})$, at most $80$ of them evenly spaced, the residual at a candidate $\lambda$ is
\begin{align}
\varrho_{t,h}(\lambda) = \big(f_\lambda(S_{t,h}, A_{t,h}) - S_{t,h+1}\big) / \tau_t,
\end{align}
where $f_\lambda$ resets the MuJoCo model to $S_{t,h}$, sets the factors to $\lambda$, applies $A_{t,h}$ for one control step and returns the resulting state, and the division is coordinate-wise by $\tau_t$, the per-coordinate standard deviation of the logged next states in episode $t$.
The estimate minimizes $\sum_h \lVert \varrho_{t,h}(\lambda) \rVert_2^2$ by trust-region least squares started at $(1, 1, 1)$ inside the box $[0.1, 10]^3$, with finite-difference step $10^{-3}$.
The naive estimate of $\lambda$ is the mean of the per-episode estimates.

\paragraph{SIMEX.}
For each episode $t$ and each noise multiplier $\omega \in \{0, 0.5, 1, 1.5, 2\}$, Gaussian noise of standard deviation $\sqrt{\omega}\, \sigma_{\mathrm{rel}}\, \sigma_j$ is added to the logged states and next states and the per-episode estimate is recomputed, averaged over two noise draws for $\omega > 0$; this gives a curve $\widehat\lambda_t(\omega)$ per episode.
For a set $\mathcal{I}$ of episodes, the curves are averaged over $t \in \mathcal{I}$, a quadratic in $\omega$ is fitted to the average at the five $\omega$ values, and the estimate is its value at $\omega = -1$, clipped to $[0.1, 10]^3$.
Because the extrapolation is linear in the curve values, this equals the average over $\mathcal{I}$ of the per-episode corrected estimates $\widehat\lambda_t$ of the main text, up to the final clipping.
The Plug-In estimate $\widehat\lambda$ takes $\mathcal{I}$ to be all $T_{\textnormal{off}}$ episodes.
Table~\ref{tab:deeprl-app-fit} reports the bias and spread of the naive and SIMEX estimates across independent offline datasets.
The naive damping bias is between $+0.58$ and $+0.93$ at every $T_{\textnormal{off}}$ on both robots; SIMEX reduces it to between $+0.13$ and $+0.46$, at a larger spread.

\begin{table}[htbp]
\centering
\caption{Physics identification under the sensor ($\sigma_{\mathrm{rel}} = 0.05$). The truth is $\lambda^* = (1.5, 0.7, 1.4)$ on both robots. Bias and standard deviation of the estimate across independent offline datasets ($12$ datasets for the naive fit, $24$ for SIMEX); the last column is the bootstrap standard deviation of the damping estimate divided by its across-dataset standard deviation, which is $1$ for a calibrated ensemble of members.}
\label{tab:deeprl-app-fit}
\small
\begin{tabular}{@{}llrrrrrrr@{}}
\toprule
 & & \multicolumn{2}{c}{mass} & \multicolumn{2}{c}{friction} & \multicolumn{2}{c}{damping} & \\
\cmidrule(lr){3-4}\cmidrule(lr){5-6}\cmidrule(lr){7-8}
$T_{\textnormal{off}}$ & fit & bias & sd & bias & sd & bias & sd & boot/true \\
\midrule
\multicolumn{9}{@{}l}{\texttt{Hopper}} \\
$20$ & naive & $+0.008$ & $0.042$ & $+0.069$ & $0.137$ & $+0.932$ & $0.402$ & $0.71$ \\
$20$ & SIMEX & $+0.011$ & $0.200$ & $+0.162$ & $0.470$ & $+0.463$ & $0.914$ & $0.91$ \\
$80$ & naive & $+0.024$ & $0.021$ & $+0.015$ & $0.006$ & $+0.814$ & $0.122$ & $1.15$ \\
$80$ & SIMEX & $+0.022$ & $0.133$ & $+0.007$ & $0.194$ & $+0.238$ & $0.467$ & $1.01$ \\
$320$ & naive & $+0.034$ & $0.016$ & $+0.046$ & $0.029$ & $+0.815$ & $0.068$ & $1.11$ \\
$320$ & SIMEX & $+0.044$ & $0.057$ & $+0.048$ & $0.078$ & $+0.128$ & $0.155$ & $1.56$ \\
\midrule
\multicolumn{9}{@{}l}{\texttt{Walker2d}} \\
$20$ & naive & $+0.044$ & $0.025$ & $+0.166$ & $0.176$ & $+0.576$ & $0.262$ & $0.89$ \\
$20$ & SIMEX & $+0.004$ & $0.096$ & $+0.020$ & $0.525$ & $+0.289$ & $0.871$ & $0.75$ \\
$80$ & naive & $+0.057$ & $0.022$ & $+0.196$ & $0.110$ & $+0.697$ & $0.183$ & $0.79$ \\
$80$ & SIMEX & $+0.008$ & $0.047$ & $+0.048$ & $0.278$ & $+0.204$ & $0.237$ & $1.46$ \\
$320$ & naive & $+0.057$ & $0.010$ & $+0.147$ & $0.028$ & $+0.687$ & $0.053$ & $1.33$ \\
$320$ & SIMEX & $+0.014$ & $0.028$ & $+0.020$ & $0.130$ & $+0.184$ & $0.162$ & $1.08$ \\
\bottomrule
\end{tabular}
\end{table}

\subsection{UA ensemble}
\label{app:deeprl-ensemble}

Member $i = 1, \dots, n = 12$ draws a multiset $\mathcal{I}^{(i)}$ of $T_{\textnormal{off}}$ episodes from $\mathcal{D}_{\textnormal{off}}$ with replacement and uses the SIMEX estimate over $\mathcal{I}^{(i)}$, $\widehat\lambda^{(i)} = T_{\textnormal{off}}^{-1} \sum_{t \in \mathcal{I}^{(i)}} \widehat\lambda_t$, so that $\widehat M^{(i)} = M_{\widehat\lambda^{(i)}}$.
The per-episode curves are computed once per replication and reused by every member.
The last column of Table~\ref{tab:deeprl-app-fit} is the standard deviation of the members' damping estimate divided by the standard deviation of the Plug-In estimate across independent datasets.

\subsection{Selection protocol}
\label{app:deeprl-protocol}

The success score in the simulator is tabulated once per cell and read by every replication.
For each task and $T_{\textnormal{off}}$, SIMEX estimates are computed on $24$ independent offline datasets, and a $3 \times 5$ lattice is placed over the two coordinates of $\lambda$ with the largest spread, friction and damping in every cell, at the $2$nd to $98$th percentiles of the estimates; mass is held at the mean of the estimates (Table~\ref{tab:deeprl-app-lattice}).
At every lattice point $\lambda$ and every $\theta \in \Theta$, PPO is trained in $M_\lambda$ with $9$ seeds and the success score of each seed is stored.
The true success scores $J_T(\pi_\theta, M^*)$ use $30$ seeds on \texttt{Hopper} and $32$ on \texttt{Walker2d} (Table~\ref{tab:deeprl-app-truth}).

\begin{table}[htbp]
\centering
\caption{The lattice of identified physics parameters $\lambda$ on which the simulated success scores are tabulated, one per cell. The two axes are the $2$nd to $98$th percentiles of the SIMEX estimates over $24$ independent offline datasets; the third parameter is held at their mean.}
\label{tab:deeprl-app-lattice}
\small
\begin{tabular}{@{}lrlll@{}}
\toprule
task & $T_{\textnormal{off}}$ & friction axis & damping axis & mass \\
\midrule
\texttt{Hopper} & $20$ & $\{0.24, 0.72, 2.03\}$ & $\{0.42, 1.28, 1.77, 2.32, 3.77\}$ & $1.511$ \\
\texttt{Hopper} & $80$ & $\{0.34, 0.71, 1.06\}$ & $\{0.79, 1.37, 1.63, 2.05, 2.31\}$ & $1.522$ \\
\texttt{Hopper} & $320$ & $\{0.60, 0.74, 0.88\}$ & $\{1.25, 1.40, 1.52, 1.64, 1.78\}$ & $1.544$ \\
\texttt{Walker2d} & $20$ & $\{0.10, 0.58, 2.03\}$ & $\{0.14, 1.03, 1.68, 2.34, 2.95\}$ & $1.504$ \\
\texttt{Walker2d} & $80$ & $\{0.44, 0.64, 1.34\}$ & $\{1.18, 1.43, 1.60, 1.84, 1.96\}$ & $1.508$ \\
\texttt{Walker2d} & $320$ & $\{0.42, 0.73, 0.95\}$ & $\{1.24, 1.49, 1.62, 1.68, 1.81\}$ & $1.514$ \\
\bottomrule
\end{tabular}
\end{table}

One replication proceeds as follows.
\begin{enumerate}[leftmargin=*, nosep]
\item Draw $\mathcal{D}_{\textnormal{off}}$ on $M^*$ and compute $\widehat\lambda$ and $\widehat\lambda^{(1)}, \dots, \widehat\lambda^{(n)}$.
\item The success score of one simulated training run at $\lambda$ and $\theta$ is a seed drawn with replacement from the $9$ stored seeds, after bilinear interpolation of the stored table at $\lambda$ over the friction and damping axes; $\lambda$ is clipped to the lattice.
\item The Plug-In rule draws $K = 12$ runs per $\theta$ at $\widehat\lambda$ and selects $\widehat\theta_{\plug,T}$, the $\theta$ with the largest mean, as in \eqrefff{eq::plugin-selection}.
The UA rule draws $K / n = 1$ run per $\theta$ per member and selects $\widehat\theta_{\UA,T}$, the $\theta$ with the largest mean over members, as in \eqrefff{eqn:ensembleObjective}; the $10$th-percentile and Borda variants replace the mean over members by the $10$th percentile of the member values and by the mean of the members' ranks.
Both rules use $K$ simulated runs per $\theta$.
\item The success score of a rule is $J_T(\pi_{\widehat\theta}, M^*)$, read from Table~\ref{tab:deeprl-app-truth}.
\end{enumerate}
Each cell runs $300$ replications; the offline dataset and the seed draws are shared between the rules, and the paired $t$ statistic of the difference in score is reported.
Because every replication reads the same table, replications are not independent and the standard error of a single rule's score is optimistic; only paired contrasts are used as evidence.

\subsection{Results}
\label{app:deeprl-criteria}

Table~\ref{tab:deeprl-app-criteria} reports the six cells of Table~\ref{tab:deeprl-simex} under the mean, $10$th-percentile and Borda variants of the UA rule, with paired $t$ statistics, and the fraction of replications in which each rule selects $\theta_T^*$.
The mean variant has the highest score of the three in every cell.
The $10$th-percentile variant is significantly worse than the Plug-In on \texttt{Walker2d} at $T_{\textnormal{off}} = 20$; the Borda variant is never significantly different from the Plug-In.

\begin{table}[htbp]
\centering
\caption{Success score $J_T(\pi_{\widehat\theta}, M^*)$ of the weight vector selected by the Plug-In and by the UA rule under three ways of combining the members, in the six cells of Table~\ref{tab:deeprl-simex}. Mean score and the paired $t$ statistic of the difference from the Plug-In (positive favors UA; bold where $|t| > 1.96$), $300$ replications. ``share at optimum'' is the fraction of replications that select $\theta_T^*$.}
\label{tab:deeprl-app-criteria}
\small
\begin{tabular}{@{}lrrrrrrrrr@{}}
\toprule
 & & \multicolumn{2}{c}{mean} & \multicolumn{2}{c}{$10$th percentile} & \multicolumn{2}{c}{Borda} & \multicolumn{2}{c}{share at optimum} \\
\cmidrule(lr){3-4}\cmidrule(lr){5-6}\cmidrule(lr){7-8}\cmidrule(lr){9-10}
$T_{\textnormal{off}}$ & Plug-In & score & $t$ & score & $t$ & score & $t$ & Plug-In & UA \\
\midrule
\multicolumn{10}{@{}l}{\texttt{Hopper}, $J_T(\pi_{\theta_T^*}, M^*) = 0.991$} \\
$20$ & $0.933$ & $0.956$ & $\mathbf{+3.80}$ & $0.929$ & $-0.35$ & $0.926$ & $-0.89$ & $62\%$ & $69\%$ \\
$80$ & $0.965$ & $0.971$ & $+1.57$ & $0.968$ & $+0.69$ & $0.969$ & $+0.87$ & $72\%$ & $81\%$ \\
$320$ & $0.976$ & $0.981$ & $+1.91$ & $0.978$ & $+0.73$ & $0.969$ & $-1.91$ & $84\%$ & $88\%$ \\
\midrule
\multicolumn{10}{@{}l}{\texttt{Walker2d}, $J_T(\pi_{\theta_T^*}, M^*) = 0.689$} \\
$20$ & $0.585$ & $0.594$ & $+1.25$ & $0.558$ & $\mathbf{-3.45}$ & $0.571$ & $-1.71$ & $47\%$ & $53\%$ \\
$80$ & $0.594$ & $0.616$ & $\mathbf{+3.32}$ & $0.590$ & $-0.72$ & $0.606$ & $+1.76$ & $47\%$ & $59\%$ \\
$320$ & $0.590$ & $0.610$ & $\mathbf{+2.96}$ & $0.593$ & $+0.50$ & $0.603$ & $+1.91$ & $44\%$ & $54\%$ \\
\bottomrule
\end{tabular}
\end{table}

\label{app:deeprl-moves}
Figure~\ref{fig:cond-picks-deeprl} shows the UA selection given the Plug-In selection in every cell, with the weight vectors sorted by $\theta_{\textnormal{fwd}} / \theta_{\textnormal{ctrl}}$.
Relative to $\theta_T^*$, given a Plug-In selection with a larger ratio, UA selects a smaller ratio in $55$ to $95\%$ of those datasets; given a smaller ratio, UA selects a larger one in $74$ to $100\%$.
Selections with a larger ratio than $\theta_T^*$ account for $74$ to $98\%$ of the Plug-In's shortfall in score below $\theta_T^*$.

\begin{figure}[htbp]
    \centering
    \includegraphics[width=\linewidth]{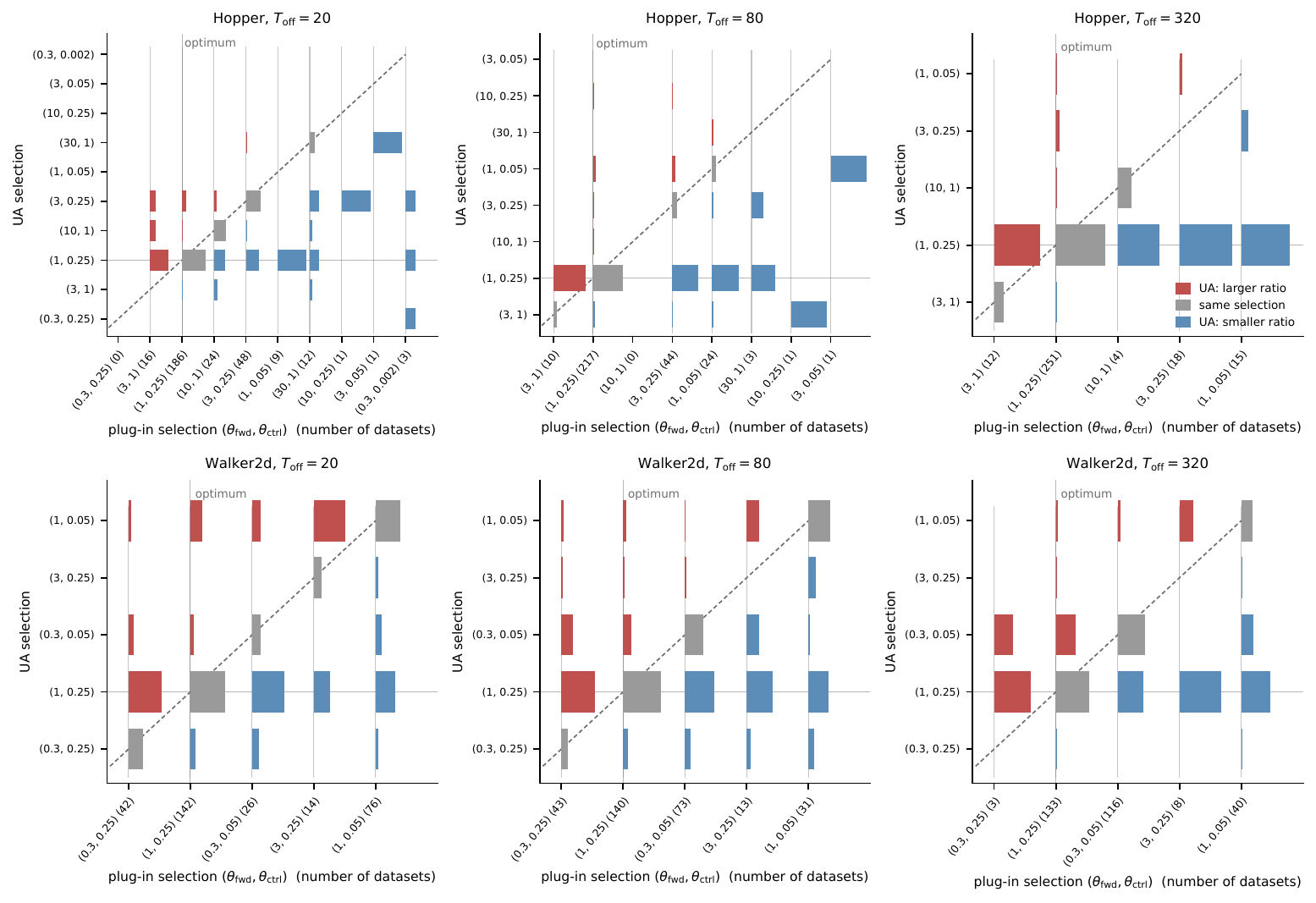}
    \caption{The UA selection $\widehat\theta_{\UA,T}$ given the Plug-In selection $\widehat\theta_{\plug,T}$, in the six cells of Table~\ref{tab:deeprl-simex}, $300$ paired offline datasets each.
    The weight vectors on both axes are sorted by the ratio $\theta_{\textnormal{fwd}} / \theta_{\textnormal{ctrl}}$; a larger ratio pays more for moving forward and charges less for actuation, so the trained policy runs harder.
    Only weight vectors selected at least once are shown.
    Each column is one Plug-In selection, with the number of datasets in parentheses; the bars in a column give the conditional distribution of the UA selection at its position on the vertical axis.
    Bars on the dashed diagonal mean the two rules agree; red bars above it mean UA selected a larger ratio than the Plug-In did, blue bars below it a smaller one.}
    \label{fig:cond-picks-deeprl}
\end{figure}

\section{Theoretical analysis of the UCB algorithm}
\label{app:ucb}

In this section, we provide a detailed theoretical analysis of the behavior and selection performance of UCB$(\theta)$ in a two-armed Gaussian bandit, building on the setup in Section \ref{sec:mab}. For completeness, we restate the algorithm in Alg.~\ref{alg:ucb}.

\begin{algorithm}[t]
\caption{UCB$(\theta)$}
\label{alg:ucb}
\begin{algorithmic}[1]
\STATE Pull each arm once.
\FOR{$t=3,\ldots,T$}
    \STATE For $i\in\{1,2\}$, compute $N_i(t-1;\theta)$ and $\widehat\mu_i(t-1;\theta)$, the number of pulls and empirical mean reward for arm $i$ up to time $t-1$.
    \STATE Pull $A_t= A_t(\theta)\in \arg\max_{i\in\{1,2\}} U_i(t;\theta)$, where $U_i(t;\theta): = \widehat\mu_i(t-1;\theta)
    + \sqrt{\frac{2\theta\log t}{N_i(t-1;\theta)}}$.
\ENDFOR
\end{algorithmic}
\end{algorithm}

We first introduce some additional notation.

\noindent\textbf{Notation.} For $x\in\RR$ and a compact interval $\cV = [\underline v,\overline v]\subset(0,\infty)$, define $\Pi_{\cV}(x):=\min\{\overline v,\max\{\underline v,x\}\}$. For a random element $X$, we write $\sigma(X)$ for the sigma-field generated by $X$. Let $\Phi$ and $\phi$ denote the cumulative distribution function and density of the standard normal distribution. For $a\in\RR$, write $a_+:=\max\{a, 0\}$.

For a generic two-armed Gaussian environment $\nu=\bigl(\mathcal N(\mu_1,\tau_1^2), \mathcal N(\mu_2,\tau_2^2)\bigr)$ with $\mu_1\neq\mu_2$ and $\tau_1^2,\tau_2^2>0$, let $o(\nu)$ and $s(\nu)$ denote its optimal and suboptimal arm, respectively, and let $\delta(\nu):=\mu_{o(\nu)}-\mu_{s(\nu)}>0$. Throughout this section, we let $\pi_\theta$ denote the policy induced by UCB$(\theta)$ in Alg. \ref{alg:ucb}; in addition, to simplify writing, $\forall \theta>0$, we denote the expected regret of UCB$(\theta)$ (Alg. \ref{alg:ucb}) in environment $\nu$ over horizon $T$ as
\begin{equation}
\Reg_T(\theta;\nu)
:=\Reg_T(\pi_\theta, \nu) = \delta(\nu)\cdot
\EE_{\nu}\bigl[N_{s(\nu)}(T;\theta)\bigr].
\label{eq:generic-risk}
\end{equation}
Here, $N_i(T;\theta)$ denotes the number of times arm $i$ is selected over $T$ rounds when $\mathrm{UCB}(\theta)$ is deployed. When the two means are equal, we define $\Reg_T(\theta;\nu)=0$.

Without loss of generality, we assume that given the fitted environment $\widehat M$ obtained from the offline data $\cD_{\off}$, the Plug-In rule selects the smallest minimizer of the expected regret in $\widehat M$:
\begin{equation}\label{eq:plugin-selector}
\widehat \theta_{\plug, T} = \widehat \theta_{\plug, T}(\cD_{\off}):=\min\argmin_{\theta\in\Theta} \Reg_{T}(\theta;\widehat M).
\end{equation}
To define the UA objective, let $Z=(Z_1,Z_2)$ with $Z_1,Z_2\stackrel{\mathrm{iid}}{\sim}\cN(0,1)$, independent of $\cD_{\off}$. Define the generic perturbed environment
\[
\widetilde M(Z)
:=
\left(
\mathcal N\!\left(
    \widetilde\mu_1(Z_1),\widehat\sigma_1^2
\right),
\mathcal N\!\left(
    \widetilde\mu_2(Z_2),\widehat\sigma_2^2
\right)
\right),
\]
where $\tilde \mu_{i} (Z_i) := \widehat\mu_i+ \frac{\widehat\sigma_i}{\sqrt{T_{1,\textnormal{off}}}}Z_i$ for each arm $i\in\{1, 2\}$. The Uncertainty-Aware rule selects the smallest minimizer of the expected regret over an ensemble of perturbed environments:
\begin{equation}
\widehat \theta_{\UA, T} = \widehat \theta_{\UA, T}(\cD_{\off}):=\min\argmin_{\theta\in\Theta}
\Reg_{\UA, T}(\theta; \cD_{\off}),
\label{eq:uq-selector}
\end{equation}
where $\Reg_{\UA, T}(\theta; \cD_{\off}): = \EE_Z[\textnormal{Reg}_T(\theta;\widetilde M(Z))\mid \mathcal D_{\textnormal{off}}]$. This definition is consistent with the main paper.

The existence of $\widehat \theta_{\plug,T}$ and $\widehat \theta_{\UA,T}$ is established in Proposition \ref{prop:existence}.

Finally, note that the two relevant deployment regret terms in (\ref{eq:G}), $\Reg_T(\widehat\theta_{\plug, T};M^*)$ and $\Reg_T(\widehat\theta_{\UA, T};M^*)$, are random variables that depend on the offline data $\cD_{\off}$ for each $T\ge 3$. For notational simplicity, we make this dependence explicit and write
$$
r_{\plug, T}(\cD_{\off}) = \Reg_T(\widehat\theta_{\plug, T};M^*), \quad
r_{\UA, T}(\cD_{\off}) = \Reg_T(\widehat\theta_{\UA, T};M^*).
$$
Thus, $G(\cD_{\off}) = \liminf_{T\to\infty}\frac{\log r_{\plug, T}(\cD_{\off})}{\log r_{\UA, T}(\cD_{\off})}$ is also a random variable that depends on $\cD_{\off}$.

\subsection{Main results and proofs}

In this section, we first establish the existence of $\widehat \theta_{\plug,T}$ and $\widehat \theta_{\UA,T}$ in Proposition~\ref{prop:existence}. We then characterize the asymmetric regret landscape in Theorem \ref{thm::UCB-regret}, followed by an analysis of the $\theta$ selected by the Plug-In and UA approaches in Theorem \ref{thm:main}.

\subsubsection{Existence of $\widehat \theta_{\plug,T}$ and $\widehat \theta_{\UA,T}$}

\begin{proposition}
\label{prop:existence}
Consider the Gaussian environment $M^*$ defined in Section \ref{sec:mab}, and assume $T_{\off}\ge 4$. Then for almost every dataset $\cD_{\off}$, the following holds for any finite horizon $T$:
\begin{itemize}
\item[(i)] $\theta\mapsto \Reg_T(\theta;\widehat M)$ is continuous on $\Theta$;
\item[(ii)] $\theta\mapsto\Reg_{\UA, T}(\theta; \cD_{\off})$ is continuous on $\Theta$.
\end{itemize}
Consequently, both argmin sets in (\ref{eq:plugin-selector}) and (\ref{eq:uq-selector}) are nonempty and compact.
\end{proposition}
\begin{proof}[Proof of Proposition \ref{prop:existence}]
Fix any $\cD_{\off}$ such that $\widehat\sigma_1^2,\widehat\sigma_2^2>0$ and $\widehat\mu_1\neq\widehat\mu_2$. We will show (i) and (ii) hold.

We first prove a generic finite-horizon continuity statement. Fix a Gaussian environment $\nu$ with positive arm variances and distinct means.  Place two independent infinite reward sequences on one probability space and use the same sequences to run UCB$(\theta)$ for every $\theta\in\Theta$.  Fix $\theta_0\in\Theta$. We claim that the complete action sequence through any fixed time $T$ is almost surely locally constant in $\theta$ at $\theta_0$.

The claim follows by induction over time. It is trivial for the two initialization pulls. Suppose the claim holds through time $t-1$. Then with probability 1, the same reward observations, pull counts $N_{i}(t-1;\theta)$, and empirical means $\widehat\mu_i(t-1;\theta)$ used at time $t$ are the same throughout a neighborhood of $\theta_0$. Here $N_{i}(t-1;\theta)$ denotes the number of pulls at arm $i$ if we deploy UCB$(\theta)$ for $t-1$ decision points, and $\widehat\mu_i(t-1;\theta)$ denotes the mean reward of all the pulls at arm $i$ if we deploy UCB$(\theta)$ for $t-1$ decision points. At time $t$, the difference between the two UCB indices is
\begin{align*}
D_t(\theta)&:= U_1(t;\theta) - U_2(t;\theta)\\
& = \widehat\mu_1(t-1;\theta)-\widehat\mu_2(t-1;\theta)
+\sqrt{2\theta\log t}
\left\{
\frac{1}{\sqrt{N_1(t-1;\theta)}}-
\frac{1}{\sqrt{N_2(t-1;\theta)}}
\right\}.
\end{align*}
Therefore, with probability 1, $\theta\mapsto D_t(\theta)$ is continuous at $\theta_0$.

In addition, we also have $\PP(D_t(\theta_0) = 0) = 0$. This is because for each possible action sequence through time $t-1$, the event that this action sequence occurs and $D_t(\theta_0)=0$ is contained in an affine hyperplane of the finitely many Gaussian rewards revealed along that sequence. Since these rewards have a joint density, this event has probability zero. There are only finitely many possible action sequences through time $t-1$, so taking their union still gives $\PP(D_t(\theta_0) = 0) = 0$.

Combining the above arguments, we have that with probability 1, $D_t(\theta_0)\neq0$, and $\theta\mapsto D_t(\theta)$ is continuous at $\theta_0$. Therefore, with probability 1, the sign of $D_t(\theta)$ is unchanged in a small enough neighborhood of $\theta_0$. This implies that with probability 1, UCB$(\theta)$ selects the same arm at time $t$ as UCB$(\theta_0)$ throughout that neighborhood. This completes the induction.

It follows that
\[
N_{s(\nu)}(T;\theta)\to N_{s(\nu)}(T;\theta_0)
\quad\text{almost surely as }\theta\to \theta_0.
\]
The count is bounded by $T$, so dominated convergence gives continuity of $\theta\mapsto \Reg_T(\theta;\nu)$.

Applying this result to $\nu=\widehat M$ proves part (i).  For part (ii), conditional on $\cD_{\off}$, the perturbed signed gap $\widetilde\mu_{1}(Z_1)-\widetilde\mu_{2}(Z_2)$ is a nondegenerate Gaussian random variable. Thus, it is nonzero for almost every $Z$, and for almost every $Z$, the preceding argument gives continuity of $\theta\mapsto \Reg_T(\theta;\tilde M(Z))$. Moreover,
\begin{equation}
0\le \Reg_T(\theta;\tilde M(Z))
\le T\,\bigl|\widetilde\mu_{1}(Z_1)-
                 \widetilde\mu_{2}(Z_2)\bigr|,
\label{eq:uq-dominating-function}
\end{equation}
and the right-hand side is integrable with respect to $Z$. We can use the dominated convergence theorem to prove continuity of the averaged objective in (\ref{eq:uq-selector}). Finally, since $\Theta$ is compact, both objectives in (\ref{eq:plugin-selector}) and (\ref{eq:uq-selector}) have nonempty compact argmin sets.
\end{proof}

\subsubsection{Asymmetric regret of UCB}

\begin{theorem}[Asymmetric regret of UCB]\label{thm::UCB-regret}
Consider the two-armed Gaussian bandit environment $M^*=\bigl(\cN(\mu_1,\sigma^2),\cN(\mu_2,\sigma^2)
\bigr)$, where $\mu_1>\mu_2$, $\sigma^2>0$. For $\theta>0$, let $\pi_\theta$ denote the policy induced by UCB$(\theta)$ in Alg.~\ref{alg:ucb}. Then the following hold for its expected regret:
\begin{itemize}
\item[(i)] For any $\theta\in(0,\sigma^2)$,
\[
\lim_{T\to\infty}
\frac{\log\Reg_T(\pi_\theta,M^*)}{\log T}
=
1-\frac{\theta}{\sigma^2}.
\]

\item[(ii)] For any $\theta\ge\sigma^2$,
\[
\lim_{T\to\infty}
\frac{\Reg_T(\pi_\theta,M^*)}{\log T}
=
\frac{2\theta}{\mu_1-\mu_2}.
\]
\end{itemize}
Thus, $\theta=\sigma^2$ separates two asymptotic regimes:
the expected regret grows polynomially in $T$ for
$\theta<\sigma^2$ and logarithmically in $T$ for
$\theta\ge\sigma^2$.
\end{theorem}

\begin{proof}[Proof of Theorem \ref{thm::UCB-regret}]
Apply Lemma~\ref{lem:ucb-bounds} to the true environment $M^*$ with $\delta=\mu_1-\mu_2$ and $\tau_o^2=\tau_s^2=\sigma^2$. For $0<\theta<\sigma^2$, part (i) of the theorem follows from Lemma~\ref{lem:ucb-bounds}(v) by taking the constant sequence $\theta_T\equiv\theta$. For $\theta>\sigma^2$, choose a compact interval $\mathcal K$ containing $\theta$ with $\inf\mathcal K>\sigma^2$. Lemma~\ref{lem:ucb-bounds}(iii), together with $\Reg_T(\theta; M^*)=\delta \mathbb{E}_{M^*}[N_s(T;\theta)]$, gives the desired limit. The boundary case $\theta=\sigma^2$ follows directly from Lemma~\ref{lem:ucb-bounds}(iv).
\end{proof}

\subsubsection{Plug-In and UA selection for UCB}

Let
\[
\widehat a:=\arg\max_{i\in\{1,2\}} \widehat\mu_i
\]
denote the optimal arm in the fitted environment obtained from offline data $\cD_{\off}$.

We prove the following stronger theorem.

\begin{theorem}\label{thm:main}
For any $T_{\off}\ge4$, on the probability-one event
\begin{equation}
\cR_{T_{\off}}
:=\left\{
\begin{array}{l}
\widehat\sigma_1^2>0, \widehat\sigma_2^2>0,
\widehat\mu_1\neq\widehat\mu_2,\\[2pt]
\widehat\sigma_i^2\notin\{\underline \theta,\sigma^2,\overline \theta\}
\text{ for }i=1,2
\end{array}
\right\},
\label{eq:regular-event}
\end{equation}
the following statements hold for the Plug-In and UA selection of the UCB algorithm:

\begin{enumerate}[label=\textup{(\roman*)}]

\item As $T\to\infty$,
\begin{equation}
\widehat \theta_{\plug, T}\to c_{\plug}(\cD_{\off}):=\Pi_{{\Theta}}(\widehat\sigma_{\hat a}^2),
\qquad
\liminf_{T\to\infty}\widehat \theta_{\UA, T}
\ge c_{\UA}(\cD_{\off}):=\Pi_{\Theta}(\widehat\sigma_1^2\vee \widehat\sigma_2^2).
\label{eq:selector-limits}
\end{equation}

\item The Plug-In conditional regret satisfies
\begin{equation}
\lim_{T\to\infty}
\frac{\log\bigl(r_{\plug, T}(\cD_{\off})\bigr)}{\log T}
=
\left(1-\frac{c_{\plug}(D^\off)}{\sigma^2}\right)_+,
\label{eq:PI-exponent}
\end{equation}
and the Uncertainty-Aware conditional regret satisfies
\begin{equation}
\limsup_{T\to\infty}
\frac{\log\bigl(r_{\UA, T}(\cD_{\off})\bigr)}{\log T}
\le
\left(1-\frac{c_{\UA}(\cD_{\off})}{\sigma^2}\right)_+
\le
\left(1-\frac{c_{\plug}(\cD_{\off})}{\sigma^2}\right)_+.
\label{eq:UQ-exponent}
\end{equation}
\item $G(\cD_{\off})\geq 1$.
\end{enumerate}
In addition, define the event $\mathcal A_{T_\off}
:=\left\{\widehat\sigma_{\widehat a}^2<\sigma^2<\widehat\sigma_{3-\widehat a}^2\right\}\cap \cR_{T_\off}$. Then  $\lim_{T_\off\to\infty}\PP(\mathcal A_{T_\off})=\frac14$, and on $\mathcal A_{T_\off}$, $G(\cD_{\off})=+\infty$.
\end{theorem}

\begin{proof}[Proof of Theorem~\ref{thm:main}]

For part (i), first apply Lemma~\ref{lem:plugin-limit} to the fitted environment $\widehat M$. Its optimal arm is $\widehat a$, and the optimal-arm variance is $\widehat\sigma_{\widehat a}^2$. Also, the regularity event excludes equality with the endpoints of $\Theta$, so
\[
\widehat {\theta}_{\plug, T}\to\Pi_{{\Theta}}(\widehat\sigma_{\widehat a}^2)=c_{\plug}({\cD_{\off}}).
\]
The UA liminf part follows from Lemma~\ref{lem:integrated-exponent}.

For (\ref{eq:PI-exponent}) in part (ii), we deploy the selected sequence $\widehat {\theta}_{\plug, T}$ in the true environment $M^*$. Its optimal-arm variance is $\sigma^2$, and the selected sequence converges to $c_{\plug}({\cD_{\off}})$.  Lemma~\ref{lem:ucb-bounds}(v) then leads to (\ref{eq:PI-exponent}).

For (\ref{eq:UQ-exponent}), fix $\varepsilon>0$. By part (i), for sufficiently large $T$, $\widehat {\theta}_{\UA, T}\ge c_{\UA}({\cD_{\off}})-\varepsilon$. Apply Lemma~\ref{lem:ucb-bounds}(i) uniformly over all $\overline\theta\ge {\theta}\ge c_{\UA}({\cD_{\off}})-\varepsilon$, we deduce that
\[
\limsup_{T\to\infty}
\frac{\log\bigl(r_{\UA, T}({\cD_{\off}})\bigr)}{\log T}
\le
\left(1-\frac{c_{\UA}({\cD_{\off}})-\varepsilon}{\sigma^2}\right)_+.
\]
Let $\varepsilon\downarrow0$, and since $c_{\UA}(\cD_{\off})\ge c_{\plug}(\cD_{\off})$, (\ref{eq:UQ-exponent}) follows.

For part (iii), first apply Lemma~\ref{lem:ucb-bounds}(iii) in the true environment uniformly over ${\theta}\in{\Theta}$ gives a constant $c({\cD_{\off}})>0$ such that
\begin{equation}
r_{\plug, T}({\cD_{\off}})\wedge r_{\UA, T}({\cD_{\off}})
\ge c({\cD_{\off}})\log T
\label{eq:both-diverge-log}
\end{equation}
for all sufficiently large $T$.

If $c_{\plug}({\cD_{\off}})>\sigma^2$, then $c_{\UA}({\cD_{\off}})\ge c_{\plug}({\cD_{\off}})>\sigma^2$. Then (i) implies that  both selected sequences have a fixed positive distance above $\sigma^2$ for all large $T$. Lemma \ref{lem:ucb-bounds}(iii) and (\ref{eq:both-diverge-log}), yields
\[
r_{\plug, T}(\cD_{\off})=\Theta(\log T),
\qquad
r_{\UA, T}(\cD_{\off})=\Theta(\log T).
\]
In this case, $G({\cD_{\off}})=1$.

If $c_{\plug}({\cD_{\off}})<\sigma^2$, set
$q=1-c_{\plug}({\cD_{\off}})/\sigma^2>0$. By
(\ref{eq:PI-exponent}), $\log(r_{\plug, T}({\cD_{\off}}))$ is $q\log T+o(\log T)$.  By (\ref{eq:UQ-exponent}), $\limsup_{T\to\infty} \frac{\log r_{\UA, T}({\cD_{\off}})}{\log T}\le q$. Together with (\ref{eq:both-diverge-log}), this implies $\liminf_{T\to\infty}\frac{\log r_{\plug, T}({\cD_{\off}})}{\log r_{\UA, T}({\cD_{\off}})}\ge 1$. The remaining equality case $c_{\plug}({\cD_{\off}})=\sigma^2$ is excluded by $\cR_{T_\off}$.  This proves (iii).

Finally, we analyze the event $\cA_{T_\off}$. Note that for Gaussian data, the sample means are independent of the sample variances. Thus, $\widehat a$ is independent of $(\widehat\sigma_1^2,\widehat\sigma_2^2)$. In addition, the two sample variances are i.i.d., and $\frac{(T_{1, \off}-1)\widehat\sigma_i^2}{\sigma^2}\sim\chi^2_{T_{1, \off}-1}$. Therefore, condition on either value of $\widehat a$, the probability that the fitted optimal arm has variance below $\sigma^2$ and the other arm has variance above $\sigma^2$ is $q_{T_\off}(1-q_{T_\off})$, where $q_{T_\off}:= {\PP}(\chi^2_{T_{1, \off}-1}<T_{1, \off}-1)$. This leads to ${\PP}(\mathcal A_{T_{\off}})=q_{T_\off}(1-q_{T_\off})$. Moreover, according to the Central Limit Theorem, $\lim_{T_{\off}\to\infty}q_{T_\off} = 1/2$. Thus, $\PP(\cA_{T_\off})\to 1/4$.

On $\mathcal A_{T_\off}$, the fitted optimal arm has sample variance below $\sigma^2$, and the other arm has sample variance above $\sigma^2$. Since $\underline {\theta}<\sigma^2<\overline {\theta}$,
\[
c_{\plug}({\cD_{\off}})<\sigma^2<c_{\UA}({\cD_{\off}}).
\]
From the previous analysis we obtain that on $\mathcal A_{T_\off}$,
\begin{align*}
r_{\plug, T}({\cD_{\off}})
&=T^{\,1-c_{\plug}({\cD_{\off}})/\sigma^2+o(1)},\\
r_{\UA, T}({\cD_{\off}})&=O(\log T),
\end{align*}
and hence
\begin{equation*}
\frac{r_{\plug, T}({\cD_{\off}})}{r_{\UA, T}({\cD_{\off}})}
\to\infty, \qquad G({\cD_{\off}})=\infty.
\end{equation*}

\end{proof}

\subsection{Auxiliary lemmas}

\begin{lemma}[Fixed-environment UCB regret]
\label{lem:ucb-bounds}
Consider a 2-arm bandit environment $\nu_\delta=\bigl(\mathcal N(\mu_o,\tau_o^2),\mathcal N(\mu_s,\tau_s^2)\bigr)$, where $\mu_o = \mu_s+\delta$ with $\delta>0$. Let $\cK=[\theta_-,\theta_+]\subset(0,\infty)$. Then the following holds for the regret of UCB$(\theta)$ and the arm pull counts in environment $\nu_\delta$:

\begin{itemize}
\item[(i)] There is a finite constant $C$, depending only on $\theta_-,\theta_+,\tau_o^2,\tau_s^2$, such that for all $T\ge3$, $\theta\in \cK$, and $\delta>0$:
\begin{equation}
\Reg_T(\theta;\nu_\delta) \le C\left[
\delta+\frac{\log T}{\delta} +\frac1\delta \left(1+\sum_{t=3}^T t^{-\theta/\tau_o^2}\right)
\right].
\label{eq:fixed-upper-envelope}
\end{equation}

\item[(ii)] Fix $0<\delta_-\le\delta_+<\infty$ and suppose $\theta_+<\tau_o^2$. For every $\eta>0$, there is $T_0<\infty$ such that, for all $T\ge T_0$,
\begin{equation}
\inf_{\substack{\delta\in[\delta_-,\delta_+]\\\theta\in \cK}}
\Reg_T(\theta;\nu_\delta)
\ge \frac{\delta_-}{4}
T^{\,1-\theta_+/\tau_o^2-\eta}.
\label{eq:fixed-catastrophic-lower}
\end{equation}

\item[(iii)] For every fixed $\delta>0$,
\begin{equation}
\liminf_{T\to\infty}\inf_{\theta\in \cK}
\frac{\delta^2\EE_{\nu_\delta}[N_s(T;\theta)]}
{2\theta\log T}\ge1.
\label{eq:forced-exploration-lower}
\end{equation}
If, in addition, $\theta_-\ge(1+\xi)\tau_o^2$ for some $\xi>0$, then
\begin{equation}
\lim_{T\to\infty}\sup_{\theta\in \cK}
\left|
\frac{\delta^2\EE_{\nu_\delta}[N_s(T;\theta)]}{2\theta\log T}-1
\right|=0.
\label{eq:supercritical-asymptotic}
\end{equation}

\item[(iv)] For every fixed $\delta>0$, when $\theta=\tau_o^2$,
\begin{equation}
\label{eq:ucb-critical-regret}
\lim_{T\to\infty}
\frac{\operatorname{Reg}_T(\theta;\nu_\delta)}{\log T}
=
\frac{2\theta}{\delta}.
\end{equation}

\item[(v)] If $\theta_T\in \cK$ and $\theta_T\to \theta$, then
\begin{equation}
\lim_{T\to\infty}
\frac{\log\bigl(\Reg_T(\theta_T;\nu_\delta)\bigr)}{\log T}
=\left(1-\frac{\theta}{\tau_o^2}\right)_+.
\label{eq:fixed-regret-exponent}
\end{equation}
\end{itemize}
\end{lemma}

\begin{proof}[Proof of Lemma \ref{lem:ucb-bounds}]

Without loss of generality, we assume that any UCB algorithm is deployed by placing the two independent reward sequences on a common probability space and revealing the reward $Y_{i,t}$ when arm $i$ is selected for the $t$th time. If $\mathrm{UCB}(\theta)$ is deployed, let the arm means and pull counts be
\[
N_i(t;\theta):=\sum_{r=1}^t \ind\{A_r(\theta)=i\},
\qquad
\overline Y_{i,n}:=\frac{1}{n}\sum_{s=1}^n Y_{i,s},
\qquad
\widehat\mu_i(t; \theta):=\overline Y_{i,N_i(t; \theta)}.
\]

\textbf{Part (i).} We establish this part using a pull-count decomposition similar to that in \citet{auer2002, fan2022typical}, combined with Gaussian concentration bounds. Fix $\theta\in \cK$ and let $m_T(\theta,\delta):=\lceil\frac{8\theta\log T}{\delta^2}\rceil$. For $t\ge3$, define
\[
B_t(\theta):=\left\{U_o(t;\theta)<\mu_o-\frac{\delta}{4}
\right\},
\]
where $U_i(t;\theta) = \widehat\mu_i(t-1;\theta) + \sqrt{\frac{2\theta\log t}{N_i(t-1;\theta)}}$ is the UCB index of arm $i$ at time $t$. If for $m\in\NN$, $N_o(t-1;\theta)=m$ and $B_t(\theta)$ occurs, then $\overline Y_{o,m}-\mu_o
< -\frac{\delta}{4}-\sqrt{\frac{2\theta\log t}{m}}$. Gaussian concentration and a union bound over $m$ give
\begin{align}
{\PP} (B_t(\theta))
&\le \sum_{m=1}^{t-1}
\exp\left\{-\frac{m}{2\tau_o^2}
\left(\frac{\delta}{4}+\sqrt{\frac{2\theta\log t}{m}}\right)^2
\right\}\notag\\
&\le t^{-\theta/\tau_o^2}
\sum_{m=1}^{\infty}
\exp\left\{-\frac{m\delta^2}{32\tau_o^2}\right\}\notag\\
&\le C_1\tau_o^2\delta^{-2}t^{-\theta/\tau_o^2}.
\label{eq:Bt-upper}
\end{align}
Here $C_1$ is a universal constant. Here we use the fact that $e^x-1\ge x$ for $x\geq 0$.

Suppose that arm $s$ is selected at time $t$, $B_t(\theta)$ does not occur, and
$N_s(t-1;\theta)=k\ge m_T(\theta,\delta)$.  Then $U_s(t;\theta)\ge U_o(t;\theta)\ge\mu_o-\frac{\delta}{4}$, and $\sqrt{\frac{2\theta\log t}{k}}
\le\sqrt{\frac{2\theta\log T}{m_T(\theta,\delta)}}
\le\frac{\delta}{2}$. Thus we deduce that $\overline Y_{s,k}-\mu_s\ge\delta/4$. Each pre-pull count $k$ is encountered
at most once, so
\begin{align}
N_s(T;\theta)
&\le m_T(\theta,\delta)+1
+\sum_{t=3}^T\ind\{B_t(\theta)\}
+\sum_{k=m_T(\theta,\delta)}^{\infty}
\ind\left\{\overline Y_{s,k}-\mu_s\ge\frac{\delta}{4}\right\}.
\label{eq:count-decomposition}
\end{align}
The final sum has expectation at most
\begin{align*}
\sum_{k=m_T(\theta,\delta)}^\infty
\exp\left\{-\frac{k\delta^2}{32\tau_s^2}\right\}
\leq \frac1{1-\exp\left\{-\frac{\delta^2}{32\tau_s^2}\right\}}\leq 1+\frac{32\tau_s^2}{\delta^{2}}.
\end{align*}
Taking expectations in (\ref{eq:count-decomposition}), using
(\ref{eq:Bt-upper}), and multiplying by $\delta$ proves (\ref{eq:fixed-upper-envelope}).  If $m_T(\theta,\delta)>T$, the trivial bound
${\Reg}_T(\theta;\nu_\delta)\le T\delta$ is already bounded by a constant multiple of
$\delta+(\log T)/\delta$, so the same conclusion holds.

\textbf{Part (ii).} Choose $G>0$ large enough such that $\sum_{k=1}^{\infty}
\exp\left\{-\frac{kG^2}{2\tau_s^2}\right\}
\le\frac12$. Let
\[
F:=\left\{\overline Y_{s,k}>\mu_s-G\text{ for every }k\ge1\right\}.
\]
Then a union bound implies ${\PP} (F)\ge1/2$. Moreover, define
\begin{equation}
E_T:=\left\{
Y_{o,1}\le
\mu_s-G-\sqrt{2\theta_+\log T}
\right\}.
\label{eq:bad-first-reward}
\end{equation}
On $E_T\cap F$, after the two initialization pulls, the optimal arm has index at most $\mu_s-G$ at every $t\le T$ as long as it is not selected again. The suboptimal arm has index strictly larger than $\mu_s-G$ at every time by the definition of $F$. Thus, we can use induction to show that the suboptimal arm is selected at every time $3,\ldots,T$, so $N_s(T;\theta)=T-1$ simultaneously for every $\theta\in \cK$.

Now for any $\eta>0$, uniformly over $\delta\in[\delta_-,\delta_+]$,
\[
{\PP} (E_T)
=\Phi\left(
-\frac{\delta+G+\sqrt{2\theta_+\log T}}{\tau_o}
\right)
\ge T^{-\theta_+/\tau_o^2-\eta}
\]
for all sufficiently large $T$. Here $\Phi$ is the CDF of a standard normal distribution. The above inequality follows directly from the standard property of Gaussian distribution: $\forall x>0$, $\Phi(-x)\geq \frac{x}{1+x^2}\cdot \frac{1}{\sqrt{2\pi}}\exp(-x^2/2)$, combined with the fact that $e^{-O(\sqrt{\log T})}=T^{-o(1)}$. In addition, the events $E_T$ and $F$ involve different armwise reward sequences and are independent. Therefore,
\[
{\Reg}_T(\theta;\nu_\delta) \ge\delta_-(T-1){\PP} (E_T){\PP} (F)
\ge\frac{\delta_-}{4}
T^{1-\theta_+/\tau_o^2-\eta}
\]
for all large $T$, which proves (\ref{eq:fixed-catastrophic-lower}).

\textbf{Part (iii).} We first prove (\ref{eq:forced-exploration-lower}). Fix $\rho\in(0,1)$ and define $m_{\max}(T):=
\left\lceil\frac{2\theta_+(1-\rho)\log T}{\delta^2}\right\rceil = o(T)$. Choose small constants $\zeta,\epsilon>0$ such that $b:=\sqrt{1-\zeta}-(1+2\epsilon)\sqrt{1-\rho}>0$. Next, set $\ell_0(T)=T-2m_{\max}(T)-1$ and define
\begin{align*}
G_{o,T}
&:=\bigcap_{\ell=\ell_0(T)}^T
\left\{\overline Y_{o,\ell}-\mu_o\le\epsilon\delta\right\},\\
G_{s,T}
&:=\bigcap_{k=1}^{m_{\max}(T)}
\left\{\sum_{j=1}^k(Y_{s,j}-\mu_s)
>-b\sqrt{2\theta_-k\log T}\right\}.
\end{align*}
A Gaussian union bound gives
\begin{align*}
{\PP} (G_{o,T}^c)
&\le\sum_{\ell = \ell_0(T)}^T\exp\left\{-\frac{\ell\epsilon^2\delta^2}{2\tau_o^2}\right\}\\
&\leq (2m_{\max}(T) + 2)\cdot\exp\left\{-\frac{\ell_0(T)\epsilon^2\delta^2}{2\tau_o^2}\right\}\to 0
\end{align*}
as $T\to\infty$. Also, for every $k\le m_{\max}(T)$,
\[
\PP\left(
\sum_{j=1}^k(Y_{s,j}-\mu_s)
\le-b\sqrt{2\theta_-k\log T}
\right)
\le T^{-b^2\theta_-/\tau_s^2},
\]
where we use the fact that for $x\geq 0$, $\Phi(-x)\leq \exp(-x^2/2)$. Therefore, as $T\to\infty$,
\[
{\PP} (G_{s,T}^c)
\le m_{\max}(T)T^{-b^2\theta_-/\tau_s^2}
\to 0.
\]
Also, $\sup_{t\le T,\ell\ge\ell_0(T),\theta\in \cK} \sqrt{\frac{2\theta\log t}{\ell}} \le\epsilon\delta$ for all sufficiently large $T$.

We now show that on $G_{o,T}\cap G_{s,T}$,
$N_s(T;\theta)\ge m_-(\theta,T):=\lfloor\frac{2\theta(1-\rho)\log T}{\delta^2}\rfloor$ for every $\theta\in \cK$. Suppose otherwise. Let $r$ be the last time the optimal arm is selected.  Since fewer than $m_{\max}(T)$ suboptimal
pulls occur, $r\ge T-m_{\max}(T)$, and the optimal-arm pre-pull count at time
$r$ is at least $\ell_0(T)$.  Hence
\[
U_o(r;\theta)\le\mu_o+2\epsilon\delta.
\]
Writing $k=N_s(r-1;\theta)<m_-(\theta,T)$ and using
$\log r\ge(1-\zeta)\log T$ for all large $T$, we obtain
\begin{align*}
&\sqrt{2\theta k\log r}-(1+2\epsilon)\delta k\\
\quad\ge&
\sqrt{2\theta k\log T}
\left[\sqrt{1-\zeta}-(1+2\epsilon)
\sqrt{\frac{\delta^2k}{2\theta \log T}}\right]\\
\quad\ge& b\sqrt{2\theta k\log T}
\ge b\sqrt{2\theta _-k\log T}.
\end{align*}
On $G_{s,T}$ this implies
\[
U_s(r;{\theta})
=\overline Y_{s,k}+\sqrt{\frac{2{\theta}\log r}{k}}
>\mu_s+(1+2\epsilon)\delta
=\mu_o+2\epsilon\delta\ge U_o(r;{\theta}),
\]
contradicting selection of the optimal arm at time $r$. Combining the above, we deduce that\\ $\lim_{T\to\infty}\inf_{{\theta}\in \cK}{\PP} (N_s(T;{\theta})\ge m_-({\theta},T))=1$, and thus
\[
\liminf_{T\to\infty}\inf_{{\theta}\in \cK}\frac{\delta^2\EE_{\nu_\delta}[N_s(T;{\theta})]}{2{\theta}\log T}\geq 1-\rho.
\]
Letting $\rho\downarrow0$ proves (\ref{eq:forced-exploration-lower}).

We now prove a corresponding upper bound. Assume ${\theta}_-\ge(1+\xi)\tau_o^2$.  Fix $\rho'\in(0,1)$ and define $\epsilon':=\frac13(1-\frac1{\sqrt{1+\rho'}})$, $m_+({\theta},T):=
\lceil\frac{2{\theta}(1+\rho')\log T}{\delta^2}\rceil$. Define $B_t'({\theta})=\{U_o(t;{\theta})<\mu_o-\epsilon'\delta\}$. Using the same analysis as in (\ref{eq:Bt-upper}),
\[
{\PP} (B_t'({\theta} ))
\le \frac{2\tau_o^2}{\epsilon'^2\delta^2}t^{-{\theta} /\tau_o^2}
\le \frac{2\tau_o^2}{\epsilon'^2\delta^2}t^{-(1+\xi)}
\]
uniformly over ${\theta} \in \cK$, so the sum over $t$ is uniformly bounded. Now suppose arm $s$
is selected at time $t$ with pre-pull count $N_s(t-1;{\theta} )=k\ge m_+({\theta} ,T)$ and $B_t'({\theta} )$ does not occur, then $\overline Y_{s,k}-\mu_s\ge2\epsilon'\delta$. Similar to (\ref{eq:count-decomposition}) we have
\begin{align}
N_s(T;{\theta})\leq m_+({\theta}, T) + 1 + \sum_{t=3}^T\ind\{B_t'({\theta})\} + \sum_{k = m_+({\theta}, T)}^\infty\ind\{\bar Y_{s, k} - \mu_s\geq 2\epsilon'\delta\}.\label{eq::ucb-fixed-environment-part-iii-counting-decomposition}
\end{align}
Taking expectations and by controlling the Gaussian tail bound, we deduce that
$$
\EE[N_s(T;{\theta})]\leq 2 + \frac{2{\theta}(1+\rho')\log T}{\delta^2} + C,
$$
$\forall {\theta}\in\cK$. Here the constant $C$ depends only on $\xi$, $\rho'$, $\delta$, $\tau_o^2$ and $\tau_s^2$. Combining this upper bound with (\ref{eq:forced-exploration-lower}) and then letting $\rho'\downarrow 0$ proves (\ref{eq:supercritical-asymptotic}).

\textbf{Part (iv).} Fix $\delta>0$ and assume $\theta=\tau_o^2$. LHS$\ge$RHS in (\ref{eq:ucb-critical-regret}) follows from part (iii), applied to a compact interval containing $\tau_o^2$.

To prove LHS$\le$RHS in (\ref{eq:ucb-critical-regret}), fix $\rho'\in(0,1)$ and use the same definitions of $\epsilon'$, $m_+(\theta,T)$, and $B'_t(\theta)$ as in the proof of part (iii). For every $t\ge 3$, we have
\begin{align*}
\mathbb{P}_{\nu_\delta}\bigl(B'_t(\theta)\bigr)
&\le
\sum_{m=1}^{t-1}
\exp\left\{
-\frac{m}{2\tau_o^2}
\left(
\epsilon'\delta+\sqrt{\frac{2\theta\log t}{m}}
\right)^2
\right\}
\\
&=
t^{-1}
\sum_{m=1}^{t-1}
\exp\left\{
-\frac{m(\epsilon'\delta)^2}{2\tau_o^2}
-\frac{\sqrt{2}\epsilon'\delta}{\tau_o}
 \sqrt{m\log t}
\right\}
\\
&\le
C_1 t^{-1}\exp\{-c_2\sqrt{\log t}\},
\end{align*}
where $C_1:=\sum_{m=1}^{\infty}\exp\left\{-\frac{m(\epsilon'\delta)^2}{2\tau_o^2}\right\}<\infty$, $c_2:=\frac{\sqrt{2}\epsilon'\delta}{\tau_o}>0$. As $\int_e^\infty\frac{\exp\{-c_2\sqrt{\log x}\}}{x}\,dx=2\int_1^\infty u e^{-c_2u}\,du<\infty$, we can deduce that
\[
\sum_{t=3}^{\infty} \PP_{\nu_\delta}\bigl(B'_t(\theta)\bigr)<\infty.
\]

Now we use the counting decomposition (\ref{eq::ucb-fixed-environment-part-iii-counting-decomposition}) with $\theta=\tau_o^2$. Taking expectations and applying the same suboptimal-arm Gaussian deviation bound to the last term leads to
\begin{align*}
\mathbb{E}_{\nu_\delta}[N_s(T;\theta)]
&\le
m_+(\theta,T)+1
+\sum_{t=3}^{\infty}
 \mathbb{P}_{\nu_\delta}\bigl(B'_t(\theta)\bigr)
+\sum_{k=1}^{\infty}
 \exp\left\{-\frac{2k(\epsilon'\delta)^2}{\tau_s^2}\right\}
\\
&=
\frac{2\theta(1+\rho')\log T}{\delta^2}+O(1),
\end{align*}
where the $O(1)$ term may depend on
$\rho'$, $\delta$, $\tau_o^2$, and $\tau_s^2$, but not on $T$.
Dividing by $\log T$, letting $T\to\infty$, and then letting
$\rho'\downarrow0$ yields
\[
\limsup_{T\to\infty}
\frac{\mathbb{E}_{\nu_\delta}[N_s(T;\theta)]}{\log T}
\le
\frac{2\theta}{\delta^2}.
\]
Finally, we obtain part (iv) by combining the above analysis and using the fact that $\Reg_T(\theta;\nu_\delta)=\delta\mathbb{E}_{\nu_\delta}[N_s(T;\theta)]$.

\textbf{Part (v).} If ${\theta}<\tau_o^2$, choose $\varepsilon>0$ sufficiently small that ${\theta}+\varepsilon<\tau_o^2$. For large enough $T$, eventually ${\theta}_T\in[{\theta}-\varepsilon,{\theta}+\varepsilon]$. Part (ii), followed by $\eta\downarrow0$ and then $\varepsilon\downarrow0$, gives $\liminf_{T\to\infty} \frac{\log {\Reg}_T({\theta}_T;\nu_\delta)}{\log T} \ge 1-\frac{{\theta}}{\tau_o^2}$.
On the other hand, Part (i) and ${\theta}_T\to {\theta}$ gives $\limsup_{T\to\infty} \frac{\log {\Reg}_T({\theta}_T;\nu_\delta)}{\log T} \le 1-\frac{{\theta}}{\tau_o^2}$.

If ${\theta}\ge\tau_o^2$, part (i) gives $\limsup_{T\to\infty} \frac{\log {\Reg}_T({\theta}_T;\nu_\delta)}{\log T} \le 0$, while nonnegativity gives the reverse inequality. This proves (\ref{eq:fixed-regret-exponent}).

\end{proof}

\begin{lemma}[Limit of the Plug-In minimizer]
\label{lem:plugin-limit}
Let $\nu$ be a fixed two-armed Gaussian environment with distinct means and positive variances.  Let $\tau_o^2$ denote the variance of its optimal arm and assume $\tau_o^2\notin\{\underline \theta,\overline \theta\}$. For each $T$, let $\theta_T(\nu)\in\argmin_{\theta\in\Theta}\Reg_T(\theta;\nu)$. Then
\begin{equation}
\lim_{T\to\infty}\theta_T(\nu)=\Pi_{\Theta}(\tau_o^2).
\label{eq:generic-plugin-limit}
\end{equation}
\end{lemma}

\begin{proof}[Proof of Lemma \ref{lem:plugin-limit}]
For simplicity, write ${\theta}_T={\theta}_T(\nu)$ and $c=\Pi_{{\Theta}}(\tau_o^2)$. If $\tau_o^2<\underline {\theta}$, by Lemma \ref{lem:ucb-bounds}(iii), for every fixed $\varepsilon>0$,
\[
\inf_{{\theta}\in[\underline {\theta}+\varepsilon,\overline {\theta}]}
{\Reg}_T({\theta};\nu)
>
{\Reg}_T(\underline {\theta};\nu)
\]
for all sufficiently large $T$, because the leading term is proportional to
${\theta}\log T$.  Hence ${\theta}_T<\underline {\theta}+\varepsilon$ eventually. Thus ${\theta}_T\to\underline {\theta}=c$.

Suppose $\tau_o^2\in(\underline {\theta},\overline {\theta})$.  Fix a sufficiently small
$\varepsilon>0$. For ${\theta}\le\tau_o^2-\varepsilon$, Lemma~\ref{lem:ucb-bounds}(ii) gives a positive polynomial lower bound, whereas ${\theta}_0=\tau_o^2+\varepsilon/2$ has logarithmic regret by Lemma \ref{lem:ucb-bounds}(iii).  Thus for sufficiently large $T$, ${\theta}_T>\tau_o^2-\varepsilon$. On the other hand, Lemma \ref{lem:ucb-bounds}(iii) implies
\[
\inf_{{\theta}\in[\tau_o^2+\varepsilon,\overline {\theta}]}
{\Reg}_T({\theta};\nu)>{\Reg}_T({\theta}_0;\nu)
\]
for sufficiently large $T$. Thus,
${\theta}_T<\tau_o^2+\varepsilon$ for sufficiently large $T$. Letting $\varepsilon\downarrow0$ proves convergence to $c=\tau_o^2$.

Finally, suppose $\tau_o^2>\overline {\theta}$.  Fix $\varepsilon>0$.  By the uniform lower bound in Lemma~\ref{lem:ucb-bounds}(ii), every ${\theta}\le\overline {\theta}-\varepsilon$ has regret exponent at least $1-(\overline {\theta}-\varepsilon)/\tau_o^2$, whereas the upper envelope in Lemma~\ref{lem:ucb-bounds}(i) shows that $\overline {\theta}$ has exponent at most $1-\overline {\theta}/\tau_o^2$. The former exponent is strictly larger, so ${\theta}_T>\overline {\theta}-\varepsilon$ eventually.  Hence ${\theta}_T\to\overline {\theta}=c$ as $T\to\infty$.
\end{proof}

\begin{lemma}[Limit of the Gaussian-perturbed minimizer]\label{lem:integrated-exponent}
Fix any offline data $\cD_{\off}$ such that $\widehat\sigma_1^2,\widehat\sigma_2^2>0$. Define $\bar\sigma^2:=\widehat\sigma_1^2\vee \widehat\sigma_2^2$. If $\theta_T\in\Theta$ and $\theta_T\to \theta$, then
\begin{equation}
\lim_{T\to\infty}
\frac{\log\bigl(\Reg_{\UA, T}(\theta_T;\cD_{\off})\bigr)}{\log T}=
\left(1-\frac{\theta}{\bar\sigma^2}\right)_+.
\label{eq:integrated-exponent}
\end{equation}
As a result, every sequence of minimizers in (\ref{eq:uq-selector}) satisfies
\begin{equation}
\liminf_{T\to\infty}\widehat \theta_{\UA, T}
\ge\Pi_{\Theta}(\bar\sigma^2).
\label{eq:uq-liminf-lemma}
\end{equation}
\end{lemma}

\begin{proof}[Proof of Lemma \ref{lem:integrated-exponent}]
Let $X:=\left|\widetilde\mu_{1}(Z_1)-\widetilde\mu_{2}(Z_2)\right|$.
Conditional on $\cD_{\off}$, $X$ has a folded-normal distribution. Its density on $(0,\infty)$ is bounded, and $\EE[X]<\infty$.

For every perturbed environment, the variance of its optimal arm is at most
$\bar\sigma^2$. Applying Lemma \ref{lem:ucb-bounds} gives
\begin{equation}
\Reg_T(\theta_T;\widetilde M(Z))
\le C\left[ X+\frac{B_T}{X}\right],
\label{eq:integrated-pointwise-upper}
\end{equation}
where $B_T:=T^{q_T}(1+\log T)$, $q_T=(1-{\theta}_T/\bar\sigma^2)_+$; the constant $C$ depends on $\cD_{\off}$ and $\Theta$, but not on $T$ or $Z$. We also have the trivial bound $\Reg_T\le TX$.

Let $x_T=(B_T/T)^{1/2}$.  Since $\underline \theta>0$, $q_T$ is bounded strictly below one for all sufficiently large $T$, and hence $x_T\to0$. Denote $f_X$ as the density of $X$, then $\|f_X\|_\infty<\infty$, and
\begin{align*}
\EE[TX\ind\{X\le x_T\}]
&\le \frac12\|f_X\|_\infty T x_T^2
=O(B_T),\\
\EE\left[\frac{B_T}{X}\ind\{x_T<X\le1\}\right]
&\le \|f_X\|_\infty B_T\log(1/x_T)
=O(B_T\log T),\\
\EE\left[\frac{B_T}{X}\ind\{X>1\}\right]&=O(B_T).
\end{align*}
Combining these bounds with (\ref{eq:integrated-pointwise-upper}) yields
\begin{equation}
\Reg_{\UA, T}({\theta}_T; \cD_{\off})
\le C' T^{q_T}(1+\log T)^2.
\label{eq:integrated-upper-final}
\end{equation}
Since $q_T\to(1-{\theta}/{\bar\sigma^2})_+$, this proves that LHS$\le$RHS in (\ref{eq:integrated-exponent}).

Below we prove the reverse direction LHS$\ge$RHS. The case ${\theta}\ge {\bar\sigma^2}$ is immediate from nonnegativity. Now suppose ${\theta}<{\bar\sigma^2}$, and let $h$ be an arm with $\widehat\sigma_h^2={\bar\sigma^2}$. The signed perturbed gap $\widetilde\mu_{h}(Z_h)-\widetilde\mu_{3-h}(Z_{3-h})$ is a nondegenerate Gaussian random variable. Thus, there exist constants
$0<g_-<g_+<\infty$ such that the event
\begin{equation}
\mathcal W:=\left\{
g_-\le\widetilde\mu_{h}(Z_h)-\widetilde\mu_{3-h}(Z_{3-h})\le g_+
\right\}
\label{eq:witness-event}
\end{equation}
has conditional probability $w({\cD_{\off}})>0$ condition on ${\cD_{\off}}$. On this event, arm $h$ is optimal, its variance is ${\bar\sigma^2}$, and the gap lies in $[g_-,g_+]$.

Fix $\varepsilon,\eta>0$ small enough such that ${\theta}+\varepsilon<{\bar\sigma^2}$.  Eventually
${\theta}_T\le {\theta}+\varepsilon$. Lemma~\ref{lem:ucb-bounds}(ii), applied uniformly on
$\mathcal W$, gives
\[
\Reg_{\UA, T}({\theta}_T; {\cD_{\off}})\ge w({\cD_{\off}})c T^{1-({\theta}+\varepsilon)/{\bar\sigma^2}-\eta}
\]
for all sufficiently large $T$.  Letting $\varepsilon,\eta\downarrow0$ proves
the desired inequality.

Finally we prove (\ref{eq:uq-liminf-lemma}). Let $c=\Pi_{{\Theta}}({\bar\sigma^2})$. Suppose otherwise. Then there exists a subsequence of UA minimizers with a limit ${\theta}<c$. If ${\bar\sigma^2}\in(\underline {\theta},\overline {\theta})$, choose a comparator $\theta'\in({\bar\sigma^2},\overline {\theta}]$.  By (\ref{eq:integrated-exponent}), the minimizers along the subsequence have a strictly positive objective exponent, while the comparator has exponent zero, contradicting optimality. If ${\bar\sigma^2}\ge\overline {\theta}$, use the comparator $\theta'=\overline {\theta}$; then $1-\frac{{\theta}}{{\bar\sigma^2}}>1-\frac{\overline {\theta}}{{\bar\sigma^2}}$, again contradicting optimality. If ${\bar\sigma^2}\le\underline {\theta}$, then $c=\underline {\theta}$ and no ${\theta}\in{\Theta}$ can be smaller than $c$. Therefore, no limit lies below $c$.
\end{proof}

\section{Theoretical analysis of the Thompson Sampling algorithm}
\label{app:ts}

In this section, we provide a detailed theoretical analysis of the behavior and selection performance of the Thompson Sampling (TS) algorithm in a two-armed Gaussian bandit. More specifically, we consider the standard Gaussian TS algorithm with an improper flat prior on each arm mean \citep{daniel2018tutorial}: The agent keeps a working reward model for each arm $i\in\{1, 2\}$ as
$$
Y|\mu_i\sim \cN(\mu_i, \theta),
$$
with a flat prior $\pi_i(\mu_i)\propto 1$. At each round $t$, let $N_i(t-1;\theta)$ and $\widehat\mu_i(t-1;\theta)$ denote the number of pulls and the empirical mean reward for arm $i$ prior to time $t$ under the current algorithm. Then the posterior for $\mu_i$ equals
$$
\mu_i|\cH_{t-1}\sim \cN\left(\widehat\mu_i(t-1;\theta), \frac{\theta}{N_i(t-1;\theta)}\right).
$$
After each arm is pulled once, at each round $t$, for each arm $i\in\{1, 2\}$, the agent independently samples a posterior draw
$$
\Xi_{i}(t, \theta) = \widehat\mu_i(t-1;\theta)+\sqrt{\frac{\theta}{N_i(t-1;\theta)}}Z_{i,t},\qquad Z_{1,t},Z_{2,t}\stackrel{\mathrm{iid}}{\sim}\cN(0,1),
$$
and then selects
$$
A_t(\theta)\in\argmax_{i\in\{1,2\}}\Xi_i(t;\theta).
$$
See Alg.~\ref{alg:ts}. Below, we refer to this algorithm as TS$(\theta)$, where the tuning parameter $\theta$ represents the working reward variance.

\begin{algorithm}[t]
\caption{TS$(\theta)$}
\label{alg:ts}
\begin{algorithmic}[1]
\STATE Pull each arm once.
\FOR{$t=3,\ldots,T$}
    \STATE For $i\in\{1,2\}$, compute $N_i(t-1;\theta)$ and $\widehat\mu_i(t-1;\theta)$, the number of pulls and empirical mean reward for arm $i$ up to time $t-1$.
    \STATE Independently draw $Z_{i,t}\sim\cN(0,1)$ for each $i\in\{1,2\}$ and compute
    \[
    \Xi_i(t;\theta)
    =
    \widehat\mu_i(t-1;\theta)
    +
    \sqrt{\frac{\theta}{N_i(t-1;\theta)}}Z_{i,t}.
    \]
    \STATE Pull $A_t = A_t(\theta)\in\arg\max_{i\in\{1,2\}}\Xi_i(t;\theta)$.
\ENDFOR
\end{algorithmic}
\end{algorithm}

\noindent\textbf{Notation.} We use the notation introduced in Appendix~\ref{app:ucb} whenever it is algorithm-independent. In particular, $\Pi_V$, $\sigma(X)$, $\Phi$, $\phi$, $a_+$, and, for a generic two-armed Gaussian environment $\nu$, $o(\nu)$, $s(\nu)$, and $\delta(\nu)$ have the same meanings as in Appendix~\ref{app:ucb}. In addition, for $x\in\RR$, we define the Gaussian odds function $H(x):=\frac{\normcdf(-x)}{\normcdf(x)}$.

Throughout this section, however, $\pi_\theta$ denotes the policy induced by TS$(\theta)$ in Alg.~\ref{alg:ts}, and $N_i(T;\theta)$ denotes the number of times arm $i$ is selected over $T$ rounds under TS$(\theta)$ instead of UCB$(\theta)$. For a generic two-armed Gaussian environment $\nu$ with distinct arm means, we write
\[
\Reg_T(\theta;\nu):=\Reg_T(\pi_\theta,\nu)=
\delta(\nu)\EE_\nu\left[N_{s(\nu)}(T;\theta)\right],
\]
and when the two arm means are equal, we define $\Reg_T(\theta;\nu)=0$.

We define the Plug-In and Uncertainty-Aware (UA) selection rules, $\widehat\theta_{\plug,T}$ and $\widehat\theta_{\UA,T}$, as in (\ref{eq:plugin-selector})-(\ref{eq:uq-selector}), with $\Reg_T(\theta;\nu)$ now referring to the regret of TS$(\theta)$. The corresponding quantities $\Reg_{\UA,T}(\theta;\cD_{\off})$, $r_{\plug,T}(\cD_{\off})$, $r_{\UA,T}(\cD_{\off})$, and $G(\cD_{\off})$ are defined analogously. Also, as in Appendix~\ref{app:ucb}, we write
\[
\widetilde{\mu}_i(Z_i):=\widehat{\mu}_i+\sqrt{\frac{\widehat{\sigma}_i^2}{T_{1,\off}}}Z_i,\qquad i\in\{1,2\},
\]
for the perturbed arm means that define $\widetilde M(Z)$.

\subsection{Main results and proofs}

In this section, we first establish the existence of $\widehat\theta_{\plug,T}$ and $\widehat\theta_{\UA,T}$ in the TS setting. We then derive the TS analogues of our main UCB results: a characterization of its asymmetric regret landscape, paralleling Theorem~\ref{thm::UCB-regret}, followed by an analysis of the $\theta$ selected by the Plug-In and UA approaches, paralleling Theorem~\ref{thm:main}.

Without loss of generality, throughout this section, we assume that any TS$(\theta)$ algorithm is deployed by placing the two independent reward sequences and all index randomizations on a common probability space, and the reward $Y_{i,t}$ is revealed when arm $i$ is selected for the $t$th time. If TS$(\theta)$ is deployed, let the arm means and pull counts be
\[
N_i(t;\theta):=\sum_{r=1}^t \ind\{A_r(\theta)=i\},
\qquad
\overline Y_{i,n}:=\frac{1}{n}\sum_{s=1}^n Y_{i,s},
\qquad
\widehat\mu_i(t; \theta):=\overline Y_{i,N_i(t; \theta)}.
\]

\subsubsection{Existence of $\widehat \theta_{\plug,T}$ and $\widehat \theta_{\UA,T}$}

\begin{proposition}
\label{prop:existence-TS}
Consider the Gaussian environment $M^*$ defined in Section~\ref{sec:mab}, with $\Reg_T(\theta;\cdot)$ corresponding to TS$(\theta)$ as defined above, and assume $T_{\off}\ge 4$. Then for almost every dataset $\cD_{\off}$, the following holds for any finite horizon $T$:
\begin{itemize}
\item[(i)] $\theta\mapsto \Reg_T(\theta;\widehat M)$ is continuous on $\Theta$;
\item[(ii)] $\theta\mapsto\Reg_{\UA, T}(\theta; \cD_{\off})$ is continuous on $\Theta$.
\end{itemize}
Consequently, both argmin sets in (\ref{eq:plugin-selector}) and (\ref{eq:uq-selector}) are nonempty and compact.
\end{proposition}
\begin{proof}[Proof of Proposition \ref{prop:existence-TS}]
Fix any $\cD_{\off}$ such that $\widehat\sigma_1^2,\widehat\sigma_2^2>0$ and $\widehat\mu_1\neq\widehat\mu_2$. We will show (i) and (ii) hold.

We first prove a generic finite-horizon continuity statement.  Fix a Gaussian environment $\nu$ with positive arm variances and distinct means.  Put two independent infinite armwise reward sequences and all standard normal index randomizations $\{Z_{i,t}:i\in\{1,2\},t\ge3\}$ on one probability space, and use these same random variables to run TS$({\theta})$ for all ${\theta}\in{\Theta}$.

Fix ${\theta}_0\in{\Theta}$.  We show by induction on $t$ that, almost surely, the action sequence through time $t$ is constant in ${\theta}$ on some neighborhood of ${\theta}_0$.
The initialization pulls are independent of ${\theta}$.  Suppose the claim holds through time $t-1$.  On the corresponding probability-one event, throughout a neighborhood of ${\theta}_0$ the same armwise rewards have been revealed, so $N_i(t-1;{\theta})$ and $\widehat\mu_i(t-1;{\theta})$ are constant in ${\theta}$. On that neighborhood, write
\begin{align*}
D_t({\theta})
&:=\Xi_1(t;{\theta})-\Xi_2(t;{\theta})\\
&=\widehat\mu_1(t-1;{\theta})-\widehat\mu_2(t-1;{\theta})+\sqrt {\theta}\left
\{
\frac{Z_{1,t}}{\sqrt{N_1(t-1;{\theta})}}-
\frac{Z_{2,t}}{\sqrt{N_2(t-1;{\theta})}}
\right\}.
\end{align*}
For each realized past, ${\theta}\mapsto D_t({\theta})$ is continuous.  Moreover, conditional on the history through time $t-1$, $D_t({\theta}_0)$ is a nondegenerate Gaussian random variable with conditional variance
\[
{\theta}_0\left[N_1(t-1;{\theta}_0)^{-1}+N_2(t-1;{\theta_0})^{-1}\right]>0.
\]
Therefore $\PP(D_t({\theta}_0)=0)=0$. Thus, almost surely, the sign of $D_t({\theta})$ is unchanged on a sufficiently small neighborhood of ${\theta}_0$, proving the induction step.

It follows that, for any fixed $T$,
\[
N_{s(\nu)}(T;{\theta})\to N_{s(\nu)}(T;{\theta}_0)
\qquad\text{almost surely as }{\theta}\to {\theta}_0.
\]
The count is bounded by $T$, so dominated convergence yields continuity of ${\theta}\mapsto \Reg_T({\theta};\nu)$.

Applying this result to $\nu=\widehat M$ proves part (i).  For part (ii), conditional on $\cD_{\off}$, the perturbed signed gap is a nondegenerate Gaussian random variable and is nonzero for almost every $Z$.  Hence the preceding continuity statement applies for almost every $Z$. Also,
\[
0\le \Reg_T({\theta};\tilde M(Z))
\le
T\left|\tilde\mu_{1}(Z_1)-\tilde\mu_{2}(Z_2)\right|,
\]
and the right-hand side is integrable in $Z$.  Dominated convergence proves continuity of the corresponding objective.

\end{proof}

\subsubsection{Asymmetric regret of Thompson Sampling}

\begin{theorem}[Asymmetric regret of TS]\label{thm::TS-regret}
Consider the two-armed Gaussian bandit environment $M^*=\bigl(\cN(\mu_1,\sigma^2),\cN(\mu_2,\sigma^2)
\bigr)$, where $\mu_1>\mu_2$, $\sigma^2>0$. For $\theta>0$, let $\pi_\theta$ denote the policy induced by TS$(\theta)$ in Alg.~\ref{alg:ts}. Then the following hold for its expected regret:
\begin{itemize}
\item[(i)] For any $\theta\in(0,\sigma^2)$,
\[
\lim_{T\to\infty}
\frac{\log\Reg_T(\pi_\theta,M^*)}{\log T}
=
1-\frac{\theta}{\sigma^2}.
\]

\item[(ii)] For any $\theta\ge\sigma^2$,
\[
\lim_{T\to\infty}
\frac{\Reg_T(\pi_\theta,M^*)}{\log T}
=
\frac{2\theta}{\mu_1-\mu_2}.
\]
\end{itemize}
Thus, $\theta=\sigma^2$ separates two asymptotic regimes:
the expected regret grows polynomially in $T$ for
$\theta<\sigma^2$ and logarithmically in $T$ for
$\theta\ge\sigma^2$.
\end{theorem}

\begin{proof}[Proof of Theorem \ref{thm::TS-regret}]
Apply Lemma~\ref{lem:fixed-env-TS} to the true environment $M^*$ with $o=1$, $s=2$, $\delta=\mu_1-\mu_2$, and $\tau_o^2=\tau_s^2=\sigma^2$. For $0<\theta<\sigma^2$, part (i) follows from Lemma~\ref{lem:fixed-env-TS}(v), applied to the constant sequence $\theta_T\equiv\theta$.

For $\theta>\sigma^2$, choose a compact interval $\cK$ containing $\theta$ with $\inf \cK>\sigma^2$. Lemma~\ref{lem:fixed-env-TS}(iii), together with the fact $\Reg_T(\pi_\theta,M^*)=\Reg_T(\theta;M^*)=\delta\,\E_{M^*}[N_2(T;\theta)]$, gives the desired limit. The boundary case $\theta=\sigma^2$ follows directly from Lemma~\ref{lem:fixed-env-TS}(iv).
\end{proof}

\subsubsection{Plug-In and UA selection for Thompson Sampling}

We now characterize the asymptotic behavior of the Plug-In and UA selection rules for Thompson Sampling. The following result is the TS analogue of Theorem~\ref{thm:main} for UCB and provides the full technical characterization of the selected tuning parameters and their resulting deployment regrets.

\begin{theorem}\label{thm:main-TS}
For any $T_{\off}\ge4$, on the probability-one event $\cR_{T_{\off}}$ defined by (\ref{eq:regular-event}), the following statements hold for the Plug-In and UA selection of the Thompson Sampling algorithm:

\begin{enumerate}[label=\textup{(\roman*)}]

\item As $T\to\infty$,
\begin{equation}
\widehat \theta_{\plug, T}\to c_{\plug}(\cD_{\off}):=\Pi_{{\Theta}}(\widehat\sigma_{\hat a}^2),
\qquad
\liminf_{T\to\infty}\widehat \theta_{\UA, T}
\ge c_{\UA}(\cD_{\off}):=\Pi_{\Theta}(\widehat\sigma_1^2\vee \widehat\sigma_2^2).
\label{eq:selector-limits-TS}
\end{equation}

\item The Plug-In conditional regret satisfies
\begin{equation}
\lim_{T\to\infty}
\frac{\log\bigl(r_{\plug, T}(\cD_{\off})\bigr)}{\log T}
=
\left(1-\frac{c_{\plug}(D^\off)}{\sigma^2}\right)_+,
\label{eq:PI-exponent-TS}
\end{equation}
and the Uncertainty-Aware conditional regret satisfies
\begin{equation}
\limsup_{T\to\infty}
\frac{\log\bigl(r_{\UA, T}(\cD_{\off})\bigr)}{\log T}
\le
\left(1-\frac{c_{\UA}(\cD_{\off})}{\sigma^2}\right)_+
\le
\left(1-\frac{c_{\plug}(\cD_{\off})}{\sigma^2}\right)_+.
\label{eq:UQ-exponent-TS}
\end{equation}
\item $G(\cD_{\off})\geq 1$.
\end{enumerate}
In addition, recall the event $\mathcal A_{T_\off}
=\left\{\widehat\sigma_{\widehat a}^2<\sigma^2<\widehat\sigma_{3-\widehat a}^2\right\}\cap \cR_{T_\off}$. Then  $\lim_{T_\off\to\infty}\PP(\mathcal A_{T_\off})=\frac14$, and on $\mathcal A_{T_\off}$, $G(\cD_{\off})=+\infty$.
\end{theorem}

\begin{proof}[Proof of Theorem \ref{thm:main-TS}]

For (i), first apply Lemma~\ref{lem:PI-minimizer-TS} (ii) to the fitted environment $\widehat M$, and we deduce that
\[
\widehat {\theta}_{\plug, T}({\cD_{\off}})\to\Pi_{\Theta}(\widehat\sigma_{\hat a}^2)=c_{\plug}({\cD_{\off}}).
\]
Also, Lemma~\ref{lem:UQ-exponent-TS} implies
\[
\liminf_{T\to\infty}\widehat {\theta}_{\UA, T}({\cD_{\off}})\ge\Pi_{\Theta}({\bar\sigma^2})=c_{\UA}({\cD_{\off}}).
\]
Since ${\bar\sigma^2}\ge \widehat\sigma_{\hat a}^2$ and projection onto an interval is monotone, $c_{\UA}({\cD_{\off}})\ge c_{\plug}({\cD_{\off}})$.

For (\ref{eq:PI-exponent-TS}) in part (ii), deploy the selected sequence $\widehat {\theta}_{\plug, T}$ to the true environment $M^*$. its optimal-arm variance is $\sigma^2$, and the sequence converges to $c_{\plug}({\cD_{\off}})$. Lemma~\ref{lem:fixed-env-TS}(v) gives (\ref{eq:PI-exponent-TS}).

For (\ref{eq:UQ-exponent-TS}), fix $\varepsilon>0$. By part (i), for all sufficiently large $T$, $\widehat {\theta}_{\UA, T}\ge c_{\UA}({\cD_{\off}})-\varepsilon$. Apply Lemma~\ref{lem:fixed-env-TS}(i) uniformly over all ${\theta}\in[(c_{\UA}({\cD_{\off}})-\varepsilon)\vee\underline {\theta},\overline {\theta}]$, we deduce that
\[
\limsup_{T\to\infty}
\frac{\log(r_{\UA, T}({\cD_{\off}}))}{\log T}
\le
\left(1-\frac{c_{\UA}({\cD_{\off}})-\varepsilon}{\sigma^2}\right)_+.
\]
Letting $\varepsilon\downarrow0$, and since $c_{\UA}({\cD_{\off}})\ge c_{\plug}({\cD_{\off}})$, (\ref{eq:UQ-exponent-TS}) follows.

For part (iii), first apply Lemma~\ref{lem:fixed-env-TS}(iii) in the true environment uniformly over ${\theta}\in\Theta$, gives a constant $c({\cD_{\off}})>0$ such that, for all large $T$,
\begin{equation}
r_{\plug, T}({\cD_{\off}})\wedge r_{\UA, T}({\cD_{\off}}) \ge c({\cD_{\off}})\log T.
\label{eq:both-log-lower}
\end{equation}
If $c_{\plug}({\cD_{\off}})>\sigma^2$, then $c_{\UA}({\cD_{\off}})\ge c_{\plug}({\cD_{\off}})>\sigma^2$. Part (i) implies that both selected sequences have a fixed positive distance above $\sigma^2$ for all large $T$. Lemma~\ref{lem:fixed-env-TS}(iii) and (\ref{eq:both-log-lower}) yields
\[
r_{\plug, T}({\cD_{\off}})=\Theta(\log T),
\qquad
r_{\UA, T}({\cD_{\off}})=\Theta(\log T).
\]
In this case, $\cG({\cD_{\off}})=1$.

If $c_{\plug}({\cD_{\off}})<\sigma^2$, set $q=1-c_{\plug}({\cD_{\off}})/\sigma^2>0$. By (\ref{eq:PI-exponent-TS}), $\log(r_{\plug, T}({\cD_{\off}}))=q\log T+o(\log T)$. By (\ref{eq:UQ-exponent-TS}), $\limsup_{T\to\infty}\frac{\log(r_{\UA, T}({\cD_{\off}}))}{\log T}\le q$. Together with (\ref{eq:both-log-lower}), this implies $\liminf_{T\to\infty}\frac{\log(r_{\plug, T}({\cD_{\off}}))}{\log(r_{\UA, T}({\cD_{\off}}))}\ge1$. The equality case $c_{\plug}({\cD_{\off}})=\sigma^2$ is excluded on $\cR_{T_\off}$. This proves part (iii).

Finally, we analyze the event $\mathcal A_{T_\off}$. Note that for Gaussian data, the sample means are independent of the sample variances. Thus, $\hat a$ is independent of $(\widehat\sigma_{1}^2,\widehat\sigma_{2}^2)$. In addition, the two sample variances are i.i.d. and $\frac{(T_{1, \off}-1)\widehat\sigma_{i}^2}{\sigma^2}\sim\chi^2_{T_{1, \off}-1}$. Therefore, condition on either value of $\hat a$, the probability that the fitted optimal arm has variance below $\sigma^2$ and the other arm has variance above $\sigma^2$ is $q_{T_{\off}}(1-q_{T_{\off}})$, where $q_{T_{\off}}:={\PP}(\chi^2_{T_{1, \off}-1}<T_{1, \off}-1)$. This leads to ${\PP}(\cA_{T_{\off}})=q_{T_{\off}}(1-q_{T_{\off}})$. Moreover, according to the central limit theorem, $\lim_{T_{ \off}\to\infty}q_{T_{ \off}}\to1/2$. Thus ${\PP}(\cA_{T_{\off}})\to1/4$.

On $\cA_{T_{\off}}$, the fitted optimal arm has sample variance below $\sigma^2$, and the other arm has sample variance above $\sigma^2$. Since $\underline {\theta}<\sigma^2<\overline {\theta}$,
\[
c_{\plug}({\cD_{\off}})<\sigma^2<c_{\UA}({\cD_{\off}}).
\]
From the previous analysis we obtain that on $\mathcal A_{T_{\off}}$,
\begin{align*}
r_{\plug, T}({\cD_{\off}})&=T^{1-c_{\plug}({\cD_{\off}})/\sigma^2+o(1)}, \\
r_{\UA, T}({\cD_{\off}})&=O(\log T),
\end{align*}
and hence $\cG({\cD_{\off}}) = \infty$.

\end{proof}

\subsection{Auxiliary lemmas}

\begin{lemma}[Fixed-environment TS regret]
\label{lem:fixed-env-TS}
Consider a 2-arm bandit environment $\nu_\delta
=\bigl(\cN(\mu_o,\tau_o^2),\cN(\mu_s,\tau_s^2)\bigr)$, where $\mu_o=\mu_s+\delta$ with $\delta>0$. Let  $\mathcal K=[\theta_-,\theta_+]\subset(0,\infty)$. Then the following holds for the regret of TS$(\theta)$ and the arm pull counts in environment $\nu_\delta$:
\begin{enumerate}[label=(\roman*)]
\item For any constant $\eta>0$, there is a finite constant $C_\eta$, depending
only on $\mathcal K,\tau_o^2,\tau_s^2,$ and $\eta$, such that for all $T\ge3$, $\theta\in \mathcal K$, and $\delta>0$,
\begin{equation}
\Reg_T(\theta;\nu_\delta) \le C_\eta\left[\delta+\frac{\log T}{\delta}+\frac1\delta T^{(1-\theta/\tau_o^2)_++\eta}\right].
\label{eq:TS-upper-envelope}
\end{equation}

\item Fix $0<\delta_-\le\delta_+<\infty$ and suppose $\theta_+<\tau_o^2$.
For every $\eta>0$, there exist constants $c_\eta>0$ and $T_\eta<\infty$ such that for all $T\ge T_\eta$,
\begin{equation}
\inf_{\substack{\delta\in[\delta_-,\delta_+]\\\theta\in \mathcal K}}
\Reg_T(\theta;\nu_\delta)\ge c_\eta T^{1-\theta_+/\tau_o^2-\eta}.
\label{eq:TS-subcritical-lower}
\end{equation}

\item For every fixed $\delta>0$,
\begin{equation}
\liminf_{T\to\infty}
\inf_{\theta\in \mathcal K}
\frac{\delta^2\E_{\nu_\delta}[N_s(T;\theta)]}{2\theta\log T}\ge1.
\label{eq:TS-log-lower}
\end{equation}
If, in addition, $\theta_-\ge(1+\xi)\tau_o^2$ for some $\xi>0$, then
\begin{equation}
\lim_{T\to\infty}
\sup_{\theta\in \mathcal K}
\left|
\frac{\delta^2\E_{\nu_\delta}[N_s(T;\theta)]}{2\theta\log T}-1\right|=0.
\label{eq:TS-supercritical-sharp}
\end{equation}

\item For every fixed $\delta>0$, when $\theta=\tau_o^2$,
\begin{equation}
\lim_{T\to\infty}
\frac{\Reg_T(\theta;\nu_\delta)}{\log T}
=
\frac{2\theta}{\delta}.
\label{eq:ts-critical-regret}
\end{equation}

\item If $\theta_T\in \mathcal K$ and $\theta_T\to \theta$, then
\begin{equation}
\lim_{T\to\infty}
\frac{\log\bigl(\Reg_T(\theta_T;\nu_\delta)\bigr)}{\log T}=\left(1-\frac{\theta}{\tau_o^2}\right)_+.
\label{eq:TS-fixed-exponent}
\end{equation}
\end{enumerate}
\end{lemma}

\begin{proof}[Proof of Lemma \ref{lem:fixed-env-TS}]
\noindent\textbf{Part (i).} Fix ${\theta}\in \mathcal K$ and let $L_T({\theta},\delta):=\lceil\frac{16{\theta}\log T}{\delta^2}\rceil$. Similar to the UCB analysis, we have
\begin{equation}
N_s(T;{\theta})
\le
1+L_T({\theta},\delta) + S_{1, T} + S_{2, T} + S_{3, T},\label{eq:TS-upper-decomp}
\end{equation}
where
\begin{align*}
S_{1, T}& = \sum_{t=3}^T
\ind\left\{
A_t(\theta)=s,\,
N_s(t-1;{\theta})\ge L_T({\theta},\delta),\,
\overline Y_{s,N_s(t-1;\theta)}>\mu_s+\delta/4
\right\}\\
S_{2, T}& = \sum_{t=3}^T
\ind\left\{
A_t(\theta)=s,\,
N_s(t-1;{\theta})\ge L_T({\theta},\delta),\,
\overline Y_{s,N_s(t-1;\theta)}\le\mu_s+\delta/4,\,
{\Xi}_s(t;{\theta})>\mu_o - \delta/4
\right\}.\\
S_{3, T}& = \sum_{t=3}^T\ind\left\{A_t(\theta)=s,\,{\Xi}_s(t;{\theta})\le \mu_o - \delta/4 \right\}.
\end{align*}

For $S_{1, T}$, note that whenever arm $s$ is selected, its pre-pull count increases by one.  Hence each value $k=N_s(t-1;{\theta})$ can occur at most once among rounds at which $A_t(\theta)=s$. Therefore,
\[
\sum_{t=3}^T
\ind\left\{
A_t(\theta)=s,\,
N_s(t-1;{\theta})\ge L_T({\theta},\delta),\,
\overline Y_{s,N_s(t-1;\theta)}>\mu_s+\delta/4
\right\}
\le
\sum_{k=1}^{\infty}
\ind\left\{
\overline Y_{s,k}>\mu_s+\delta/4
\right\}.
\]
Since $\overline Y_{s,k}-\mu_s\sim
\cN\left(0,\frac{\tau_s^2}{k}\right)$, Gaussian concentration gives ${\PP}(\overline Y_{s,k}>\mu_s+\delta/4)
\le\exp\left\{-\frac{k\delta^2}{32\tau_s^2}\right\}$. Thus we have
\begin{equation}
\E\left[S_{1, T}\right]\le
1+\frac{32\tau_s^2}{\delta^2}.
\end{equation}

We next consider $S_{2, T}$. If $N_s(t-1;{\theta}) = k$, conditional on $\mathcal H_{t-1}$, the history up to time $t-1$, ${\Xi}_s(t;{\theta})\sim \mathcal N\left(\overline Y_{s,k},\frac{{\theta}}{k}\right)$. Hence on the event $N_s(t-1;{\theta}) = k$, $\bar Y_{s, k}\leq \mu_s + \delta/4$,
\begin{align*}
{\PP}\left(
{\Xi}_s(t;{\theta})>\mu_o - \delta/4|\mathcal H_{t-1}
\right)
&=
\Phi\left(
-\frac{\mu_o - \delta/4-\overline Y_{s,k}}{\sqrt{{\theta}/k}}
\right)\\
&\le
\exp\left\{
-\frac{k(\mu_o - \delta/4-\overline Y_{s,k})^2}{2{\theta}}
\right\}\\
&\le
\exp\left\{
-\frac{k\delta^2}{8{\theta}}
\right\}.
\end{align*}
Whenever this term is relevant, $k=N_s(t-1;{\theta})\ge L_T({\theta},\delta)$.
By the definition of $L_T({\theta},\delta)$, $\exp\left\{
-\frac{k\delta^2}{8{\theta}}
\right\}\le T^{-2}$. Taking expectations and then summing over $t=3,\ldots,T$, we obtain that
\begin{equation}
\E S_{2, T}\leq T^{-1}.
\end{equation}

Finally, for $S_{3,T}$, for $k\ge1$, define
\[
p_k({\theta}):=
\Phi\left(
(\overline Y_{o,k}-\mu_o+\delta/4)\sqrt{\frac{k}{{\theta}}}
\right).
\]
Using Lemma~\ref{lem:odds-decomp} leads to
\[
\EE[S_{3,T}]\le\sum_{k=1}^{T-1}
\EE\left[
\min\left\{T,\frac{1-p_k({\theta})}{p_k({\theta})}\right\}
\right].
\]
Lemma~\ref{lem:Gaussian-odds}(ii) further implies that, for every $\eta>0$, there exists a constant $C_\eta$ such that
\begin{equation}
\E[S_{3,T}]\le
C_\eta {\theta}_+\delta^{-2}
T^{(1-{\theta}/\tau_o^2)_++\eta}.
\end{equation}

Substituting the above analysis into
(\ref{eq:TS-upper-decomp}), we have
$$
\E[N_s(T;{\theta})]
\le
1+L_T({\theta},\delta)+C(1+\delta^{-2})+T^{-1}+
C_\eta {\theta}_+\delta^{-2} T^{(1-{\theta}/\tau_o^2)_++\eta}.
$$
We deduce (\ref{eq:TS-upper-envelope}) by multiply both sides by $\delta$, plug in the definition of $L_T({\theta},\delta)$, and noticing that $\mathcal K$ is compact.

\noindent\textbf{Part (ii).} Fix $G>0$ large enough such that $\sum_{k=1}^{\infty}\exp\left\{-\frac{kG^2}{8\tau_s^2}\right\}
\le\frac12$. Recall that $Y_{i,k}$ denotes the reward observed on the $k$th pull of arm $i$, whereas $Z_{i,t}$ denotes the Gaussian index randomization used
for arm $i$ at calendar round $t$.

Define
\begin{align*}
F_Y
&:=
\bigcap_{k\ge1}
\left\{\overline Y_{s,k}\ge\mu_s-G/2\right\},\\
F_Z
&:=
\bigcap_{k\ge1}
\left\{Z_{s,k+2}\ge-\frac{G\sqrt{k}}{2\sqrt{{\theta}_+}}\right\},\\
H_T
&:=
\bigcap_{t=3}^{T}
\left\{Z_{o,t}\le\sqrt{2\log T}\right\},\\
E_T
&:=
\left\{
Y_{o,1}
\le
\mu_s-G-\sqrt{2{\theta}_+\log T}
\right\}.
\end{align*}
A Gaussian union bound gives ${\PP}(F_Y)\ge1/2$.  Moreover, we have ${\PP}(F_Z)=\prod_{k=1}^{\infty}\normcdf\left(\frac{G\sqrt{k}}{2\sqrt{{\theta}_+}}\right)=:c_Z>0$. From a union bound we deduce that $\lim_{T\to\infty}{\PP}(H_T)=1$. We also have ${\PP}(E_T)=T^{-{\theta}_+/\tau_o^2-o(1)}$ uniformly over $\delta\in[\delta_-,\delta_+]$: Indeed,
define $x_T(\delta):=\frac{\delta+G+\sqrt{2{\theta}_+\log T}}{\tau_o}$. Then uniformly over $\delta\in[\delta_-,\delta_+]$, $\frac{x_T(\delta)^2}{2}=\frac{{\theta}_+}{\tau_o^2}\log T+O(\sqrt{\log T})$. Thus the desired bound holds from the Gaussian Mills lower bound \citep{gordon1941values}, which also implies that for every $\eta>0$ and all sufficiently large $T$, ${\PP}(E_T)\ge T^{-{\theta}_+/\tau_o^2-\eta}$.

Now on $E_T\cap F_Y\cap F_Z\cap H_T$, we show by induction that arm $s$ is selected at every round $t=3,\ldots,T$, simultaneously for all ${\theta}\in \mathcal K$. Specifically, suppose that arm $s$ has been selected at all preceding post-initialization rounds $3,\ldots,t-1$. Then, the pre-pull count of arm $s$ at round $t$ is $N_s(t-1;{\theta})=k = t-2$, while the optimal arm has not been selected since initialization so $N_o(t-1;{\theta})=1$. Therefore, the posterior samples generated at round $t$ are
\[
{\Xi}_s(t;{\theta})
=
\overline Y_{s,k}
+
\sqrt{\frac{{\theta}}{k}}\,Z_{s,k+2},
\qquad
{\Xi}_o(t;{\theta})
=
Y_{o,1}
+
\sqrt {\theta}\,Z_{o,t}.
\]
Thus, on $F_Y\cap F_Z$,
\begin{align*}
{\Xi}_s(t;{\theta})\ge\mu_s-\frac G2-
\sqrt{\frac {\theta}k}\frac{G\sqrt{k}}{2\sqrt{{\theta}_+}}
=
\mu_s-\frac G2-\frac G2\sqrt{\frac{{\theta}}{{\theta}_+}}\ge\mu_s-G.
\end{align*}
At the same time, on $E_T\cap H_T$,
\begin{align*}
{\Xi}_o(t;{\theta})\le\mu_s-G-\sqrt{2{\theta}_+\log T}+\sqrt {\theta}\sqrt{2\log T}\le\mu_s-G.
\end{align*}
Hence arm $s$ is selected at round $t$, completing the induction.

Finally, note that the four events are independent: $E_T$ depends only on the optimal-arm reward sequence, $F_Y$ only on the suboptimal-arm reward sequence, $F_Z$ only on the suboptimal-arm index randomizations, and $H_T$ only on the optimal-arm index randomizations.  Therefore, from the above probability calculations,
\[
\Reg_T({\theta};\nu_\delta)
\ge
\delta_-\cdot (T-1)\cdot
{\PP}(E_T){\PP}(F_Y){\PP}(F_Z){\PP}(H_T),
\]
uniformly over ${\theta}\in \mathcal K$ and $\delta\in[\delta_-,\delta_+]$, and thus (\ref{eq:TS-subcritical-lower}) holds.

\noindent\textbf{Part (iii).} We first establish the lower bound (\ref{eq:TS-log-lower}). We begin with a uniform finite-exploration fact.  For every fixed integer
$k_0\ge2$,
\begin{equation}
\lim_{t\to\infty}
\sup_{{\theta}\in \mathcal K}
{\PP}\left(\min\{N_o(t;{\theta}),N_s(t;{\theta})\}<k_0\right)=0.
\label{eq:finite-exploration}
\end{equation}
To prove this, let
\[
\mathcal B_B
:=
\left\{
\sup_{i\in\{o,s\}}\sup_{k\ge1}|\overline Y_{i,k}|\le B
\right\}.
\]
Since $\overline Y_{i,k}\to\mu_i$ almost surely, ${\PP}(\mathcal B_B)\uparrow1$ as
$B\to\infty$.  On $\mathcal B_B$, suppose arm $i$ has been pulled $r<k_0$ times at the beginning of round $t\geq 3$, and let $j$ be the other arm. Uniformly over ${\theta}\in \cK$, conditional independence of the current posterior
samples gives
\begin{align*}
{\PP}(A_t=i\mid\mathcal H_{t-1})
&\ge
{\PP}({\Xi}_i(t;{\theta})>B+1,\ {\Xi}_j(t;{\theta})\le B+1|\mathcal H_{t-1})\\
&\ge
\normcdf\!\left(-(2B+1)\sqrt{\frac{k_0-1}{{\theta}_-}}\right)
\normcdf\!\left(\frac1{\sqrt{{\theta}_+}}\right)
=:p_B>0.
\end{align*}
This implies that as long as arm $i$ has fewer than $k_0$ pulls, the waiting time to its next pull is stochastically dominated by a geometric random variable with success probability $p_B$. Thus, for each arm $i\in\{o,s\}$ and uniformly over ${\theta}\in \mathcal K$,
\[
{\PP}(N_i(t;{\theta})<k_0,\ \mathcal B_B)
\le
{\PP}\left(\sum_{\ell=1}^{k_0-1}G_\ell>t-2\right),
\]
where $G_1,\ldots,G_{k_0-1}$ are i.i.d. geometric$(p_B)$ random
variables.  Therefore, for every ${\theta}\in \cK$,
\begin{align*}
&{\PP}\left(\min\{N_o(t;{\theta}),N_s(t;{\theta})\}<k_0\right)
\\
\quad\le&
{\PP}(\mathcal B_B^c)+
{\PP}\left(\left\{\min\{N_o(t;{\theta}),N_s(t;{\theta})\}<k_0\right\}\cap\mathcal B_B
\right)
\\
\quad\le&
{\PP}(\mathcal B_B^c)
+
\sum_{i\in\{o,s\}}
{\PP}(N_i(t;{\theta})<k_0,\ \mathcal B_B)
\\
\quad\le&
{\PP}(\mathcal B_B^c)
+
2{\PP}\left(
\sum_{\ell=1}^{k_0-1}G_\ell>t-2
\right).
\end{align*}
The right-hand side above does not depend on ${\theta}$, and hence
\[
\sup_{{\theta}\in \mathcal K}
{\PP}\left(\min\{N_o(t;{\theta}),N_s(t;{\theta})\}<k_0\right)
\le
{\PP}(\mathcal B_B^c)
+
2{\PP}\left(
\sum_{\ell=1}^{k_0-1}G_\ell>t-2
\right).
\]
For each fixed $B$, $p_B>0$ and $k_0$ is fixed, so
$\sum_{\ell=1}^{k_0-1}G_\ell<\infty$ almost surely.  Therefore, ${\PP}(\sum_{\ell=1}^{k_0-1}G_\ell>t-2)\to0$ as $t\to\infty$. Therefore, for any fixed $B$,
\[
\limsup_{t\to\infty}
\sup_{{\theta}\in \cK}
{\PP}\left(\min\{N_o(t;{\theta}),N_s(t;{\theta})\}<k_0\right)\le
{\PP}(\mathcal B_B^c).
\]
Letting $B\to\infty$ proves (\ref{eq:finite-exploration}).

We will also use the following elementary proposition on Bernoulli random variables: Let $(\cF_j)_{j=0}^L$ be a filtration and let $I_j\in\{0,1\}$ be $\cF_j$-measurable for $j=1,\ldots,L$. Suppose that there exists a constant $p\in(0,1]$ and integer $m\ge1$, such that $\forall j=1,\ldots,L$,
\[
{\PP}(I_j=1\mid\cF_{j-1})\ge p
\qquad\text{a.s. on }
\left\{
\sum_{\ell=1}^{j-1}I_\ell<m
\right\},
\]
Then
\begin{equation}
{\PP}\left(
\sum_{j=1}^L I_j<m
\right)
\le
{\PP}(\operatorname{Binomial}(L,p)<m).
\label{eq:adaptive-binomial}
\end{equation}
This property follows by coupling the $I_j$'s with i.i.d. Bernoulli$(p)$ variables $B_j$ (e.g. using independent uniform random variables), so that $B_j\le I_j$ whenever $\sum_{\ell=1}^{j-1}I_\ell<m$.

Now return to the proof of (iii). Fix $\rho\in(0,1)$ and set $m_{\max}(T):=\lceil\frac{2{\theta}_+(1-\rho)\log T}{\delta^2}\rceil$. Choose $\varepsilon>0$ sufficiently small that $\alpha:=(1+3\varepsilon)^2(1-\rho)<1$. Let $\ell_0(T)=\lfloor T/2\rfloor-m_{\max}(T)-2$. For any fixed $\zeta>0$, we can choose $k_0$ large enough such that
\begin{equation}
{\PP}(G_s) = {\PP}\left(\cap_{k\geq k_0}\{
\overline Y_{s,k}\ge\mu_s-\varepsilon\delta\}
\right)\ge1-\zeta.
\label{eq:suboptimal-mean-good}
\end{equation}
At the same time, a union bound gives
\begin{equation}
\lim_{T\to\infty}{\PP}(G_{o, T}) = \lim_{T\to\infty}{\PP}\left(\cap_{\ell=\ell_0(T)}^T\{\overline Y_{o,\ell}\le\mu_o+\varepsilon\delta\}\right)=1.
\label{eq:optimal-mean-good}
\end{equation}
Also, by (\ref{eq:finite-exploration}),
\begin{equation}
\lim_{T\to\infty}\sup_{{\theta}\in \mathcal K}{\PP}\left(N_s(\lfloor T/2\rfloor;{\theta})<k_0\right)=0.
\label{eq:early-exploration-good}
\end{equation}

Fix ${\theta}\in \mathcal K$ and consider the event $\mathcal G_T({\theta}):= G_s\cap G_{o, T}\cap\{N_s(\lfloor T/2\rfloor;{\theta})\geq k_0\}$. Then we have $\sup_{{\theta}\in\mathcal K}{\PP}(\mathcal G_T({\theta})^c) \le \zeta + o(1)$ as $T\to\infty$. On $\mathcal G_T({\theta})$, at any decision point $t\in\{\lfloor T/2\rfloor+1,\ldots,T\}$ before $k:=N_s(t-1;{\theta})$ reaches $m_-({\theta},T): = \lfloor\frac{2{\theta}(1-\rho)\log T}{\delta^2}\rfloor$, we have $k\ge k_0$, $k<m_-({\theta},T)\le m_{\max}(T)$, and $N_o(t-1;{\theta})=t-1-k\ge\ell_0(T)$. Thus
\[
\overline Y_{s,k}\ge\mu_s-\varepsilon\delta = \mu_o - (1+\varepsilon)\delta,
\qquad
\overline Y_{o,N_o(t-1;{\theta})}\le\mu_o+\varepsilon\delta.
\]
Note that conditional on the history, the two current posterior samples are independent, thus on $\mathcal G_T({\theta})$,
\begin{align*}
{\PP}(A_t=s|\mathcal H_{t-1})
&\ge
{\PP}({\Xi}_s(t;{\theta})\ge\mu_o+2\varepsilon\delta, {\Xi}_o(t;{\theta})\le\mu_o+2\varepsilon\delta |\mathcal H_{t-1})\\
&\ge
\frac12\,
\normcdf\left(-(1+3\varepsilon)\delta\sqrt{k/{\theta}}\right)\\
&\ge \frac12\normcdf\left(-\sqrt{2\alpha\log T}\right)\\
&\ge \frac12 \frac{\sqrt{2\alpha\log T}}{1+2\alpha\log T}\cdot \frac{1}{\sqrt{2\pi}}T^{-\alpha}\\
&\ge c_\alpha(\log T)^{-1/2}T^{-\alpha}=:p_T.
\end{align*}
Here $c_\alpha$ is a constant that depends only on $\alpha$. We have used the Gaussian Mills lower bound on normal CDF and the fact that $k/{\theta}\le2(1-\rho)\log T/\delta^2$.

Now for any ${\theta}\in \mathcal K$, for each
$t=\lfloor T/2\rfloor+1,\ldots,T$, define
\[
I_t({\theta})
:=
\begin{cases}
\ind\{A_t=s\},
&\text{if }
\begin{array}{l}
N_s(t-1;{\theta})<m_-({\theta},T),\\[-2pt]
\overline Y_{s,N_s(t-1;{\theta})}\ge\mu_s-\epsilon\delta,\\[-2pt]
\overline Y_{o,N_o(t-1;{\theta})}\le\mu_o+\epsilon\delta,
\end{array}
\\[12pt]
1,
&\text{otherwise}.
\end{cases}
\]
By the preceding calculation, for every
$t=\lfloor T/2\rfloor+1,\ldots,T$, ${\PP}(I_t({\theta})=1\mid\cH_{t-1})\ge p_T$. Therefore we deduce from (\ref{eq:adaptive-binomial}) that for every integer $m\ge1$,
\[
{\PP}\left(
\sum_{t=\lfloor T/2\rfloor+1}^{T} I_t({\theta})<m
\right)
\le
{\PP}\left(
\operatorname{Binomial}
\bigl(T-\lfloor T/2\rfloor,p_T\bigr)<m
\right).
\]
Note that if $N_s(T;{\theta})<m_-({\theta},T)$, then on $\mathcal G_T({\theta})$, whenever $N_s(t-1;{\theta})\ge k_0$, the two empirical-mean conditions in the
definition of $I_t({\theta})$ hold. Hence $I_t({\theta})=\ind\{A_t=s\}$ throughout $t=\lfloor T/2\rfloor+1,\ldots,T$. In addition, since $N_s(\lfloor T/2\rfloor;{\theta})\ge k_0$ on $\mathcal G_T({\theta})$,
\[
\sum_{t=\lfloor T/2\rfloor+1}^{T}I_t({\theta})
=
N_s(T;{\theta})-N_s(\lfloor T/2\rfloor;{\theta})
<
m_-({\theta},T)-k_0
\le m_{\max}(T).
\]
Therefore,
\begin{align*}
{\PP}(N_s(T;{\theta})<m_-({\theta},T))
&\le
{\PP}(\mathcal G_T({\theta})^c)
+{\PP}\left(\sum_{t=\lfloor T/2\rfloor+1}^{T}I_t({\theta})<m_{\max}(T)\right)\\
&\le\zeta+o(1)+{\PP}\left(\operatorname{Binomial}\bigl(T-\lfloor T/2\rfloor,p_T\bigr)<m_{\max}(T)\right),
\end{align*}
uniformly over ${\theta}\in \mathcal K$. Because $(T-\lfloor T/2\rfloor)p_T\asymp\frac{T^{1-\alpha}}{\sqrt{\log T}}\gg m_{\max}(T)$, the last probability tends to zero.  Letting $\zeta\downarrow0$ yields
\[
\lim_{T\to\infty}
\inf_{{\theta}\in \mathcal K}
{\PP}(N_s(T;{\theta})\ge m_-({\theta},T))=1.
\]
Consequently, $\liminf_{T\to\infty}\inf_{{\theta}\in \mathcal K}\frac{\delta^2\E[N_s(T;{\theta})]}{2{\theta}\log T}\ge1-\rho$. Let $\rho\downarrow0$ to obtain (\ref{eq:TS-log-lower}).

We now show (\ref{eq:TS-supercritical-sharp}). Fix constants $\rho>0$ and $\varepsilon\in(0,1/2)$, and define $b_\varepsilon:=\mu_o-\varepsilon\delta$, $L_T'({\theta}):=\lceil\frac{2{\theta}(1+\rho)\log T}{(1-2\varepsilon)^2\delta^2}\rceil$. Similar to part (i), for every ${\theta}\in \mathcal K$, decompose
\begin{align}
N_s(T;{\theta})
&\le
1+L_T'({\theta})
+S_{1,T}'({\theta})+S_{2,T}'({\theta})+S_{3,T}'({\theta}),
\label{eq:TS-supercritical-decomp}
\end{align}
where
\begin{align*}
S_{1,T}'({\theta})
&:=
\sum_{t=3}^T
\ind\Bigl\{
A_t(\theta)=s,\,
N_s(t-1;{\theta})\ge L_T'({\theta}),\,
\overline Y_{s,N_s(t-1;{\theta})}>\mu_s+\varepsilon\delta
\Bigr\},
\\
S_{2,T}'({\theta})
&:=
\sum_{t=3}^T
\ind\Bigl\{
A_t(\theta)=s,\,
N_s(t-1;{\theta})\ge L_T'({\theta}),\,
\overline Y_{s,N_s(t-1;{\theta})}\le\mu_s+\varepsilon\delta,\,
{\Xi}_s(t;{\theta})>b_\varepsilon
\Bigr\},
\\
S_{3,T}'({\theta})
&:=
\sum_{t=3}^T
\ind\Bigl\{
A_t(\theta)=s,\,
{\Xi}_s(t;{\theta})\le b_\varepsilon
\Bigr\}.
\end{align*}
We bound the three terms separately. For $S'_{1,T}({\theta})$, Similar to the analysis in part (i), $S_{1,T}'({\theta})\le\sum_{k=L_T'({\theta})}^\infty\ind\{\overline Y_{s,k}>\mu_s+\varepsilon\delta\}$. Since $\overline Y_{s,k}-\mu_s\sim\mathcal N\left(0,\frac{\tau_s^2}{k}\right)$, we have
\begin{align}
\E[S_{1,T}'({\theta})]\le
\sum_{k=L_T'({\theta})}^\infty\exp\left\{
-\frac{k\varepsilon^2\delta^2}{2\tau_s^2}
\right\}
=
\frac{\exp\{-L_T'({\theta})\varepsilon^2\delta^2/(2\tau_s^2)\}}{1-\exp\{-\varepsilon^2\delta^2/(2\tau_s^2)\}}.
\label{eq:TS-supercritical-S1}
\end{align}
Since $L_T'({\theta})\ge\frac{2{\theta}_-(1+\rho)\log T}{(1-2\varepsilon)^2\delta^2}$, the numerator in (\ref{eq:TS-supercritical-S1}) is no larger than $T^{-c_{1}}$, where $c_{1}:=\frac{{\theta}_-(1+\rho)\varepsilon^2}{(1-2\varepsilon)^2\tau_s^2}>0$. Thus,
\begin{equation}
\sup_{{\theta}\in \mathcal K}\E[S'_{1,T}({\theta})]
=O(T^{-c_{\rho,\varepsilon}})
=o(1).
\label{eq:TS-supercritical-S1-final}
\end{equation}

Next consider $S'_{2,T}({\theta})$.  If $N_s(t-1;{\theta}) = k$, then conditional on $\mathcal H_{t-1}$, ${\Xi}_s(t;{\theta})\sim \mathcal N\left(\overline Y_{s,k},\frac{\theta}{k}\right)$. Thus, on the event $k=N_s(t-1;{\theta})\ge L_T'({\theta})$, $\overline Y_{s,k}\le\mu_s+\varepsilon\delta$,
\begin{align*}
{\PP}\{{\Xi}_s(t;{\theta})>b_\varepsilon\mid\mathcal H_{t-1}\}
&=
\normcdf\left(
-\frac{b_\varepsilon-\overline Y_{s,k}}{\sqrt{{\theta}/k}}
\right)
\\
&\le
\exp\left\{
-\frac{k(b_\varepsilon-\overline Y_{s,k})^2}{2{\theta}}
\right\}
\\
&\le
\exp\left\{
-\frac{k(1-2\varepsilon)^2\delta^2}{2{\theta}}
\right\}
\\
&\le
T^{-(1+\rho)}.
\end{align*}
Taking expectations and summing over $t=3,\ldots,T$ gives
\begin{equation}
\sup_{{\theta}\in \mathcal K}\E[S'_{2,T}({\theta})]\le T^{-\rho}.
\label{eq:TS-supercritical-S2}
\end{equation}

Finally, we control $S'_{3,T}({\theta})$.  For $k\ge1$, define $p_k({\theta}):=\normcdf\left((\overline Y_{o,k}-b_\varepsilon)\sqrt{\frac{k}{{\theta}}}\right)$. Without loss of generality, we can write $\overline Y_{o,k}=\mu_o+\frac{\tau_o}{\sqrt{k}}Z$ for some $Z\sim \mathcal N(0,1)$.
Thus we have
\begin{equation}
(\overline Y_{o,k}-b_\varepsilon)\sqrt{\frac{k}{{\theta}}}
=
\frac{\varepsilon\delta}{\sqrt{\theta}}\sqrt{k}
+
\frac{\tau_o}{\sqrt{\theta}}Z.
\label{eq:TS-supercritical-Xk}
\end{equation}
Moreover, by the assumption ${\theta}\ge {\theta}_-\ge(1+\xi)\tau_o^2$, the parameter $\lambda({\theta}):=\frac{\tau_o}{\sqrt{\theta}}$ ranges over the fixed compact interval $[\frac{\tau_o}{\sqrt{{\theta}_+}},\frac{\tau_o}{\sqrt{{\theta}_-}}]\subset(0,1)$. We can use Lemma~\ref{lem:odds-decomp} and Lemma~\ref{lem:Gaussian-odds}(i) to obtain
\begin{align}
\EE[S'_{3,T}({\theta})]
&\le
\sum_{k=1}^{T-1}
\E\left[
\frac{1-p_k({\theta})}{p_k({\theta})}
\right]
\nonumber\\
&\le
\sum_{k=1}^{\infty}
\E\left[
H\left(
\frac{\varepsilon\delta}{\sqrt{\theta}}\sqrt{k}
+
\frac{\tau_o}{\sqrt{\theta}}Z
\right)
\right]
\nonumber\\
&\le
C
\left(\frac{\varepsilon\delta}{\sqrt{\theta}}\right)^{-2}\le C\frac{{\theta}_+}{\varepsilon^2\delta^2},\nonumber
\end{align}
where $C<\infty$ depends only on $\mathcal K$ and $\tau_o^2$. In particular,
\begin{equation}
\sup_{{\theta}\in \mathcal K}\E[S'_{3,T}({\theta})]=O(1).\label{eq:TS-supercritical-S3}
\end{equation}

Combining
(\ref{eq:TS-supercritical-decomp}), (\ref{eq:TS-supercritical-S1-final}), (\ref{eq:TS-supercritical-S2}), and (\ref{eq:TS-supercritical-S3}), we obtain
\begin{align}
\sup_{{\theta}\in \mathcal K}\E[N_s(T;{\theta})]
&\le L_T'({\theta}) +O(1).\label{eq:TS-supercritical-upper}
\end{align}
Since ${\theta}\ge {\theta}_->0$, $\limsup_{T\to\infty}\sup_{{\theta}\in \mathcal K}\frac{\delta^2\E[N_s(T;{\theta})]}{2{\theta}\log T}\le\frac{1+\rho}{(1-2\varepsilon)^2}$. Letting first $\varepsilon\downarrow0$ and then $\rho\downarrow0$ yields
\begin{equation}
\limsup_{T\to\infty}\sup_{{\theta}\in \mathcal K}\frac{\delta^2\E[N_s(T;{\theta})]}{2{\theta}\log T}\le1.
\label{eq:TS-log-upper}
\end{equation}
Combining this inequality with (\ref{eq:TS-log-lower}), we obtain (\ref{eq:TS-supercritical-sharp}).

\noindent\textbf{Part (iv).}
Fix $\delta>0$ and assume $\theta=\tau_o^2$. LHS$\geq$RHS in (\ref{eq:ts-critical-regret}) follows from part (iii) applied to a compact
interval containing $\theta$.

To prove LHS$\leq$RHS in (\ref{eq:ts-critical-regret}), fix $\rho>0$ and $\varepsilon\in(0,1/2)$, and use the same definitions of $b_\varepsilon$, $L'_T(\theta)$, and $S'_{j,T}(\theta)$, $j=1,2,3$, as in part (iii). The analysis for
$S'_{1,T}(\theta)$ and $S'_{2,T}(\theta)$ remain valid. We now show that $\E_{\nu_\delta}[S'_{3,T}(\theta)]=O(1)$.

Let $a:=\frac{\varepsilon\delta}{\sqrt{\theta}}>0$, $X_k:=(\overline Y_{o,k}-b_\varepsilon)\sqrt{\frac{k}{\theta}}$ for $k\ge1$. Since $\theta=\tau_o^2$, we have $X_k\sim\cN(a\sqrt{k},1)$. According to Lemma \ref{lem:odds-decomp},
\[
\E_{\nu_\delta}[S'_{3,T}(\theta)]\le
\sum_{k=1}^{T-1}\E_{\nu_\delta}[H(X_k)].
\]
We now show that the corresponding infinite sum is still finite. From Lemma \ref{lem:Gaussian-odds}, we have
\[
\E_{\nu_\delta}
\left[H(X_k)\mathbf{1}\{X_k\ge0\}\right]
\le
C\E_{\nu_\delta}\left[e^{-X_k^2/2}\right]
=
\frac{C}{\sqrt{2}}e^{-a^2k/4}.
\]
On the other hand, we also have
\begin{align*}
\E_{\nu_\delta}
\left[H(X_k)\mathbf{1}\{X_k<0\}\right]
&\le
\frac{C}{\sqrt{2\pi}}
\int_{-\infty}^{0}
(1+|x|)
\exp\left\{
\frac{x^2}{2}
-\frac{(x-a\sqrt{k})^2}{2}
\right\}\,dx
\\
&=
\frac{C}{\sqrt{2\pi}}e^{-a^2k/2}
\int_{-\infty}^{0}
(1+|x|)e^{a\sqrt{k}x}\,dx
\\
&=
\frac{C}{\sqrt{2\pi}}e^{-a^2k/2}
\left(
\frac{1}{a\sqrt{k}}+\frac{1}{a^2k}
\right).
\end{align*}
Since $a>0$ is fixed, both upper bounds summed in $k$ is finite. Consequently, $\sum_{k=1}^{\infty}
\E_{\nu_\delta}[H(X_k)]<\infty$, and hence
$\E_{\nu_\delta}[S'_{3,T}(\theta)]=O(1)$.

Combining this bound with the decomposition and the bounds for
$S'_{1,T}(\theta)$ and $S'_{2,T}(\theta)$ from part (iii) yields
\[
\E_{\nu_\delta}[N_s(T;\theta)]
\le
\frac{2\theta(1+\rho)}
{(1-2\varepsilon)^2\delta^2}\log T+O(1),
\]
where the $O(1)$ term may depend on
$\rho,\varepsilon,\delta,\tau_o^2,\tau_s^2$, but not on $T$. Thus,
\[
\limsup_{T\to\infty}
\frac{\Reg_T(\theta;\nu_\delta)}{\log T}
\le
\frac{2\theta(1+\rho)}
{(1-2\varepsilon)^2\delta}.
\]
Letting $\varepsilon\downarrow0$ and $\rho\downarrow0$, and combining with the first part, we prove (\ref{eq:ts-critical-regret}).

\noindent\textbf{Part (v).} Suppose first that ${\theta}<\tau_o^2$.  Choose $\varepsilon>0$ such that ${\theta}+\varepsilon<\tau_o^2$.  For all sufficiently large $T$, ${\theta}_T\in[{\theta}-\varepsilon,{\theta}+\varepsilon]$.  Part (ii) gives, for every
$\eta>0$,
\[
\liminf_{T\to\infty}
\frac{\log(\Reg_T({\theta}_T;\nu_\delta))}{\log T}
\ge
1-\frac{{\theta}+\varepsilon}{\tau_o^2}-\eta.
\]
By letting $\eta\downarrow0$ and $\varepsilon\downarrow0$, we obtain that LHS$\ge$ RHS in (\ref{eq:TS-fixed-exponent}). The $\le$ side is a direct consequence of Part (i).

Now assume ${\theta}\ge\tau_o^2$. Part (i) implies that the limsup in (\ref{eq:TS-fixed-exponent}) is no larger than zero. Since the initialization pulls the suboptimal arm once, $\Reg_T({\theta}_T;\nu_\delta)\ge\delta$, so the liminf is at least zero. This proves (\ref{eq:TS-fixed-exponent}).
\end{proof}

\begin{lemma}[Limit of the Plug-In minimizer]\label{lem:PI-minimizer-TS}
Let $\nu$ be a fixed two-armed Gaussian environment with distinct means and positive variances. Let $\tau_o^2$ be the variance of its optimal arm and assume $\tau_o^2\notin\{\underline \theta,\overline \theta\}$. For each $T$, let $\theta_T(\nu)\in\argmin_{\theta\in\Theta}\Reg_T(\theta;\nu)$, where $\Reg_T(\theta;\cdot)$ corresponds to the TS algorithm. Then
\begin{equation}
\lim_{T\to\infty}\theta_T(\nu)=\Pi_\Theta(\tau_o^2).
\label{eq:PI-fixed-limit}
\end{equation}
\end{lemma}

\begin{proof}[Proof of Lemma \ref{lem:PI-minimizer-TS}]
The proof is similar to the UCB setting and is stated here for completeness. For simplicity, write ${\theta}_T = {\theta}_T(\nu)$ and  $c=\Pi_\Theta(\tau_o^2)$. If $\tau_o^2<\underline {\theta}$, by (\ref{eq:TS-supercritical-sharp}), for every fixed
$\varepsilon>0$,
\[
\inf_{{\theta}\in[\underline {\theta}+\varepsilon,\overline {\theta}]}
\Reg_T({\theta};\nu)
>
\Reg_T(\underline {\theta};\nu)
\]
for all sufficiently large $T$, because the leading term is
$2{\theta}\log T/\delta$.  Hence ${\theta}_T<\underline {\theta}+\varepsilon$ eventually. Thus, ${\theta}_T\to\underline {\theta}=c$.

Suppose $\tau_o^2\in(\underline {\theta},\overline {\theta})$.  Fix a sufficiently small $\varepsilon>0$. For ${\theta}\le\tau_o^2-\varepsilon$, Lemma~\ref{lem:fixed-env-TS}(ii) gives a positive polynomial regret lower bound, whereas ${\theta}_0=\tau_o^2+\varepsilon/2$ has logarithmic regret by Lemma~\ref{lem:fixed-env-TS}(iii). Thus ${\theta}_T>\tau_o^2-\varepsilon$ for sufficiently large $T$. On the other hand, Lemma~\ref{lem:fixed-env-TS}(iii) implies
\[
\inf_{{\theta}\in[\tau_o^2+\varepsilon,\overline {\theta}]}
\Reg_T({\theta};\nu)>\Reg_T({\theta}_0;\nu)
\]
for all sufficiently large $T$.  Thus ${\theta}_T<\tau_o^2+\varepsilon$ eventually. Let $\varepsilon\downarrow0$ proves convergence to $c = \tau_o^2$.

Finally, suppose $\tau_o^2>\overline {\theta}$.  Fix $\varepsilon>0$.  By the uniform lower bound in Lemma~\ref{lem:fixed-env-TS}(ii), every ${\theta}\le\overline {\theta}-\varepsilon$ has regret exponent at least $1-(\overline {\theta}-\varepsilon)/\tau_o^2-o(1)$, whereas the upper envelope in Lemma~\ref{lem:fixed-env-TS}(i) shows that $\overline {\theta}$ has exponent at most $1-\overline {\theta}/\tau_o^2+o(1)$. The former exponent is strictly larger, so ${\theta}_T>\overline {\theta}-\varepsilon$ eventually. Hence ${\theta}_T\to\overline {\theta}=c$ as $T\to\infty$.
\end{proof}

\begin{lemma}[Limit of the Gaussian-perturbed minimizer] \label{lem:UQ-exponent-TS}
Consider the TS$(\theta)$ algorithm defined in Alg. \ref{alg:ts}. Fix any offline data $\cD_{\off}$ with $\widehat\sigma_{1}^2,\widehat\sigma_{2}^2>0$. Let $\bar \sigma^2=\widehat\sigma_{1}^2\vee \widehat\sigma_{2}^2$. If $\theta_T\in\Theta$ and $\theta_T\to \theta$, then
\begin{equation}
\lim_{T\to\infty}
\frac{\log\bigl(
\Reg_{\UA, T}(\theta_T;\cD_{\off})\bigr)}{\log T}=\left(1-\frac{\theta}{\bar\sigma^2}\right)_+.
\label{eq:UQ-objective-exponent-TS}
\end{equation}
Consequently, every sequence of minimizers in (\ref{eq:uq-selector}) satisfies
\begin{equation}
\liminf_{T\to\infty}\widehat \theta_{\UA, T}(\cD_{\off})\ge\Pi_\Theta(\bar\sigma^2).
\label{eq:UQ-minimizer-liminf-TS}
\end{equation}
\end{lemma}

\begin{proof}[Proof of Lemma \ref{lem:UQ-exponent-TS}]
The proof is similar to the UCB setting and is stated here for completeness. Let $X:=|\tilde\mu_{1}(Z_1)-\tilde\mu_{2}(Z_2)|$. Conditional on $\cD_{\off}$, $X$ has a folded-normal distribution with bounded density on $(0,\infty)$ and finite first moment.

Fix $\eta$ such that $0<\eta<\frac12[1-(1-\underline{{\theta}}/{\bar\sigma^2})_+]$. For any perturbed environment, its optimal-arm variance is at most ${\bar\sigma^2}$. Apply Lemma~\ref{lem:fixed-env-TS}(i) gives
\[
\Reg_T({\theta}_T;\tilde M(Z))
\le
C_\eta\left[X+\frac{B_T}{X}\right],
\]
where $B_T:=\log T+T^{q_T+\eta}$, $q_T:=\left(1-\frac{{\theta}_T}{{\bar\sigma^2}}\right)_+$. We also have the trivial bound $\Reg_T\le TX$.

Let $x_T:=\sqrt{B_T/T}$. Then $\lim_{T\to\infty} x_T=0$. Denote $f_X$ as the density of $X$, then $\|f_X\|_\infty<\infty$, and
\begin{align*}
\E[TX\ind\{X\le x_T\}]&\le\tfrac12\|f_X\|_\infty Tx_T^2=O(B_T),\\
\E\left[\frac{B_T}{X}\ind\{x_T<X\le1\}\right]
&\le\|f_X\|_\infty B_T\log(1/x_T)=O(B_T\log T),\\
\E\left[\frac{B_T}{X}\ind\{X>1\}\right]&=O(B_T).
\end{align*}
Therefore
\[
\Reg_{\UA, T}({\theta}_T;\cD_{\off})\le C'B_T(1+\log T),
\]
and hence
\[
\limsup_{T\to\infty}\frac{\log(\Reg_{\UA, T}({\theta}_T;\cD_{\off}))}{\log T}\le\left(1-\frac{{\theta}}{{\bar\sigma^2}}\right)_++\eta.
\]
Let $\eta\downarrow0$, and we prove LHS$\leq$RHS in  (\ref{eq:UQ-objective-exponent-TS}).

Below we prove the reverse direction LHS$\geq$RHS. The case ${\theta}\ge {\bar\sigma^2}$ is immediate from nonnegativity. Now suppose ${\theta}<{\bar\sigma^2}$, and let $h$ be an arm with $\widehat\sigma_{h}^2={\bar\sigma^2}$. The signed perturbed gap $\tilde\mu_{h}(Z_h)-\tilde\mu_{3-h}(Z_{3-h})$ is nondegenerate Gaussian. Thus, there exist $0<g_-<g_+<\infty$ such that the event
\[
\mathcal W:=
\{g_-\le\tilde\mu_{h}(Z_h)-\tilde\mu_{3-h}(Z_{3-h})\le g_+\}
\]
has conditional probability $w(\cD_{\off})>0$.  On $\mathcal W$, arm $h$ is optimal, its variance is ${\bar\sigma^2}$, and the gap lies in $[g_-,g_+]$.

Fix $\varepsilon,\eta>0$ small enough such that ${\theta}+\varepsilon<{\bar\sigma^2}$. Eventually ${\theta}_T\le {\theta}+\varepsilon$. Lemma~\ref{lem:fixed-env-TS}(ii), applied uniformly on $\mathcal W$, gives
\[
\Reg_{\UA, T}({\theta}_T;\cD_{\off})\ge w(\cD_{\off})c_{\varepsilon,\eta} T^{1-({\theta}+\varepsilon)/{\bar\sigma^2}-\eta}
\]
for a constant $c_{\varepsilon,\eta}$ depending only on $\varepsilon$ and $\eta$. Taking log and letting $\varepsilon,\eta\downarrow0$ proves the desired inequality.

Finally, we prove (\ref{eq:UQ-minimizer-liminf-TS}). Let $c=\Pi_\Theta({\bar\sigma^2})$ and suppose some subsequence of minimizers converges to ${\theta}<c$.  If ${\bar\sigma^2}\in(\underline {\theta},\overline {\theta})$, choose a comparator $\theta'\in({\bar\sigma^2},\overline {\theta}]$. Along the subsequence, the minimizers have a strictly positive objective exponent by (\ref{eq:UQ-objective-exponent-TS}), whereas the comparator has exponent zero, contradicting optimality.  If ${\bar\sigma^2}\ge\overline {\theta}$, use the comparator $\theta'=\overline {\theta}$; then $1-{\theta}/{\bar\sigma^2}>1-\overline {\theta}/{\bar\sigma^2}$, again contradicting optimality. If ${\bar\sigma^2}\le\underline {\theta}$, then
$c=\underline {\theta}$ and no point of $\Theta$ is below $c$.  Thus no limit  lies below $c$.
\end{proof}

\begin{lemma}[Gaussian odds]
\label{lem:Gaussian-odds}
Recall that $H(x)=\frac{\normcdf(-x)}{\normcdf(x)}$. There exists a universal constant $C<\infty$ such that
\begin{equation}
H(x)\le
\begin{cases}
C e^{-x^2/2},&x\ge0,\\
C(1+|x|)e^{x^2/2},&x<0.
\end{cases}
\label{eq:H-pointwise}
\end{equation}
In addition, fix any $a>0$ and let $\lambda\in \Lambda = [\lambda_-, \lambda_+]$, where $0<\lambda_-<\lambda_+<+\infty$. For each $k\in\NN^+$, let  $X_k\sim\cN(a\sqrt{k},\lambda^2)$. Then the following holds:
\begin{enumerate}[label=(\roman*)]
\item If $\lambda_+<1$, then there exists constant $C'$ depending only on $\Lambda$ such that
\begin{equation}
\sum_{k=1}^{\infty}\EE[H(X_k)]
\le C'a^{-2}.
\label{eq:odds-supercritical}
\end{equation}
\item For every $\eta>0$, there exists constant $C_\eta$ depending only on $\eta$ and $\Lambda$ such that
\begin{equation}
\sum_{k=1}^{\infty}
\EE\bigl[\min\{T,H(X_k)\}\bigr]
\le
C_\eta a^{-2}
T^{(1-\lambda^{-2})_++\eta},
\qquad \forall \lambda\in\Lambda, T\ge3.
\label{eq:odds-truncated}
\end{equation}
\end{enumerate}
\end{lemma}

\begin{proof}[Proof of Lemma \ref{lem:Gaussian-odds}]
We first prove (\ref{eq:H-pointwise}). For $x\ge0$, $\normcdf(x)\ge1/2$ and a standard Gaussian upper-tail bound gives the first line of (\ref{eq:H-pointwise}). For $x<0$, apply the Mills lower bound $\normcdf(x)=\normcdf(-|x|)\ge\frac{|x|}{1+x^2}\normpdf(|x|)$ when $|x|\ge1$. We can then adjust the constant to accommodate the case $x\in[-1,0]$.

We now prove (i). We first consider the case $\{X_k\ge0\}$. By the first case of (\ref{eq:H-pointwise}), $H(X_k)\ind\{X_k\ge0\}\le C e^{-X_k^2/2}\ind\{X_k\ge0\}\le C e^{-X_k^2/2}$. Thus
\[
\EE\!\left[H(X_k)\ind\{X_k\ge0\}\right]\le C\,\EE\!\left[e^{-X_k^2/2}\right].
\]
Since $X_k\sim N(a\sqrt{k},\lambda^2)$, $\EE[e^{-X_k^2/2}] = \frac1{\sqrt{1+\lambda^2}}
\exp\left\{-\frac{a^2k}{2(1+\lambda^2)}\right\}$. In addition, because $\lambda\le\lambda_+$, we deduce that
\[
\E\!\left[H(X_k)\ind\{X_k\ge0\}\right]
\le C
\exp\left\{-\frac{a^2k}{2(1+\lambda_+^2)}\right\}.
\]
Summing over $k$ gives
\begin{align*}
\sum_{k=1}^{\infty}
\E\!\left[H(X_k)\ind\{X_k\ge0\}\right]
&\le
C\sum_{k=1}^{\infty}
\exp\left\{
-\frac{a^2k}{2(1+\lambda_+^2)}
\right\}
\\
&=
\frac{C}{
\exp\{a^2/[2(1+\lambda_+^2)]\}-1
}
\\
&\le C_{\Lambda}a^{-2},
\end{align*}
where $C_{\Lambda}$ depends only on $\Lambda$. The last inequality follows from the fact that $e^x-1\ge x$ for $x\ge0$.

Now consider $\{X_k<0\}$. Since $X_k\sim \cN(a\sqrt{k},\lambda^2)$, the second case of (\ref{eq:H-pointwise}) gives
\begin{align*}
\E\!\left[H(X_k)\ind\{X_k<0\}\right]
&=
\int_{-\infty}^0 H(x)
\frac{1}{\sqrt{2\pi}\lambda}
\exp\left\{-\frac{(x-a\sqrt{k})^2}{2\lambda^2}\right\}\,dx
\\
&\le
\frac{C}{\lambda}
\int_{-\infty}^0
(1+|x|)
\exp\left\{
\frac{x^2}{2}
-\frac{(x-a\sqrt{k})^2}{2\lambda^2}
\right\}\,dx
\\
&=
\frac{C}{\lambda}
e^{-a^2k/(2\lambda^2)}
\int_{-\infty}^0
(1+|x|)
\exp\left\{
-\frac{1-\lambda^2}{2\lambda^2}x^2
+\frac{a\sqrt{k}}{\lambda^2}x
\right\}\,dx .
\end{align*}
Here note that when $x<0$, $\frac{a\sqrt{k}}{\lambda^2}x\le 0$. Moreover, $\frac{1-\lambda^2}{2\lambda^2}\ge c_{\lambda_+}:=\frac{1-\lambda_+^2}{2\lambda_+^2}>0$. Therefore,
\begin{align*}
\E\!\left[H(X_k)\ind\{X_k<0\}\right]\le
C_{\Lambda} e^{-a^2k/(2\lambda^2)}
\int_{-\infty}^0
(1+|x|)e^{-c_{\lambda_+}x^2}\,dx\le
C_{\Lambda}'
e^{-a^2k/(2\lambda_+^2)}
\end{align*}
for constants $C_{\Lambda}$, $C_{\Lambda}'$ depending only on $\Lambda$. Hence
\begin{align*}
\sum_{k=1}^{\infty}
\E\!\left[H(X_k)\ind\{X_k<0\}\right]\le
C_{\Lambda}'\sum_{k=1}^{\infty}
e^{-a^2k/(2\lambda_+^2)}=
\frac{C_{\Lambda}'}{e^{a^2/(2\lambda_+^2)}-1}\le
C_{\Lambda}''a^{-2},
\end{align*}
where $C_{\Lambda}''$ depends only on $\Lambda$. Here the last inequality follows from $e^x-1\ge x$.

Combining the above arguments, we obtain (\ref{eq:odds-supercritical}).

For (ii), fix a small $\gamma>0$. Then there exists a constant $C_\gamma$ such that $\forall x<0$, $1+|x|\leq C_\gamma \exp(\gamma x^2/2)$. Thus for $x<0$, $H(x)\le C_\gamma\exp\{(1+\gamma)x^2/2\}$. For any $u\geq u_0 :=eC_\gamma$, the event $H(X_k)>u$ and $X_k<0$ implies $X_k<-y_u$, $y_u:=\sqrt{\frac{2}{1+\gamma}\log(u/C_\gamma)}$. Therefore
\begin{align*}
\sum_{k=1}^\infty {\PP}(X_k<0, H(X_k)>u) &\leq
\sum_{k=1}^{\infty}{\PP}(X_k<-y_u)
\le
\sum_{k=1}^{\infty}\exp\left\{-\frac{(y_u+a\sqrt{k})^2}{2\lambda^2}\right\}\\
&\le
\exp\left\{-\frac{y_u^2}{2\lambda^2}\right\}\sum_{k=1}^{\infty}
\exp\left\{-\frac{a^2k}{2\lambda^2}\right\}
\le
\frac{2\lambda^2}{a^2}
\exp\left\{-\frac{y_u^2}{2\lambda^2}\right\}\\
&\le
C_\gamma'a^{-2}u^{-1/\{(1+\gamma)\lambda^2\}}.
\end{align*}
Here $C_\gamma'$ is a constant that depends only on $C_\gamma$ and $\Lambda$. Let
\[
S_-(T)
:=
\sum_{k=1}^{\infty}
\E\!\left[
\min\{T,H(X_k)\}\ind\{X_k<0\}
\right].
\]
Using the tail-integral identity $\E[\min\{T,Y\}]
=\int_0^T {\PP}(Y>u)\,du$ for $Y\ge0$, for $T\ge u_0$,
\begin{align*}
S_-(T)& = \sum_{k=1}^{\infty}\int_0^T{\PP}(\min\{T,H(X_k)\}\ind\{X_k<0\}>u)du\\
&\leq \sum_{k=1}^{\infty}\int_0^T{\PP}(H(X_k)\ind\{X_k<0\}>u)du\\
&=
\int_0^T
\sum_{k=1}^{\infty}
{\PP}(H(X_k)>u,\ X_k<0)\,du
\\
&=
\int_0^{u_0}
\sum_{k=1}^{\infty}
{\PP}(H(X_k)>u,\ X_k<0)\,du
+
\int_{u_0}^{T}
\sum_{k=1}^{\infty}
{\PP}(H(X_k)>u,\ X_k<0)\,du
\\
&\le
C_2a^{-2}u_0
+
C_\gamma' a^{-2}
\int_{u_0}^{T}
u^{-1/\{(1+\gamma)\lambda^2\}}\,du
\\
&\le
C_\gamma'' a^{-2}
\left[
u_0+
\int_{u_0}^{T}
u^{-1/\{(1+\gamma)\lambda^2\}}\,du
\right].
\end{align*}
Here $C_2$ is a constant depending on $\Lambda$, and $C_\gamma'' = C_2+C_\gamma'$. If $3\leq T<u_0$, we can similarly obtain that $S_-(T)\leq C_2a^{-2} u_0$.

At the same time, an argument similar to the analysis in (i) gives
\[
\sum_{k=1}^{\infty}
\mathbb E[\min\{T,H(X_k)\}\ind\{X_k\ge0\}]
\le
\sum_{k=1}^{\infty}
\mathbb E[H(X_k)\ind\{X_k\ge0\}]
\le C_{\Lambda}a^{-2}.
\]
Here $C_{\Lambda}$ is a constant depending only on $\Lambda$.

By summarizing the two parts and using the fact that $\int_{u_0}^T u^{-\alpha}\,du\le C_{u_0}T^{(1-\alpha)_+}(1+\log T)$ with a constant $C_{u_0}$ depending only on $u_0$, we obtain that
\[
\sum_{k=1}^{\infty}\E[\min\{T,H(X_k)\}]
\le
C_3a^{-2}
T^{(1-1/\{(1+\gamma)\lambda^2\})_+}(1+\log T)
\]
for a constant $C_3$ depending on $\Lambda$ and $\gamma$. Finally, because $\Lambda$ is a fixed compact interval, we can choose $\gamma$ to be positive and small so that $(1-\frac1{(1+\gamma)\lambda^2})_+\le(1-\lambda^{-2})_++\eta/2$ holds uniformly for all $\lambda\in\Lambda$. We absorb $1+\log T$ into $T^{\eta/2}$ and adjust the constant $C_3$ properly to obtain (\ref{eq:odds-truncated}).
\end{proof}

\begin{lemma}\label{lem:odds-decomp}
Let $\{A_t(\theta)\}_{t\ge 1}$ and $\{\Xi_i(t;\theta)\}_{i\in\{o, s\}, t\geq 3}$ denote the actions and posterior samples from deploying the TS$(\theta)$ algorithm in any two-armed bandit, with two arms labeled as $o$ and $s$. For $k\ge 1$, let $\overline Y_{o,k}$ denote the empirical mean of the first $k$ observed rewards from arm $o$. Fix any $b\in\R$ and let
\[
p_k(\theta):=\normcdf\left((\overline Y_{o,k}-b)\sqrt{k/\theta}\right)
\]
denote the conditional probability that a posterior sample for arm $o$ exceeds $b$ after $k$ pulls. Then
\begin{equation}
\EE\left[
\sum_{t=3}^{T}
\ind\{A_t(\theta)=s,\ \Xi_s(t;{\theta})\le b\}
\right]
\le
\sum_{k=1}^{T-1}
\EE\left[
\min\left\{T,\frac{1-p_k(\theta)}{p_k(\theta)}\right\}
\right].
\label{eq:odds-decomposition}
\end{equation}
\end{lemma}

\begin{proof}[Proof of Lemma \ref{lem:odds-decomp}]
Conditional on the previous history $\mathcal H_{t-1}$, the two current posterior samples ${\Xi}_s(t;{\theta})$ and ${\Xi}_o(t;{\theta})$ are independent. Define $p_t:=p_{N_o(t-1;\theta)}({\theta})$, $q_t:={\PP}({\Xi}_s(t;{\theta})\le b\mid\mathcal H_{t-1})$, then we have
\begin{align*}
{\PP}(A_t(\theta)=s,{\Xi}_s(t;{\theta})\le b\mid\mathcal H_{t-1})
&\le q_t(1-p_t),\\
{\PP}(A_t(\theta)=o\mid\mathcal H_{t-1})
&\ge q_t p_t.
\end{align*}
Thus,
\[
{\PP}(A_t(\theta)=s,{\Xi}_s(t;{\theta})\le b\mid\mathcal H_{t-1})
\le
\frac{1-p_t}{p_t}
{\PP}(A_t(\theta)=o\mid\mathcal H_{t-1}).
\]

For $k=1,\ldots,T-1$, let
\[
I_{t,k}:=\ind\{N_o(t-1;\theta)=k\},
\qquad
R_k({\theta}):=\frac{1-p_k({\theta})}{p_k({\theta})}.
\]
Then on the event $\{N_o(t-1;\theta)=k\}$, we have $p_t=p_k({\theta})$. Moreover, we have
\[
\sum_{t=3}^T I_{t,k}\ind\{A_t(\theta)=o\}\le1,
\]
since selecting arm $o$ when its pre-pull count is $k$ increases that count to $k+1$.

Now consider the case $\{R_k({\theta})\le T\}$. From the tower property and previous bounds, we have
\begin{align*}
&\E\left[
\sum_{t=3}^T
I_{t,k}\ind\{A_t(\theta)=s,{\Xi}_s(t;{\theta})\le b\}
\ind\{R_k({\theta})\le T\}
\right]
\\
\quad\le&
\E\left[
R_k({\theta})\ind\{R_k({\theta})\le T\}
\sum_{t=3}^T I_{t,k}\ind\{A_t(\theta)=o\}
\right]
\\
\quad\le&
\E\left[
R_k({\theta})\ind\{R_k({\theta})\le T\}
\right].
\end{align*}
On $\{R_k({\theta})>T\}$, since the number of relevant rounds is trivially at most
$T$, we have
\[
\E\left[
\sum_{t=3}^T
I_{t,k}\ind\{A_t(\theta)=s,{\Xi}_s(t;{\theta})\le b\}
\ind\{R_k({\theta})>T\}
\right]
\le
T{\PP}(R_k({\theta})>T).
\]
Summarizing the above two cases, we deduce that
\[
\E\left[
\sum_{t=3}^T
I_{t,k}\ind\{A_t(\theta)=s,{\Xi}_s(t;{\theta})\le b\}
\right]
\le
\E\left[
\min\left\{T,R_k({\theta})\right\}
\right].
\]
Finally, because
\[
\sum_{t=3}^T
\ind\{A_t(\theta)=s,{\Xi}_s(t;{\theta})\le b\}
=
\sum_{k=1}^{T-1}
\sum_{t=3}^T
I_{t,k}\ind\{A_t(\theta)=s,{\Xi}_s(t;{\theta})\le b\},
\]
so taking a sum of the previous bound over $k=1,\ldots,T-1$ proves (\ref{eq:odds-decomposition}).
\end{proof}
\section{Ensembles induce a curvature correction}
\label{app:regularization}

\begin{customprop}{\ref{prop:regularized1}}[Ensembles induce a curvature correction]
    Suppose $\lambda \mapsto J_T(\pi_\theta, M_\lambda)$ is twice continuously differentiable
    and its Hessian is $L$-Lipschitz, i.e., $\| \nabla_\lambda^2 J_T(\pi_\theta, M_{\lambda}) - \nabla_\lambda^2 J_T(\pi_\theta, M_{\lambda'}) \|_{\mathrm{op}} \le L \| \lambda - \lambda' \|_2$ for all $\lambda, \lambda'$, where $\| \cdot \|_{\mathrm{op}}$ is the operator norm. Then
    \begin{align*}
        \Big| \E \big[ J_T(\pi_\theta, M_{\widehat \lambda+\epsilon^{(i)}}) \mid \mathcal{D}_{\mathrm{off}} \big]
        - J_T(\pi_\theta, M_{\widehat \lambda}) - \tfrac{1}{2} \textnormal{Tr} \big( \nabla_\lambda^2 J_T(\pi_\theta, M_{\widehat \lambda}) \Sigma \big) \Big|
        \le \frac{L}{6} \, \E \big[ \| \epsilon^{(i)} \|_2^3 \mid \mathcal{D}_{\mathrm{off}} \big].
    \end{align*}
\end{customprop}

\begin{proof}[Proof of Proposition~\ref{prop:regularized1}]
    Let
    \begin{align*}
        E^{(i)} := J_T(\pi_\theta, M_{\widehat \lambda+\epsilon^{(i)}}) - J_T(\pi_\theta, M_{\widehat \lambda}) - \nabla_\lambda J_T(\pi_\theta, M_{\widehat \lambda})^\top \epsilon^{(i)} - \frac{1}{2} (\epsilon^{(i)})^\top \nabla_\lambda^2 J_T(\pi_\theta, M_{\widehat \lambda}) \epsilon^{(i)}.
    \end{align*}
    Since the Hessian of $J_T(\pi_\theta, M_{\lambda})$ is $L$-Lipschitz, by Lemma 1 of \citet{nesterov2006cubic},
    \begin{align*}
        | E^{(i)} | \le \frac{L}{6} \| \epsilon^{(i)} \|_2^3.
    \end{align*}
    Since $\E[\|\epsilon^{(i)}\|_2^3 \mid \mathcal{D}_{\mathrm{off}}] < \infty$ by assumption, $|\nabla_\lambda J_T(\pi_\theta, M_{\widehat \lambda})^\top \epsilon^{(i)}| \le \| \nabla_\lambda J_T(\pi_\theta, M_{\widehat \lambda}) \|_2 \| \epsilon^{(i)} \|_2$ and $|(\epsilon^{(i)})^\top \nabla_\lambda^2 J_T(\pi_\theta, M_{\widehat \lambda}) \epsilon^{(i)}| \le \| \nabla_\lambda^2 J_T(\pi_\theta, M_{\widehat \lambda}) \|_{\mathrm{op}} \| \epsilon^{(i)} \|_2^2$, so we may take the following conditional expectations:
    \begin{align*}
        \E \big[ J_T(\pi_\theta, M_{\widehat \lambda+\epsilon^{(i)}}) \mid \mathcal{D}_{\off} \big]
        = J_T(\pi_\theta, M_{\widehat \lambda}) + \frac{1}{2} \textnormal{Tr} \big( \nabla_\lambda^2 J_T\big(\pi_\theta, M_{\widehat \lambda}\big) \Sigma \big) + \E \big[ E^{(i)} \mid \mathcal{D}_{\mathrm{off}} \big].
    \end{align*}
    where we use that
    \begin{itemize}[leftmargin=*]
        \item $\E[ \nabla_\lambda J_T(\pi_\theta, M_{\widehat \lambda})^\top \epsilon^{(i)} \mid \mathcal{D}_{\mathrm{off}} ] = \nabla_\lambda J_T(\pi_\theta, M_{\widehat \lambda})^\top \E[ \epsilon^{(i)} \mid \mathcal{D}_{\mathrm{off}} ] = 0$
        \item $\E \big[ (\epsilon^{(i)})^\top \nabla_\lambda^2 J_T(\pi_\theta, M_{\widehat \lambda}) \epsilon^{(i)} \mid \mathcal{D}_{\mathrm{off}} \big]
        = \E \big[ \textnormal{Tr} \big( \nabla_\lambda^2 J_T(\pi_\theta, M_{\widehat \lambda}) \epsilon^{(i)} (\epsilon^{(i)})^\top \big) \mid \mathcal{D}_{\mathrm{off}} \big] = \textnormal{Tr} \big( \E \big[ \nabla_\lambda^2 J_T(\pi_\theta, M_{\widehat \lambda}) \epsilon^{(i)} (\epsilon^{(i)})^\top \mid \mathcal{D}_{\mathrm{off}} \big] \big) = \textnormal{Tr} \big( \nabla_\lambda^2 J_T(\pi_\theta, M_{\widehat \lambda}) \Sigma \big)$
    \end{itemize}
    Thus, we have that
    \begin{align*}
        &\Big| \E \big[ J_T(\pi_\theta, M_{\widehat \lambda+\epsilon^{(i)}}) \mid \mathcal{D}_{\mathrm{off}} \big]
        - J_T(\pi_\theta, M_{\widehat \lambda}) - \frac{1}{2} \textnormal{Tr} \big( \nabla_\lambda^2 J_T(\pi_\theta, M_{\widehat \lambda}) \Sigma \big) \Big| \\
        &= \big| \E [ E^{(i)} \mid \mathcal{D}_{\mathrm{off}} ] \big|
        \leq \E \big[ | E^{(i)} |\mid \mathcal{D}_{\mathrm{off}} \big]
        \leq \frac{L}{6} \E \big[ \| \epsilon^{(i)} \|_2^3 \mid \mathcal{D}_{\mathrm{off}} \big].
    \end{align*}
\end{proof}

\end{document}